\documentclass{l4dc2025}

\title[Control-Theoretic Reinforcement Learning Approach]{A General Control-Theoretic Approach for Reinforcement Learning: Theory and Algorithms}
\usepackage{times}
\usepackage{helvet}  % DO NOT CHANGE THIS
\usepackage{courier}  % DO NOT CHANGE THIS
\usepackage{algorithm}
\usepackage{algpseudocode}
\usepackage{newfloat}
\usepackage{listings}

\usepackage{threeparttable}
\usepackage{multirow}
\usepackage[english]{babel}
\usepackage[utf8]{inputenc}
\usepackage{fancyhdr}
\usepackage{makecell}

\usepackage{color}
\usepackage{fancyhdr}
\usepackage{color}
\usepackage{amsfonts}
\usepackage{array}
\usepackage{indentfirst}
\usepackage{optidef}
\usepackage{caption}
\usepackage{array}
\usepackage{float}
\usepackage{booktabs}
\usepackage{multirow}
\usepackage{pifont}
\usepackage{threeparttable}

\newcommand \coolgrey[1]      {{\color[rgb]{0.55,0.57,0.67}#1}}

\DeclareMathAlphabet{\mathpzc}{OT1}{pzc}{m}{it}
\DeclareMathOperator*{\argmax}{\arg\max}
\DeclareMathOperator*{\argmin}{\arg\min}

\newtheorem{assumption}{Assumption}%[section]

\newcommand{\bT}{{\mathbf{T}}}

\newcommand{\bbF}{{\mathbb{F}}}

\newcommand{\bbI}{{\mathbb{I}}}

\newcommand{\bbQ}{{\mathbb{Q}}}

\newcommand{\bbV}{{\mathbb{V}}}
\newcommand{\bbX}{{\mathbb{X}}}

\newcommand{\ex}{\mathbb{E}}
\newcommand{\pr}{\mathbb{P}}

\newcommand{\var}{\mbox{\sf Var}}

\newcommand{\beqn}{\begin{equation}}
\newcommand{\eeqn}{\end{equation}}

\newcommand{\cF}{{\mathcal{F}}}
\newcommand{\cG}{{\mathcal{G}}}

\newcommand{\cT}{{\mathcal{T}}}

\newcommand{\ra}{\rightarrow}           %%
\newcommand\Reals{{\mathbb{R}}}
\newcommand\Ints{{\mathbb{Z}}}

\newcommand{\setA}{{\mathbb{U}}}
\newcommand{\setPA}{{\mathbb{W}}}
\newcommand{\setQ}{{\bbQ}}
\newcommand{\setV}{{\bbV}}
\newcommand{\setX}{{\bbX}}
\newcommand{\rewardR}{{r}}

\newcommand{\BellmanOp}{{\cT_{B}}}

\newcommand{\ContractionOp}{{\bT}}
\author{%
 \Name{Weiqin Chen} \Email{weiqin.chen.work@gmail.com}\\
 \addr Department of Electrical, Computer, and Systems Engineering, Rensselaer Polytechnic Institute%
 \AND
 \Name{Mark S. Squillante} \Email{mss@us.ibm.com}\\
 \addr Mathematical Sciences Department, IBM Research%
 \AND
 \Name{Chai Wah Wu} \Email{cwwu@us.ibm.com}\\
 \addr Mathematical Sciences Department, IBM Research%
 \AND
 \Name{Santiago Paternain} \Email{paters@rpi.edu}\\
 \addr Department of Electrical, Computer, and Systems Engineering, Rensselaer Polytechnic Institute%
}

\begin{document}

\maketitle

%%%%%%%%%%%%%%%%%%%%%%%%%%%%%%%%%%%%%%%%%
\thispagestyle{empty}
%%%%%%%%%%%%%%%%%%%%%%%%%%%%%%%%%%%%%%%%%

\begin{abstract}%
 We devise a control-theoretic reinforcement learning approach to support direct learning of the optimal policy. We establish various theoretical properties of our approach, such as convergence and 
 % \blue{SP: We need to mention Q-learning.}
 % % optimality of our control-theoretic operator, 
 % optimality,
 optimality of our analog of the Bellman operator and $Q$-learning,
 a new control-policy-variable gradient theorem, and a specific gradient ascent algorithm based on this theorem
 within the context of
 % . As a representative example, we adapt our approach to 
 a specific control-theoretic framework. We
 % and 
 empirically evaluate
 % its 
 the performance of our
 control-theoretic 
 approach on several classical reinforcement learning tasks, demonstrating significant improvements in solution quality, sample complexity, and 
 running time 
 % runtime
 of our 
 % control-theoretic 
 approach 
 over state-of-the-art methods.
\end{abstract}


%===============================================================================

\section{Introduction}
\label{sec:introduction}
For many years now, 
numerous
reinforcement learning (RL) 
approaches, with differing degrees of success, have been developed to address
% has been successful in solving 
a wide variety of decision making under uncertainty problems~\citep{feng2024suf,
KaLiMo96, mcmahan2024optimal, shen2024multi, Szep10,Sutton1998, zheng2023adaptive}. 
% \red{Many different RL approaches, with varying degrees of success, have been developed to address such problems~\citep{KaLiMo96,Szep10,Sutton1998}.} \blue{SP: Remove or combine with the previous sentence.}
Model-free RL methods~\citep{RanAls98,MnKaSi+13} can often suffer from high sample complexity that may require an inordinate amount of samples for some problems, making them unsuitable for various applications where collecting large amounts of data is time-consuming, costly and potentially dangerous for the system and its surroundings~\citep{zhang2024implicit, chen2024probabilistic, peng2023weighted, dong2023model, kumar2020conservative}. On the other hand, model-based RL methods
have
been
successful in demonstrating
%\green{demonstrated} 
significantly reduced
sample complexity and
in outperforming
% \green{outperformed} 
model-free approaches
for various decision making under uncertainty 
% problems
%(e.g., Deisenroth and Rasmussen~\cite{DeiRas11} and Meger et al.~\cite{MeHiXu+15}).
problems; see, e.g.,~\cite{DeiRas11,MeHiXu+15}.
However, such model-based approaches can
% often 
suffer from the difficulty of learning an appropriate model and from worse asymptotic
performance than model-free approaches due to model bias from inherently assuming the learned system dynamics model accurately represents
the true system
% environment
%(e.g., Atkeson and Santamara~\cite{AtkSan97}, Schneider~\cite{Schn97}, and Schaal~\cite{Scha97}).
environment; see, e.g.,~\cite{AtkSan97,Schn97,Scha97}.
% ; in addition, an approximate solution of the optimal control policy is often obtained 
% %(e.g., via model predictive control)
% based on the learned system dynamics model (e.g.,~\cite{NaKaFe+18}).

In this paper, we propose a novel form of RL
that exploits optimal control-theoretic methods to solve the general problem formulation in terms of the unknown independent variables of the 
underlying control problem 
and that directly learns these unknown variables through an iterative solution process that applies the corresponding optimal control policy.
This general approach is in strong contrast to many traditional model-based
RL
methods that, after learning the system dynamics model which is often of high complexity and dimensionality, then use this system dynamics model to compute an 
% optimal 
approximate 
solution of a corresponding (stochastic) dynamic programming problem, often applying model predictive
% control
%(e.g., Nagabandi et al.~\cite{NaKaFe+18}).
control; see, e.g.,~\cite{NaKaFe+18}.
Our
control-based RL (CBRL)
approach instead directly learns 
the unknown independent variables of the 
general underlying (unknown) dynamical system
from which we 
derive
% % \red{derive, through control-theoretic means,}
% directly obtain
% an optimal control policy function
% % \red{from}
% within
% a family of control policy functions,
% % \red{often of much lower complexity and dimensionality\red{, from which the optimal 
% % control
% % policy is directly obtained}}
% from which 
the optimal control policy, often of much lower complexity and dimensionality,
% is derived 
through control-theoretic means.
% Furthermore, we establish that our 
% general CBRL
% approach converges to an optimal solution analogous to model-free
% RL 
% approaches while eliminating the problems of model bias in traditional model-based
% RL
% approaches.

The theoretical foundation and analysis of our CRBL approach are presented within the context of a general Markov decision process (MDP) framework
that ($i$) extends the methodology from the family of policies associated with the classical Bellman operator to a family of control-policy functions mapping a vector of (unknown) variables from a corresponding independent variable set to a control policy which is optimal under those variables;
and ($ii$) extends the domain of these control policies from a single state to span across all (or a large subset of) states, 
with the (unknown) variable vector encoding global and local information that needs to be learned.
Within the context of this MDP framework and 
our general CBRL approach, we establish theoretical results on convergence and optimality with respect to (w.r.t.)
a CBRL operator and a CBRL version of $Q$-learning, analogous to corresponding results for the Bellman operator
and classical $Q$-learning, respectively.
One might potentially consider our CBRL approach to be somewhat related to previous efforts on learning a parameterized policy within an MDP framework to reduce sample complexity,
such as policy gradient methods~\citep{NIPS1999_464d828b,Sutton1998} and their variants including neural network based policy optimization approaches~\citep{schulman2015trust,schulman2017proximal,JMLR:v22:19-736}.
Despite any potentially perceived similarities, it is important to note that our CBRL approach
% is not based on parameterized policies and 
is fundamentally different 
from 
policy gradient methods
in several important ways. 
More specifically, ($i$) we do not consider parameterized policies within the general MDP framework, ($ii$) we instead exploit control-theoretic methods to derive the optimal policy in terms of estimates of a few 
unknown (global and local) 
independent
variables of the corresponding control problem, and ($iii$) we directly learn these unknown variables in an iterative manner based on observations from applying the optimal control policy for the current estimate of variables.
Meanwhile, within the context of
% our MDP framework and 
our
general 
CBRL approach, we establish a new control-policy-variable gradient theorem, analogous to the standard policy gradient theorem, together with a corresponding gradient ascent method that comprises an iterative process for directly learning the (unknown) variable vector.

With its foundation being optimal control, our CBRL approach is particularly suitable for dynamical systems in general and thereby provides the optimal control policy for a wide variety of systems. In addition to its established theoretical properties, numerical experiments of various 
% simple robotics 
classical decision making under uncertainty
tasks empirically demonstrate the effectiveness and performance benefits of our CBRL approach in reducing the number of samples needed, which is a key requirement for the application of learning and decision-making algorithms in 
% real robotic 
real-world
systems.

% organization of the paper.
The remainder of the paper is organized as follows.
% In the remainder of this paper, we 
We
first present our general CBRL 
% approach and then adapt this
% % general 
% approach to a particular control-theoretic framework.
approach,
% adapted to a specific control-theoretic framework, 
% including its various theoretical properties.
which includes establishing various theoretical properties of our approach.
Numerical results are presented next, 
% \red{Empirical results are presented next,} 
% Numerical experiments are presented next to empirically demonstrate the performance benefits of our CBRL approach, 
followed by 
% a discussion of limitations and future work.
concluding remarks.
% The appendix provides 
% all proofs, 
% % experimental details, and additional results.
% and 
% additional theoretical and experimental details
% and results.
All proofs and additional
% theoretical and experimental 
details and results are provided 
in the appendices.
% in~\cite{ourArxiv}.

%reinforcement learning (RL)

%role of control theory in robotics

% notation: $[k] := \{ 1, \ldots, k \}$

%===============================================================================

% \subsection{Related Work}
% \label{sec:related_work}

%===============================================================================

% \section{Control-Based Reinforcement Learning 
% % Framework}
% Approach}
\section{CBRL 
% Framework}
Approach}
\label{sec:CBRLapproach}
Consider a standard RL framework (see, e.g.,~\cite{Sutton1998, BerTsi96}) in which a decision-making agent interacts with a stochastic environment modeled as 
% a Markov Decision Process (MDP) 
an MDP
defined over a set of states $\setX$, a set of actions $\setA$, a transition probability kernel $\pr$, and a reward function $\rewardR$ mapping state-action pairs to a bounded subset of $\Reals$. 
%We denote a probability distribution on $\setA$ as $\setPA$.
% ; i.e., $\setPA$ is an element of $[0,1]^{\setA}$ whose values sum to $1$. 
% \blue{SP: Most of our experiments deal with continuous spaces. Should we write the definitions of the MDP for this setting instead?}
% {\color{red} MSS/CWW: we have two responses.  one is that, strictly speaking, nothing in the above states that $\bbX$ and $\bbA$ are finite.  second, if you prefer, we can add "where $\bbX$ and $\bbA$ may be finite or infinite."  given space constraints, we are leaning more toward the former, but we are fine implementing the latter if you think this is best.}
The agent seeks to determine a policy $\pi(\cdot|x): \setX\ra \setA$ comprising the sequence of control actions $u_1, u_2, \ldots$ that maximizes a
discounted infinite-horizon stochastic dynamic programming (sDP) formulation expressed as 
\begin{equation}
\max_{u_1, u_2, \ldots} \ex_\pr\bigg[ \sum_{t=0}^\infty \gamma^t \rewardR(x_t, u_t) \bigg] ,
% \qquad \mbox{s.t.} \;\;
% x_{t+1} = f(x_t, u_t),
\label{eqn:general_control:max}
\end{equation}
where
% $\ex_{\pr}$ denotes expectation over all trajectories
% w.r.t.\ 
% % with respect to
% % the underlying probability kernel 
% $\pr$,
$x_t \in \setX$ denotes the state of the system at time $t\in\Ints_+$,
$u_t \in \setA$ the control action at time $t$,
% $f(\cdot, \cdot)$ the evolution function of the underlying stochastic dynamical system,
% and 
$\gamma \in (0,1)$ the discount factor,
and
$\ex_{\pr} := \ex_{x_{t+1} \sim \pr(\cdot | x_t,u_t)}$ defines expectation
w.r.t.\ 
% with respect to 
the conditional transition probability for the next state
$\pr(\cdot | x_t,u_t)$.
% Here $\pr(\cdot | x_t,u)$ denotes the conditional transition probability for the next state $x_{t+1}$
% when control action $u$ is taken in state $x_t$ at time $t$.
% For control action $u$ taken in state $x_t$ at time $t$,
% % (i.e., $(x_t,u) \in \setX \times \setA$), 
% denote by $\pr(\cdot | x_t,u)$ the conditional transition probability for the next state $x_{t+1}$ and precisely define $\ex_{\pr} := \ex_{x_{t+1} \sim \pr(\cdot | x_t,u)}$ to be expectation
% w.r.t.\ 
% % with respect to 
% $\pr(\cdot | x_t,u)$. \blue{SP: If we define the transitions as a conditional transition probability (which I'm okay with), I think we should remove the s.t. from \eqref{eqn:general_control:max}}
% % A 
%The
%stationary policy $\pi(\cdot|x): \setX\ra \setPA$ defines a distribution of available control actions given the current state $x$, 
%, which reduces to 
%a deterministic policy $\pi(\cdot|x): \setX\ra \setA$ when the conditional distribution renders a constant action for state $x$;
The
deterministic policy $\pi(\cdot|x): \setX\ra \setA$ defines the control action given the current state $x$,
for which we simply
% .
% With
% slight 
% abuse of notation, we 
% % always 
write
% policy 
$\pi(x)$. 
Let $\setQ(\setX \times \setA)$ denote the space of bounded real-valued functions over $\setX \times \setA$ with supremum norm,
and
let $\Pi$ denote the set of all
% deterministic 
stationary 
policies.
% $\setX\rightarrow \setA$.
% % Let 
% and let
% $\pi^\ast(x)$ denote the optimal stationary policy.
From standard
% RL 
theory, define the Bellman operator
% $\BellmanOp:\setQ\times\setQ$ 
$\BellmanOp$ on $\setQ(\setX \times \setA)$
as
% \begin{align}
% \BellmanOp Q(x,u) := \rewardR(x,u) + \gamma \ex_\pr \max_{u^\prime \in \setA} Q(x^\prime, u^\prime), 
% \label{eq:BellOp}
% \end{align}
\begin{equation}
\BellmanOp Q(x,u) := \rewardR(x,u) + \gamma \ex_\pr \max_{u^\prime \in \setA} Q(x^\prime, u^\prime) ,
\label{eqn:BellmanOperator}
\end{equation}
where
% with $x^\prime$ denoting the next state upon transition from $x$. 
$x^\prime$ denotes the next state upon transition from $x$;
we note that~\eqref{eqn:BellmanOperator} is also referred to in the research literature as the Bellman optimality operator.
Define $Q^\ast(x,u)$, $V^\ast(x):=\max_{u\in\setA} Q^\ast(x,u)$, and $\pi^\ast(x) := \argmax_{u\in\setA} Q^\ast(x,u)$ to be the optimal action-value function, the optimal value function, and the optimal deterministic stationary policy, respectively.
% ,
% with $Q^\ast(x,u)$ the unique fixed point of $\BellmanOp$, a contraction in supremum norm.
% % \blue{SP: If you don't mind I would like to change the notation $Q^\prime$ to $Q^\star$.}
% % \red{MSS: you are absolutely correct, it should be $Q^\ast$ and $V^\ast$ for the optimal action-value and optimal value functions.  i'm not sure why we used ' here.  something might have gotten lost when we shrunk things down.  i need to check on this.}
% Let $Q^\ast(x,u)$ denote the optimal action-value function, $V^\ast(x):=\max_u Q^\ast(x,u)$ the optimal value function, and $\pi^\ast(x) =\argmax_u Q^\ast(x,u)$ the optimal control action.
% % We note that formulation~\eqref{eqn:general_control:max} 
% Formulation~\eqref{eqn:general_control:max} 
% can in general represent a wide variety of 
% % stochastic dynamic programs 
% sDPs
% associated with RL problems based on 
% % the 
% different forms
% % taken by 
% of
% the 
% % % evolution 
% % function $f(\cdot,\cdot)$
% % % , together with the transition probability 
% % and
% kernel $\pr$, 
% % which can include characterizing the 
% including
% % evolutionary system 
% dynamics governed by stochastic (partial) differential equations.
% We
% % further 
% note that $Q^\ast(x,u)$ is the unique fixed point of the Bellman operator $\BellmanOp$, a contraction in supremum norm.

Our
% novel 
% control-based RL (CBRL) 
CBRL
approach consists of exploiting control-theoretic methods to solve the general sDP formulation~\eqref{eqn:general_control:max} 
% \red{in terms of the unknown system and control variables,} 
in terms of the corresponding unknown independent variables,
% \blue{control variables referees to the unknown variables of the system?}
% of the underlying
% % unknown 
% stochastic dynamical system
and directly learning these unknown variables 
% of the optimal control policy 
over time through an associated iterative solution process.
In particular, our general approach considers two key ideas:
% to the standard MDP framework: 
($i$)
extending the solution methodology from 
% \blue{is extending the correct word here? We extend the framework to consider a smaller class of policies. But here it seems that the class of policies itself is extended.} 
the family of policies $\Pi$ associated with classical RL to a family of control-policy functions that map a variable vector $v$ from an independent variable set to a control policy that is optimal
% (or approximately optimal) 
under
the independent variable vector 
$v$;
and
($ii$) extending the domain of these control policies from a single state to span across all 
(or a large subset of) 
states in $\setX$, with the independent variable vector $v$ encoding
both 
global and local information
(e.g., gravity and mass)
that 
is unknown and 
needs 
to be learned.
% (e.g.,
% % in strong contrast with a high-dimensional neural network which constructs a global $Q$-function or policy, our low-dimensional variable vector learns 
% global physics quantities such as gravity and mass). 
% More formally,
% \blue{SP: Up to now we are using $\bbX,\bbW$, etc. for sets. Now the set of variables is denoted by $\setV$. We can't change the notation to $\bbP$ because that's probability. Should we denote all sets with calligraphic instead?}
%
% \orange{Just a minor suggestion. Should we use a different notation for control policy or the family of control policy functions? Now it is easy to misunderstand that $\cF_v \in \bbF$ as they both use the letter $F$. But if it needs to change a lot, probably it doesn't worth it.}

More formally,
let $\setV$ 
% be a subset of a metric space (e.g., Euclidean spaces) that serves as 
denote
a subset of a metric space serving as
the set of 
variable vectors,
and
% let 
$\bbF$ 
% denote 
the family of control-policy functions
% (i.e., family of optimal control policies derived from $\setV$) 
whose elements comprise surjective functions $\cG : v\mapsto \cF_v$ that map a variable vector $v\in \setV$ to a control policy $\cF_v :\setX\rightarrow \setA$ that is optimal under $v$.
Let $\tilde{\Pi}\subseteq \Pi$ denote the set of all stationary policies that are directly determined by the control-policy functions $\cG(v)$ over all $v\in\setV$.
For any $v\in\setV$, the control-policy function $\cG(\cdot) \in \bbF$ derives a particular control policy $\cF_v \in \tilde{\Pi}$ that provides the best 
% rewards in expectation 
expected cumulative discounted reward
across all
% (or a large subset of) 
states $x\in\setX$ from among all control-policy functions in $\bbF$.
Such control policy mappings $\cG:\setV\ra \tilde{\Pi}$ derive optimal control policies from vectors in the
variable set $\setV$ through control-theoretic methods, with the goal of learning the unknown variable vector
$v\in \setV$ 
through a corresponding iterative solution process.
Hence, once the variable vector is learned
with sufficient accuracy, 
the desired optimal control policy is realized.

The control policy mappings
% $\cG(\cdot) \in \bbF$ 
$\cG(\cdot)$
that
derive an optimal control policy for each 
independent-variable vector $v\in\setV$ 
come from a specific control-theoretic framework chosen to be combined as part of our CBRL approach for this purpose.
In practice, this selected control-theoretic framework of interest may be an approximation and not provably optimal for the underlying sDP formulation~\eqref{eqn:general_control:max}, as long as it is sufficiently accurate for the RL task at hand.
% based on a selected control-theoretic framework of interest,
% which may be an approximation and not provably optimal for the underlying sDP formulation~\eqref{eqn:general_control:max}.
As a representative example,
we focus in this paper on 
our CBRL approach in combination with the
% control-theoretic framework of linear quadratic regulator (LQR) 
linear quadratic regulator (LQR) framework
within which the system dynamics are linear and the objective function is quadratic.
More formally, the LQR system dynamics are governed by
% $x_{t+1} = {A}x_t +Bu_t$ 
$\dot{x} = A_{v} x +B_{v} u$
with initial state $x_0$,
where the
% variable 
vector $v$ is contained in the matrices $A_{v}$ and $B_{v}$;
% , and $w_t$ is Gaussian noise with mean zero; 
and the LQR objective function to be 
% maximized
minimized 
is given by
$\int_0^\infty (x_t^\top \underline{Q}_{v} x_t + u_t^\top R_{v} u_t) dt$,
% $\sum_{t=0}^\infty (x_t^\top {\underline{Q}}x_t + u_t^\top {R}u_t)$,
where
the
% variable 
vector $v$ may be contained in the matrices $\underline{Q}_{v}$ and $R_{v}$.
Here the matrices $A_{v}$ and $B_{v}$ respectively define the linear system dynamics and the linear system control; and
the matrices $\underline{Q}_{v} = \underline{Q}_{v}^\top \succeq 0$ and $R_{v} = R_{v}^\top \succ 0$ respectively define the quadratic system 
% rewards 
costs
and the quadratic control 
% rewards.
costs.
% The value function of the corresponding sDP
% % , starting at time $t$, 
% is given by $V(t) = \max_{u} \int_t^\infty (x_s^\top {\underline{Q}}x_s + u_s^\top {R}u_s) ds$.
Within the context of this LQR framework,
the vector $v$ comprises the $d$ unknown real-valued independent variables contained in the LQR problem formulations, and thus the independent variable set
$\bbV \subseteq \Reals^d$;
the family of control-policy functions $\bbF$ comprises all mappings from $\bbV$ to the corresponding set of LQR problem formulations 
% (i.e., $\bbF = \{ \min \int_0^\infty (x^\top \underline{Q}_{v} x + u^\top R_{v} u) dt \: : \: \dot{x} = A_{v} x +B_{v} u , \: \forall v \in \setV \})$;
(i.e., $\bbF = \{ (\min P_v), \: \forall v \in \setV \}$, where $P_v := \int_0^\infty (x^\top \underline{Q}_{v} x + u^\top R_{v} u) dt , \; \dot{x} = A_{v} x +B_{v} u$);
and
the control-policy function $\cG(\cdot) \in \bbF$ maps a given
% estimate of 
$v$ to the specific corresponding LQR problem formulation 
% (i.e., $\cG(v) = \min \int_0^\infty (x^\top \underline{Q}_{v} x + u^\top R_{v} u) dt$, with $\dot{x} = A_{v} x +B_{v} u$ for fixed $v \in \setV$),
(i.e., $\cG(v) = (\min P_v)$, for fixed $v \in \setV$),
whose solution is the linear control policy $\cF_v \in \tilde{\Pi}$ 
% (i.e., $\cF_v = \argmin \int_0^\infty (x^\top \underline{Q}_{v} x + u^\top R_{v} u) dt$, with $\dot{x} = A_{v} x +B_{v} u$ for fixed $v \in \setV$)
(i.e., $\cF_v = (\argmin P_v)$, for fixed $v \in \setV$)
which renders the action $\cF_v(x)$ taken in state $x\in\setX$.

% In the remainder of this section, we first 
% establish theoretical properties of our general CBRL
% % approach and then consider 
% % % an algorithm to compute solutions based on our CBRL approach and its 
% approach, including convergence, optimality and 
% a new form of control-policy-variable gradient ascent
% % methods and corresponding theoretical
% % properties.
% methods,
% and then adapt our CBRL approach to the LQR control-theoretic framework.

In the remainder of this section, we establish various theoretical properties of our CBRL approach
% We next establish theoretical properties of our CBRL approach based on the \red{LQR}\blue{SP: This sentence suggests that everything we will do is LQR.} control-theoretic framework, 
including convergence, optimality, and control-policy-variable gradient ascent.

% \textbf{Convergence and Optimality.}
\subsection{Convergence and Optimality}
\label{ssec:converge-opt}
We consider an analog of the Bellman operator
$\BellmanOp$
% in~\eqref{eqn:BellmanOperator}
within our general CBRL approach
and derive a set of related theoretical results on convergence and optimality.
For each $Q$-function $q(x,u) \in \setQ(\setX\times \setA)$, 
we 
define the
generic 
CBRL function $\tilde{q}:\setX\times \tilde{\Pi}  \rightarrow \Reals$ 
as $\tilde{q}(x,\cF_{v}) := q(x,\cF_{v}(x))$ where the control policy $\cF_v$ is derived from the control-policy function $\cG$ for a given variable vector
$v \in \setV$, 
% $v$, 
and thus
% \red{evidently} \blue{SP: I would remove this word.}
$\tilde{q}(x,\cF_{v})\in \tilde{\setQ}(\setX\times \tilde{\Pi})$ with
$\tilde{\setQ}(\setX \times \tilde{\Pi})$ denoting the space of bounded real-valued functions over $\setX \times \tilde{\Pi}$ with supremum norm.
% \blue{SP: I think that defining the supremum norm and explaining that it needs to include the sup over variables would be important. Especially because we are defining the space of functions $\tilde{\setQ}$ as a function of policies.}
% \red{MSS/CWW: given the relationship between $\bbQ$ and $\tilde{\bbQ}$ with both denoting bounded functions over Cartesian products with supremum norm, we decided to address your comment by being more precise above regarding the "surjective" functions $\cG$.}
We then define our CBRL operator $\ContractionOp$ on $\tilde{\setQ}(\setX\times \tilde{\Pi})$ as
% \begin{align}
% \label{eqn:ContractionOp}
% (\ContractionOp \tilde{q})(x,\cF_v) :=  \sum_{y\in \setX} \pr_{\cF_v(x)}(x,y) &\Big[r(x,\cF_v(x)) \nonumber \\
% &+ \gamma \sup_{v^\prime \in\setV} \tilde{q}(y,\cF_{v^\prime}) \Big] .
% \end{align}
% % with $\tilde{\setQ}(\setX \times \tilde{\Pi})$ the space of bounded real-valued functions over $\setX \times \tilde{\Pi}$ with supremum norm.
\begin{align}
\label{eqn:ContractionOp}
(\ContractionOp \tilde{q})(x,\cF_v) :=  \sum_{y\in \setX} &
% \pr_{\cF_v(x)}(x,y) 
\; \pr(y \, | \, x, \cF_v) 
%\;\; \times \nonumber \\
%& 
\Big[r(x,\cF_v(x)) + \gamma \sup_{v^\prime \in\setV} \tilde{q}(y,\cF_{v^\prime}) \Big] .
\end{align}
%\blue{SP: The notation $\pr_{\cF_v(x)}(x,y)$ has not been defined. Should we use the conditional transition probability from before?}
%\red{MSS/CWW: great catch.  addressed.}
% \red{
Our approach comprises a 
% progression
sequence of steps 
from the family $\bbF$ to the
control-policy functions $\cG \in \bbF$ to the control policies $\cF_v \in \tilde{\Pi}$ to the control action $\cF_v(x) \in\setA$, upon applying $\cF_v$ in state $x\in\setX$.
It is important to note that the supremum in~\eqref{eqn:ContractionOp} is taken over all independent-variable vectors in $\setV$, where the control-policy function $\cG(v^\prime)$ for each $v^\prime\in\setV$ derives through control-theoretic means the corresponding control policy $\cF_{v^\prime}$,
and thus taking the supremum over $\setV$ in~\eqref{eqn:ContractionOp} is equivalent to taking the supremum over $\tilde{\Pi}$.
% } 
%\blue{SP: In the sentence is red is not very clear what is meant by progression.}
%\red{MSS/CWW: we reworded the sentence in the hopes that it addresses your concerns.}
We next introduce an important assumption 
on the richness of the
% policy 
family $\bbF$ within
% this 
% % progression.
% sequence.
the sequence of steps of our approach from $\bbF$ to $\cF_v(x)$.
% sequence of steps.
%
\begin{assumption}%[Richness of policy family]
\label{asm:richness}
There exist a policy function $\cG$ in the family $\bbF$ and a unique 
variable 
vector $v^\ast$ in the 
independent-variable set $\setV$ such that, for any state $x\in\setX$, 
$\pi^\ast(x) = \cF_{v^\ast}(x)=\cG(v^\ast)(x)$.
% ,
% and $\pi^\ast(x)$ 
% w.r.t.\
% $\BellmanOp$.
\end{assumption}
%\blue{SP: Should we call the optimal vector $v^\ast$ for consistency?}
%\red{MSS/CWW: addressed above.}

% Intuitively, the above assumption
Intuitively,
Assumption~\ref{asm:richness}
says that $\bbF$ is rich enough to include a global control policy that coincides with the Bellman operator for each state.
We then have the following formal result on the convergence of the operator $\ContractionOp$ of our 
% control-based RL 
CBRL
approach and the optimality of this convergence 
% with respect to 
w.r.t.\
the Bellman equation~\citep{Sutton1998, BerTsi96}.
\begin{theorem}\label{contraction}
For any $\gamma \in(0,1)$, the operator $\ContractionOp$ in~\eqref{eqn:ContractionOp} is a contraction in the supremum norm.
% {\color{red}
Moreover, the sequence defined by the iterative process $\tilde{q}_{t+1}= \ContractionOp (\tilde{q}_{t})$ converges with rate $\gamma$ to $\tilde{q}^\ast$ as $t\rightarrow\infty$ such that
\begin{equation*}
\| \tilde{q}_{t} - \tilde{q}^\ast \|_\infty \leq \gamma^t \| \tilde{q}_{0} - \tilde{q}^\ast \|_\infty .
\end{equation*}
% }
Supposing Assumption~\ref{asm:richness} holds for the family of policy functions $\bbF$ and its variable set $\setV$,
the
contraction operator $\ContractionOp$ achieves the same asymptotically optimal outcome as that of
the Bellman operator
$\BellmanOp$.
\end{theorem}

Theorem~\ref{contraction}
% {\color{red}
guarantees
% }
that, under the contraction operator $\ContractionOp$ and Assumption~\ref{asm:richness}, our 
% control-based 
CBRL
approach is optimal by realizing the same unique fixed point of
the Bellman operator
$\BellmanOp$.
% % From the first part of Theorem~\ref{contraction}, 
% From Theorem~\ref{contraction}, 
In particular, from Theorem~\ref{contraction} for any function $\tilde{q} \in \tilde{\setQ}(\setX\times \tilde{\Pi})$, the iterations $\ContractionOp^t(\tilde{q})$ converge
% {\color{red}
exponentially with rate $\gamma$
% }
as $t\rightarrow\infty$ to $\tilde{q}^\ast(x,\cF_v)$, the unique optimal fixed point of the 
CBRL
% contraction 
operator $\ContractionOp$, 
%%  CONFERENCE VERSION: COMMENT FROM HERE
{where $\tilde{q}^\ast(x,\cF_v)$ satisfies
\begin{equation}
\label{eqn:fixed_point}
\tilde{q}^\ast(x,\cF_v) \!= \! \! \sum_{y\in \setX} 
% \! \pr_{\cF_v(x)}(x,y) \! 
\pr(y \, | \, x, \cF_v)
\Big[ r(x,\cF_v(x)) + \gamma \sup_{v^\prime\in\setV} \tilde{q}^\ast(y,\cF_{v^\prime})  \Big].
\end{equation}
}
%% TO HERE
%SP: I think the above expression is not needed.
%
% The second part of Theorem~\ref{contraction} 
% Theorem~\ref{contraction} further
% implies that, under the contraction operator $\ContractionOp$ and Assumption~\ref{asm:richness}, our 
% % control-based 
% CBRL
% approach is optimal by realizing the same unique fixed point of
% the Bellman operator
% $\BellmanOp$.
We note, however, that this optimality is achieved with great reduction in the sample complexity due in part to another important difference 
% between our 
% % control-based RL 
% CBRL
% framework and the standard RL framework, 
with standard RL,
namely the search of our CBRL approach across all
% (or a large subset of) 
states
$x\in\setX$
% $x$ 
to find
% a single 
an
% (or small collection of) optimal variable vector(s) 
optimal independent-variable vector
% $v \in \setV$ 
$v^\ast$ 
that identifies a single 
% (or small collection of) optimal control-policy function(s) 
optimal control-policy function
$\cG^\ast \in \bbF$ 
% $\cG$ 
which coincides with the Bellman equation for each state.
% % For each $Q$-function $q(x,u) \in \setQ(\setX\times \setA)$, we define the generic CBRL function $\tilde{q}:\setX\times \tilde{\Pi}  \rightarrow \Reals$ 
% % as $\tilde{q}(x,\cF_{v}) := q(x,\cF_{v}(x))$ where the control policy $\cF_v$ is obtained from the control-policy function $\cG$ derived from the variable vector
% % $v \in \setV$, 
% % % $v$, 
% % and thus evidently $\tilde{q}(x,\cF_{v})\in \tilde{\setQ}(\setX\times \tilde{\Pi})$.
% % Iterations 
% % % with respect to 
% % w.r.t.\
% % % the operator $\ContractionOp$ of our CBRL 
% % % % framework 
% % % approach
% % our CBRL operator $\ContractionOp$ 
% % then consist of improving the estimates of the variable vector $v$ while applying the optimal control-policy function $\cG$ derived from the current estimate of~$v$.
%
%%  CONFERENCE VERSION: UNCOMMENT FROM HERE
% Moreover, an analogous result
% can be obtained for $Q$-learning; see~\cite[Theorem~2]{ourArxiv}.
%% TO HERE

%
%%  CONFERENCE VERSION: COMMENT FROM HERE
%We next
Now
consider convergence
%of a form
of the $Q$-learning algorithm within the context of our general 
% MDP 
% % framework. In particular, 
% framework, 
CBRL approach,
where
we focus on the following
% form of the 
CBRL version of the
classical $Q$-learning update rule \citep{Watk89}:
\begin{equation}\label{eqn:q-learn-new}
 \tilde{q}_{t+1}(x_t,\cF_{v,t}) = \tilde{q}_{t}(x_t,\cF_{v,t}) + \alpha_t(x_t,\cF_{v,t})\bigg[ r_t + \gamma \sup_{v^\prime\in\setV} \tilde{q}_t(x_{t+1},\cF_{v^\prime}) - \tilde{q}_t(x_t,\cF_{v,t})\bigg],
\end{equation}
for $0 < \gamma < 1$, $0 \leq \alpha_t(x_t,\cF_{v,t})\leq 1$ and iterations $t$. 
Let $\{ \cF_{v,t} \}$ be a sequence of control policies that covers all state-action pairs and $r_t$ the corresponding reward of applying $\cF_{v,t}$ to state $x_t$. We then have the following formal result on $Q$-learning convergence and 
% global optimality.
the optimality of this convergence.
\begin{theorem}
\label{thm:QLearning}
Suppose 
Assumption~\ref{asm:richness} holds for the family of policy functions $\bbF$ and its independent-variable set $\setV$ with a contraction operator $\ContractionOp$ as defined in~\eqref{eqn:ContractionOp}.
If $\sum_t \alpha_t = \infty$, $\sum_t \alpha_t^2 < \infty$, and $r_t$ are bounded, then $\tilde{q}_{t}$ under the $Q$-learning update rule~\eqref{eqn:q-learn-new} converges to the optimal fixed point $\tilde{q}^*$ as $t\rightarrow\infty$
and
the optimal policy function is obtained from a unique variable vector $v^* \in \setV$.
\end{theorem}
%% TO HERE

Iterations $t$
% with respect to 
w.r.t.\
either the operator $\ContractionOp$ in~\eqref{eqn:ContractionOp} or the $Q$-learning update rule
%%  CONFERENCE VERSION: COMMENT FROM HERE
in~\eqref{eqn:q-learn-new}
%% TO HERE
then consist of improving the estimates of the independent-variable vector $v$ while applying the optimal control policy derived from the optimal control-policy function $\cG$ for the current variable vector estimate $v_t$ based on the control-theoretic framework of interest.
Within the context of the LQR control-theoretic framework as a representative example, 
it is well known from classical control theory
that the solution of the corresponding sDP is determined by solving the algebraic Riccati equation (ARE)~\citep{riccati:1995}, whose 
continuous-time 
version (CARE) is
expressed as
\begin{equation} \label{eqn:care}
A_{v}^\top P + P A_{v} - P B_{v} R_{v}^{-1} B_{v}^\top P + \underline{Q}_{v} = 0 .
\end{equation}
The optimal control policy action at iteration $t$,
in system state $x_t$ with independent-variable vector estimate $v_t$,
is then obtained from the solution $P$ of the CARE~\eqref{eqn:care}
%continuous-time ARE (CARE) 
as $\cF_{v_t}(x_t) = -(R_{v_t}^{-1} B_{v_t}^\top P) x_t$, together with the change of coordinates $\tilde{x} = x-x^*$ in general when the target $x^*$ is not the origin.

Before addressing in
% the next subsection 
Section~\ref{ssec:PGM}
one
such iterative process to efficiently estimate the variable vector $v$, 
it is important to note that the control-theoretic framework selected 
to be combined as part
of our CBRL approach need not be provably optimal for the RL task at hand.
In particular, as long as the selected control-theoretic framework is sufficiently accurate, our
% CBRL 
approach can yield an approximate solution within a desired level of accuracy.
To further support framework selection for such approximations, 
we first relax some of the above conditions to consider
%control-policy function 
families $\bbF$ that do not satisfy Assumption~\ref{asm:richness} and consider
control-policy 
% policy
functions that map a variable vector to a control policy whose domain spans across a large subset of states (as opposed to all states) in $\setX$ and renders asymptotically optimal
(as opposed to globally optimal)
rewards.
% (as opposed to globally optimal rewards).
% in expectation.
% As such, our search now is across large subsets of states in $\setX$ to find a small collection of optimal variable vectors $v^\prime$ that derives a small collection of optimal control-policy functions $\cG^\prime$, which may or may not coincide with the Bellman equation for each state depending upon the richness of the family of control-policy functions.
% % Throughout this subsection, let 
Supposing $\bbF$
% is sufficiently rich to satisfy 
satisfies
Assumption~\ref{asm:richness},
consider 
% our general CBRL framework under two 
a sequence of less rich families $\bbF_1, \bbF_2, \ldots, \bbF_{k-1} , \bbF_k$ of policy functions $\cG^{(i)}\in\bbF_i$ obtained from independent-variable vectors of the corresponding 
independent variable 
sets $\setV_i$, 
% $i\in [k]$, 
and define the operators $\ContractionOp_i:\tilde{\setQ}(\setX\times \tilde{\Pi}_i) \rightarrow \tilde{\setQ}(\setX\times \tilde{\Pi}_i)$ as in~\eqref{eqn:ContractionOp} for any function
$\tilde{q}_i(x,\cF_v^{(i)})\in \tilde{\setQ}(\setX\times \tilde{\Pi}_i)$, $i\in [k] := \{ 1, \ldots, k \}$.
Then, from
% From 
Theorem~\ref{contraction},
each operator $\ContractionOp_i$ under variable set $\setV_i$ is a contraction in the supremum norm and converges
% asymptotically 
to the unique fixed point $\tilde{q}_i^\ast(x,\cF_v^{(i)})$ that satisfies the corresponding version of~\eqref{eqn:fixed_point},
for all $x\in\setX$ and $v\in \setV_i$, $i\in [k]$.
We then have the following formal result on the asymptotic optimality of our CBRL approach in such approximate settings.

\begin{theorem}\label{thm:PWL}
Assume the state and action spaces are compact and $\cF_v$ is uniformly continuous for each $v$. 
Consider $\bbF$ and a sequence of families of policy functions $\bbF_1,  \bbF_2, \ldots, \bbF_{k-1}, \bbF_k$, with $\setV$ and $\setV_i$ respectively denoting the independent-variable sets corresponding to $\bbF$ and $\bbF_i$, $i\in [k]$.
Let $d(\cdot,\cdot)$ be a $\sup$-norm distance function defined over the policy space $\tilde{\Pi}$, i.e., $d(\cF^{(i)}_{v_i},\cF^{(j)}_{v_j}) = \sup_{x\in\setX} \| \cF^{(i)}_{v_i}(x) - \cF^{(j)}_{v_j}(x) \|$, $v_i\in\setV_i, v_j\in\setV_j$.
Further let $d^\prime(\cdot,\cdot)$ be an $\inf$-norm distance function defined over the policy function space $\bbF$, i.e., $d^\prime(\cG^{(i)},\cG^{(j)}) = \inf_{v_i\in\setV_i,v_j\in\setV_j} d(\cG^{(i)}(v_i) , \cG^{(j)}(v_j))$, $\cG^{(i)}\in\bbF_i, \cG^{(j)}\in\bbF_j$.
Suppose, for all $\cG \in \bbF$, there exists a $\cG^{(k)}\in\bbF_k$ such that $d^\prime(\cG^{(k)},\cG) \rightarrow 0$ as $k\rightarrow \infty$.
Then,
% $\sup_{x\in\setX} \|\tilde{q}_k^\ast-\tilde{q}^\ast\| \rightarrow 0$ 
$\|\tilde{q}_k^\ast-\tilde{q}^\ast\|_\infty \rightarrow 0$ 
as $k\rightarrow \infty$.
\end{theorem}
%
% \orange{Isn't $\bigcup_{i=1}^k \bbF_i$ just $\bbF_k$? given the fact that $\bbF_1 \subset \bbF_2 \subset \cdots \subset \bbF_{k-1} \subset \bbF_k$}

An analogous version of Theorem~\ref{thm:PWL} based on $Q$-learning can be established in a similar manner w.r.t.\ Theorem~\ref{thm:QLearning}, where the construction of the corresponding update rules 
%%  CONFERENCE VERSION: COMMENT FROM HERE
in~\eqref{eqn:q-learn-new} 
%% TO HERE
under the variable sets $\setV_i$ parallels the above construction of the operators $\ContractionOp_i$ under the same variable sets $\setV_i$.
%
% Hence, under relaxed conditions 
% % with respect to 
% w.r.t.\
% Assumption~\ref{asm:richness}, our general 
% % control-based RL 
% CBRL
% approach can provide asymptotically optimal rewards in expectation, where the degree of approximation depends upon how close the families $\bbF_i$, $i\in [k]$, are to each other and to $\bbF$.
%
% One 
In the cases of both
Theorem~\ref{thm:PWL} and its $Q$-learning variant,
one
specific instance of a sequence of
the 
families 
of policy functions 
$\bbF_1 , \ldots, \bbF_k$
% in Theorem~\ref{thm:PWL}
% or its $Q$-learning variant
consists of piecewise-linear control policies of increasing richness
(e.g., 
the 
class of canonical piecewise-linear functions 
considered in 
\citep{Lin1992}) 
% with respect to 
w.r.t.\
finer and finer granularity of the control policy function space converging towards $\bbF$.

% \textbf{Control-Policy-Variable Gradient Ascent.}
\subsection{Control-Policy-Variable Gradient Ascent}
\label{ssec:PGM}
Turning now to our CBRL iterative process within the context of policy gradient methods, we build on our foregoing results to establish theoretical results analogous to the standard policy gradient theorem that establishes, for policies parameterized by
% a vector 
$\theta$,
% \blue{$\in \mathbb{R}^d$?}, 
% \red{MSS/CWW: we prefer to leave this generic as there are many policy gradient theorems out there. 
% also, we do not want to introduce new notation such as $d$ given that this point is solely intended to remind the reader about the connection between $\theta$ and the policy in standard policy gradient theorems, explain why this does not apply in our case, and then compare this with what our policy gradient theorem does.}
a connection between the gradient of the value function 
% with respect to 
w.r.t.\
$\theta$ and the gradient of the policy action
% with respect to 
w.r.t.\
$\theta$.
This standard theoretical result does not apply to our 
% control-based RL 
CBRL
approach
because there is no direct connection between the independent-variable vector $v\in \setV$ and the 
% control action $u\in\setA$.
control policy action,
as described above.
Rather, under our
% control-based RL 
CBRL
approach
and for any 
% variable 
vector
% $v\in\setV$, 
$v$, 
the control-policy function $\cG(v) \in \bbF$ derives a particular control policy $\cF_v \in \tilde{\Pi}$, and then the
% control 
policy applied in any state $x\in \setX$ yields the
% control 
action
taken.
% $\cF_v(x) \in \setA$.

So far, we have considered a deterministic control policy for the action taken at each state, since it is well known that a deterministic stationary policy is optimal for our sDP~\eqref{eqn:general_control:max}; refer to~\cite{Puterman2005}.
%It is known that for infinite horizon discounted MDPs, there exists a stationary and deterministic policy such that it is optimal for all starting state simultaneously.
% In the study of policy gradient, it is more appropriate to consider a control policy where the next action follows a probability distribution.
For consistency with standard policy gradient methods, however, we consider in this subsection a general stochastic control policy that follows a probability distribution for the action taken at each state, where a deterministic policy is a special case.
To this end,
let us denote by $\setPA$ a probability distribution on the set of actions $\setA$ where
%Recall that 
the stationary control policy $\cF_v : \setX\rightarrow \setPA$ defines a probability distribution over all available control actions $u\in \setA$ given the current state $x\in \setX$.
Further define $\cF_{v,u}(x)$ to be the element of $\cF_{v}(x)$ corresponding to action $u$,
% (i.e., the probability that control action $u\in \setA$ is taken when the system is in state $x\in \setX$ under control policy $\cF_v$), 
i.e., $\cF_{v,u}(x) := \pr[u_t = u \, | \, x_t = x, \cF_v]$,
and correspondingly define $\tilde{q}_u(x,\cF_v) := q_u(x,\cF_v(x)) := q^{v}(x,u)$, where the latter denotes the $Q$-function of the policy $\cF_v$.
Let $x_0$ denote the start state at time $t=0$ and 
% $\pr_{\cF_v} (x_0,x,k)$ 
$\pr(x, k \, | \, x_0, \cF_v)$
the probability of going from state $x_0$ to state $x$ in $k$ steps under the control policy $\cF_v$.
Our general control-policy-variable gradient ascent result is then formally expressed as follows.

% \begin{theorem}
% Consider a family of control policy functions $\bbF$ and its variable set $\setV$ with contraction operator $\ContractionOp$ in the form of~\eqref{eqn:ContractionOp}.
% We then have
% \begin{align}\label{eqn:nablavp}
%     \nabla_v V_{{\cF_{v}}} (x_0) = & \sum_{x\in \setX} \: \sum_{k=0}^\infty \gamma^k \pr_{\cF_v} (x_0,x,k) \;\; \times \nonumber \\
%     & \quad \left[ \sum_{u \in \setA} \frac{\partial \cF_{v,u}(x)}{\partial \cF_v} \frac{\partial\cF_v(x)}{\partial v} \tilde{q}_u(x,{\cF}_{v}) \right],
% \end{align}
% where it is assumed that 
% $\cF_{v,u}(x)$ is differentiable 
% % with respect to 
% w.r.t.\
% $\cF_v$ and $\cF_v(x)$ is differentiable 
% % with respect to 
% w.r.t.\
% $v$, i.e.,
% both
% $\partial \cF_{v,u}(x) / \partial \cF_v$,
% and 
% $\partial \cF_v(x) / \partial v$
% exist.
% \label{thm:PGT}
% \end{theorem}

\begin{theorem}
Consider a family of control policy functions $\bbF$, its independent-variable set $\setV$ with contraction operator $\ContractionOp$ in the form of~\eqref{eqn:ContractionOp}, and
the value function $V_{{\cF_{v}}}$
under the control policy $\cF_v$.
Assuming $\cF_{v,u}(x)$ is differentiable 
% with respect to 
w.r.t.\
$\cF_v$ and $\cF_v(x)$ is differentiable 
% with respect to 
w.r.t.\
$v$,
we then have
\begin{align}\label{eqn:nablavp}
    \nabla_v V_{{\cF_{v}}} (x_0) = & \sum_{x\in \setX} \: \sum_{k=0}^\infty \gamma^k \,
    % \pr_{\cF_v} (x_0,x,k) 
    \pr(x, k \, | \, x_0, \cF_v)
    \; \left[ \sum_{u \in \setA} \frac{\partial \cF_{v,u}(x)}{\partial \cF_v} \frac{\partial\cF_v(x)}{\partial v} \tilde{q}_u(x,{\cF}_{v}) \right] .
\end{align}
\label{thm:PGT}
\end{theorem}
The corresponding gradient ascent result for the 
% special 
case of deterministic control policies follows directly from Theorem~\ref{thm:PGT} with the conditional probability distribution $\cF_{v,u}(x)$ given by the associated indicator function $\bbI_{\cF_v(x)}(u)$, returning $1$ if and only if $u = \cF_v(x)$ and zero otherwise.

Following along
similar 
lines 
of the various forms of policy-gradient ascent methods based on the standard policy gradient theorem, we devise control-policy-variable gradient ascent methods within the context of our 
% control-based RL 
CBRL
approach based on Theorem~\ref{thm:PGT}.
% Since
% % the applications considered in 
% Sections~\ref{sec:adaptation} and~\ref{sec:experiments} consider a particular family of deterministic control policies, our focus is on this special case of Theorem~\ref{thm:PGT} noting that similar methods can be readily derived for the general case.
One such 
% approach for the corresponding 
gradient ascent method comprises an iterative process for directly learning the unknown 
independent-variable 
vector $v$ of the optimal control policy 
% with respect to 
w.r.t.\
the value function $V_{\cF_v}$
whose iterations
% , 
% with stepsize $\eta$ and $\nabla_v V_{{\cF_{v}}}$ given by~\eqref{eqn:nablavp}, 
proceed according to
\begin{equation}\label{eq:policy-gradient}
    v_{t+1} = v_{t}+\eta  \nabla_v V_{{\cF_{v}}} (x_t) ,
\end{equation}
where $\eta$ is the step size and
$\nabla_v V_{{\cF_{v}}}$ is given by~\eqref{eqn:nablavp}.
% where $\nabla_v V_{{\cF_{v}}}$ is as given by~\eqref{eqn:nablavp}. 
% where the first gradient term ${\partial V}/{\partial {\cF}_{v_t}}$ is essentially the standard policy gradient, and the second gradient term ${\partial {\cF}_{v_t}(x)}/{\partial {v_t}}$ is specific to our general control-based RL framework.
% We note 
Note
that standard policy gradient methods are special cases of~\eqref{eq:policy-gradient} where
% the variable vector 
$v_t$ is directly replaced by the policy
% $\pi_t$ (see, e.g.,~\cite{JMLR:v22:19-736}).
% % In particular, the special case of $\cal G$ being an identity map with $\frac{\partial V}{\partial {\cF}_{v_t}} \frac{\partial {\cF}_{v_t}}{\partial {v_t}}$ replaced by $\frac{\partial V}{\partial \pi_t}$ corresponds to the direct policy gradient parameterization case in~\cite{JMLR:v22:19-736}. 
$\pi_t$;
% e.g., 
identity map for $\cal G$ and $\frac{\partial V}{\partial {\cF}_{v_t}} \frac{\partial {\cF}_{v_t}}{\partial {v_t}}$ replaced by $\frac{\partial V}{\partial \pi_t}$ corresponds to the direct policy gradient parameterization
in~\cite{JMLR:v22:19-736}.

More precisely, consider the iterative process of our control-policy-variable gradient ascent method above within the context of the LQR control-theoretic framework where at iteration $t$ the system is in state $x_t$ and the variable vector estimate is $v_t$.
The optimal control policy $\cF_{v_t}$ is derived by solving for $P$ in the corresponding CARE~\eqref{eqn:care} and then setting $u_t^\ast = \cF_{v_t}(x_t) = -K x_t$,
where $K = R_{v_t}^{-1} B_{v_t}^\top P$,
together with the change of coordinates $\tilde{x} = x-x^*$ whenever the target $x^*$ is not the origin.
Upon applying the control policy $\cF_{v_t}$ by taking action $u_t^\ast$, we subsequently update the estimate of the variable vector according to iteration~\eqref{eq:policy-gradient} where $\nabla_v V_{{\cF_{v}}}$ is obtained from~\eqref{eqn:nablavp} of Theorem~\ref{thm:PGT}.
In particular, the first partial derivative term on the right hand side of~\eqref{eqn:nablavp} is given by the standard policy gradient solution, and the second partial derivative term on the right hand side of~\eqref{eqn:nablavp} is obtained by differentiation of the
CARE~\eqref{eqn:care}, 
which in turn renders~\citep{kao2020automatic}
\begin{equation} \label{eqn:PDcare}
dP \tilde{A}_{v} + \tilde{A}_{v}^\top dP + (dZ + dZ^\top + d\underline{Q}_{v} + K^\top dR_{v} K) = 0,
\end{equation}
where $\tilde{A}_{v} = A_{v} - B_{v}K$
and $dZ = P (dA_{v} - dB_{v} K)$.

Our CBRL approach 
combined with
the LQR control-theoretic framework concerns the linear dynamics $\dot{x} = A_v x + B_v u$ where $A_v$ and $B_v$ contain elements of the unknown independent-variable vector $v$.
By leveraging known basic information about the LQR control problem at hand, only a relatively small number $d$ of unknown variables need to be learned.
% In our CRBL approach we leverage background knowledge about the task at hand in order to reduce the number of unknown variables. 
For
% robotics and physics 
a wide range of
applications where state variables are derivatives of each other w.r.t.\ time (e.g., position, velocity, acceleration, jerk), the corresponding rows in the $A_v$ matrix consist of a single $1$ and the corresponding rows in $B_v$ comprise zeros.
We exploit this basic information to consider general matrix forms for $A_v$ and $B_v$ that reduce the number of unknown variables to be learned. 
As a representative illustration, the system dynamics when there are two groups of such variables have the form
% given by
given by~\eqref{matrixAB} in 
% % Appendix~\ref{app:adaptation}.
% the appendix.
Appendix~\ref{app:algorithm}.

\section{Experimental Results}
\label{sec:experiments}

% \orange{Weiqin: we are not really using the ``canonical form'' (see Eq. (8)). Should we name it something like ``general form''? we basically parameterize everything except the derivative terms.}

% \red{MSS: to save space, i think we can/should include common information in these introductory paragraphs so that we do not have to repeat them for each of the tasks.  examples are pointing to the appendix for detail, from where the control tasks come, etc.  i commented out such stmts in each of the control task subsections.}
% \orange{Weiqin: Agreed.}

In this section, we conduct numerical experiments to empirically evaluate the performance of our general CBRL approach
% % of Section~\ref{sec:CBRLapproach}
% % % (Section~\ref{sec:CBRLapproach}) 
% adapted to the
% above
combined with
the
LQR control-theoretic framework
of Section~\ref{sec:CBRLapproach},
% according to 
% % Section~\ref{sec:adaptation}, 
% % % (Section~\ref{sec:adaptation}), 
% the previous section,
as summarized in Algorithm~\ref{alg:control_rl}
of Appendix~\ref{app:algorithm}.
% of the appendix.
The objective~\eqref{eqn:general_control:max} seeks to maximize the expected cumulative discounted reward,
simply referred to as the expected return.
Our numerical experiments consider several classical
% control 
RL
tasks from Gymnasium~\citep{towers_gymnasium_2023}, including 
\textit{Cart Pole}, 
\textit{Lunar Lander (Continuous)},
\textit{Mountain Car (Continuous)}, and
\textit{Pendulum}.
We compare our CBRL method with three state-of-the-art RL algorithms,
namely DQN~\citep{MnKaSi+13} for discrete actions, DDPG~\citep{lillicrap2015continuous} for continuous actions and PPO~\citep{schulman2017proximal}, together with a variant of PPO that solely replaces the nonlinear policy of PPO with a linear policy (since we know the optimal policy for some of the problems,
e.g., \textit{Cart Pole}, 
% i.e., \textit{Cart Pole} and \textit{Lunar Lander (Continuous)},
is linear).
% \orange{
These baselines are selected
% due to the fact that they are 
as the state-of-the-art
% RL 
algorithms for solving the 
% control 
RL
tasks under consideration.
% and \orange{that the linear policy that can solve e.g., \textit{CartPole-v0} should exist as well.}
% % } 
% \red{MSS: i don't understand what the end of the previous sentence is saying.  what is ``and that the linear policy that can solve e.g., \textit{CartPole-v0} should exist as well.''} \orange{Weiqin: Ideally, one can just use a linear policy to solve the problem of CartPole. This is why we want to include it as a baseline. I might not explain this clearly.}
Experimental details and additional results
% , including validation of the robustness of our CBRL approach, 
are provided in 
Appendix~\ref{app:experiments}
% the appendix
% , including the experimental set up 
both in general and
for each
% control 
RL
task.

% \red{MSS: we need to define ``the linear policy''} \orange{Weiqin: I explained a bit above. Does it look good?}
% \red{MSS: yes, this is good.  but please provide a more detailed description in the appendix.  note that we have an extra week for that.} \orange{Weiqin: Sounds good.}

We have chosen the LQR control-theoretic framework 
to be combined as part
% within the context 
of our CBRL approach
% as a representative example 
with the understanding that not all 
of the above
% control 
RL
tasks can be adequately 
solved using LQR even if the variable vector $v$ is known.
Recall, however, that our CBRL approach allows the domain of the control policies to span a subset of 
states in 
$\setX$, thus enabling partitioning of the state space so that 
% appropriately 
properly
increased richness w.r.t.\ finer and finer granularity can provide improved approximations and 
asymptotic optimality according to Theorem~\ref{thm:PWL},
% of Section~\ref{sec:CBRLapproach}, 
analogous to 
the class of 
canonical piecewise-linear function approximations~\citep{Lin1992}. 
While \textit{Cart Pole} and \textit{Lunar Lander
(Continuous)} 
can be directly addressed within 
the context of 
LQR, this is not the case for \textit{Mountain Car
(Continuous)} 
and \textit{Pendulum}
% in 
which 
% the bounded control input requires
require a nonlinear controller.
We therefore partition the state space in the case of such 
% control 
RL
tasks and consider a corresponding piecewise-LQR controller where the learned variable vectors may differ or be shared across the partitions.
% \orange{Weiqin: it depends on how we define the variable vector. Use mountaincar as an example, if the variable vector includes $a_0 \sin(3s_0)$, then yes, the variable is different in each partition. But if we only consider $[a_0, a_1, b_0, c_0]$ as the variable vector, then all partitions share the same variable vector.}

All variables of our CBRL approach are randomly initialized uniformly within $(0,1)$ and then
% We initialize all variables in \eqref{matrixAB_cartpole} to uniform (0, 1), which are 
learned using our 
control-policy-variable 
% policy-variable
gradient ascent iteration~\eqref{eq:policy-gradient}.
% in an iterative manner using the variable policy gradient presented in \eqref{eq:policy-gradient} 
% (see Fig.~\ref{fig_cp_episode}~--~\ref{fig_cp_variable} (c)). With this fixed initialization, we run our CBRL approach over five random seeds. We further validate the robustness of our CBRL algorithm w.r.t.\ the initialization in Appendix~\ref{app:cartpole}, where three more sets of random initialization are considered and compared.
%
We consider four sets of initial variables to validate the robustness of our CBRL approach for each RL task (see Appendix~\ref{app:experiments}).

\subsection{\textit{Cart Pole} under CBRL LQR} 
\label{ssec:cartpole}
%

% \paragraph{Environment.}
% We consider the environment of the \textit{CartPole-v0} from Gymnasium~\cite{towers_gymnasium_2023}. 
\textit{Cart Pole},
as depicted in Fig.~\ref{cartpole_env} of 
% the appendix,
Appendix~\ref{app:cartpole},
% is a classical control problem consisting 
consists
of a pole connected to a horizontally moving cart with the goal of balancing the pole by applying a force on the cart to the left or right.
The state of the system is $x=[p, \dot{p}, \theta, \dot{\theta}]$ in terms of the position of the cart $p$, the velocity of the cart $\dot{p}$, the angle of the pole $\theta$, and the angular velocity of the pole $\dot{\theta}$.
With upright initialization of the pole on the cart, each episode comprises $200$ steps and the problem is considered ``solved'' upon achieving an average return of $195$.

% \paragraph{Controller.} 
The LQR controller can be used to solve the \textit{Cart Pole} problem provided that all the variables of the system are known (e.g., mass of cart, mass and length of pole, gravity), where the angle is kept small by our CBRL approach
% \red{adapted to}
combined with
the LQR framework.
%\purple{and the angle remains small.}
We address the problem within the context of our CBRL approach by exploiting the 
% canonical (switch to word that we use in Section 3) 
general matrix
form for the LQR dynamics given by~\eqref{matrixAB_cartpole} 
% (Appendix~\ref{app:cartpole}), 
in 
% the appendix,
Appendix~\ref{app:cartpole},
solely taking into account general physical relationships
(e.g., the derivative of the angle of the pole is equivalent to its angular velocity) and laws
(e.g., the force can only affect the acceleration), with the $d=10$ unknown variables $a_0, \ldots, a_7, b_0, b_1$ to be learned; refer to Fig.~\ref{fig_cp_variable}.
%
% %%%%%%%%%%%%%%%%%%%%%%%%%%%%%%%%%%%%%%%%
% %%%%%%%%%%%%%%%%%%%%%%%%%%%%%%%%%%%%%%%%
% %%%%%%%%%%%%%%%%%%%%%%%%%%%%%%%%%%%%%%%%
% %%%%%%%%%%%%%%%%%%%%%%%%%%%%%%%%%%%%%%%%
% \begin{figure}[H]
% % \vspace*{-0.4in}
% \centering 
% \setcounter{subfigure}{0}
% \subfigure[Return vs.\ Episode]
% {
% \begin{minipage}{0.30\linewidth}
% \centering    
% \includegraphics[width=0.9\columnwidth]{figures/cp_return_episode_p1.png}  
% \end{minipage}
% }
% \subfigure[Return vs.\ Time]
% {
% \begin{minipage}{0.30\linewidth}
% \centering    
% \includegraphics[width=0.9\columnwidth]{figures/cp_return_CPUtime_p1.png}  
% \end{minipage}
% }
% \subfigure[Variable of CBRL]
% {
% 	\begin{minipage}{0.30\linewidth}
% 	\centering 
% 	\includegraphics[width=0.9\columnwidth]{figures/cp_parameter_c_2_large.png}  
% 	\end{minipage}
% }

% \caption{Learning curves of \textit{CartPole-v0} over five independent runs. Our CBRL approach is in comparison with the Linear policy, PPO, and DQN.} 
% \label{fig_learn_curve_cartpole}
% \end{figure}
% %

% \paragraph{Numerical Results.}
Fig.~\ref{fig_cp_episode}~--~\ref{fig_cp_variable} and Table~\ref{tab_mean_cp}
% (appendix) 
(Appendix~\ref{app:cartpole})
present numerical results for the three state-of-the-art baselines (discrete actions) and our CBRL approach, with each run over five independent random seeds;
% Having introduced the controller, we are in the stage of demonstrating the performance of our approach compared with all baselines. 
%
% All variables in~\eqref{matrixAB_cartpole} are initialized uniformly within $(0,1)$ and then
% % We initialize all variables in \eqref{matrixAB_cartpole} to uniform (0, 1), which are 
% learned using our 
% control-policy-variable 
% % policy-variable
% gradient ascent iteration~\eqref{eq:policy-gradient}.
% % in an iterative manner using the variable policy gradient presented in \eqref{eq:policy-gradient} 
% % (see Fig.~\ref{fig_cp_episode}~--~\ref{fig_cp_variable} (c)). With this fixed initialization, we run our CBRL approach over five random seeds. We further validate the robustness of our CBRL algorithm w.r.t.\ the initialization in Appendix~\ref{app:cartpole}, where three more sets of random initialization are considered and compared.
% %
% We consider four sets of initial variables to validate the robustness of our CBRL approach; see Table~\ref{tab_init_para_cp} and Fig.~\ref{fig_all_curve_cartpole} in the appendix.
%
see Table~\ref{tab_init_para_cp} and Fig.~\ref{fig_all_curve_cartpole} in 
% the appendix 
Appendix~\ref{app:cartpole}
for the four sets of initial variables in~\eqref{matrixAB_cartpole}.
Fig.~\ref{fig_cp_episode}~--~\ref{fig_cp_second}, Table~\ref{tab_mean_cp} and Fig.~\ref{fig_all_curve_cartpole}(a)~--~\ref{fig_all_curve_cartpole}(b) 
clearly demonstrate that our CBRL approach 
% outperforms 
provides far superior performance (both mean and standard deviation) over all baselines w.r.t.\ both the number of episodes and running time, in addition to demonstrating a more stable training process.
Fig.~\ref{fig_all_curve_cartpole}(c)~--~\ref{fig_all_curve_cartpole}(f) illustrates the learning behavior of CBRL variables given different initialization.

\subsection{\textit{Lunar Lander
%(Continuous)
} under CBRL LQR} 
\label{ssec:lunarlander}
%

% \paragraph{Environment.}
% We consider the environment of \textit{LunarLanderContinuous-v2} from Gymnasium~\cite{towers_gymnasium_2023}.
\textit{Lunar Lander (Continuous)}, as depicted in Fig.~\ref{lunarlander_env} of 
% the appendix, 
Appendix~\ref{app:lunarlander},
is a classical spacecraft trajectory optimization problem with the goal to land softly and fuel-efficiently on a landing pad by applying thrusters to the left, to the right, and upward. The state of the system is $x=[p_x, v_x, p_y, v_y, \theta, \dot{\theta}]$ in terms of the $(x,y)$ positions $p_x$ and $p_y$, two linear velocities $v_x$ and $v_y$, angle $\theta$, and angular velocity $\dot{\theta}$.
Starting at the top center of the viewport with a random initial force applied to its center of mass, the lander (spacecraft) is subject to gravity, friction and turbulence while surfaces on the ``moon'' are randomly generated in each episode with the target landing pad centered at $(0,0)$. The problem is considered ``solved'' upon achieving an average return of $200$.
The LQR controller can be used to solve the \textit{Lunar Lander (Continuous)} problem provided that all variables of the system are known (e.g., mass of lander, gravity, friction, etc.), where the angle is kept small by our CBRL approach 
% % \red{adapted to}
% adapted to
combined with
the LQR framework.
%\purple{, and the angle remains small}.
We address the problem within the context of our CBRL approach by exploiting the 
% canonical
% (switch to word that we use in Section 3) 
general matrix
form for the LQR dynamics given by~\eqref{matrixAB_lunarlander} 
% (Appendix~\ref{app:lunarlander}), 
in 
% the appendix,
Appendix~\ref{app:lunarlander},
solely taking into account general physical relationships (akin to \textit{Cart Pole}) and mild physical information from the system state (e.g., the acceleration is independent of the position), with the $d=15$ unknown variables $a_0, \ldots, a_{11}, b_0, b_1, b_2$ to be learned; refer to Fig.~\ref{fig_ll_variable}.

Fig.~\ref{fig_ll_episode}~--~\ref{fig_ll_variable}  and Table~\ref{tab_mean_ll} 
% (see appendix) 
(Appendix~\ref{app:lunarlander})
present numerical results for the three state-of-the-art baselines (continuous actions) and our CBRL approach, with each run over five independent random seeds;
%
% All variables in~\eqref{matrixAB_lunarlander} are initialized uniformly within $(0,1)$ and then learned using our control-policy-variable gradient ascent iteration~\eqref{eq:policy-gradient}.
% % We initialize all the variables in \eqref{matrixAB_lunarlander} to uniform (0, 1), which are then learned iteratively using the variable policy gradient presented in \eqref{eq:policy-gradient}.
% We consider four sets of initial variables to validate the robustness of our CBRL approach; refer to Table~\ref{tab_init_para_ll} and Fig.~\ref{fig_all_curve_lunarlander} in the appendix.
%
see Table~\ref{tab_init_para_ll} and Fig.~\ref{fig_all_curve_lunarlander} in 
% the appendix 
Appendix~\ref{app:lunarlander}
for the four sets of initial variables in~\eqref{matrixAB_lunarlander}.
Fig.~\ref{fig_ll_episode}~--~\ref{fig_ll_second}, Table~\ref{tab_mean_ll} and Fig.~\ref{fig_all_curve_lunarlander}(a)~--~\ref{fig_all_curve_lunarlander}(b)
clearly demonstrate that our CBRL approach 
% outperforms 
provides far superior performance (both mean and standard deviation) over
all baselines w.r.t.\ both the number of episodes and running time, in addition to demonstrating a more stable training process.
% It is worth noting that the baseline algorithms consume less time than CBRL overall, as they perform worse and the game terminates earlier accordingly.
We note that the baseline algorithms often crash and terminate sooner than the more successful landings of our CBRL approach, resulting in the significantly worse performance exhibited in 
% Figs.~\ref{fig_learn_curve_lunarlander} and~\ref{fig_all_curve_lunarlander} and Table~\ref{tab_mean_ll}.
Fig.~\ref{fig_ll_episode}~--~\ref{fig_ll_second}, Table~\ref{tab_mean_ll} and Fig.~\ref{fig_all_curve_lunarlander}.
Finally, Fig.~\ref{fig_all_curve_lunarlander}(c)~--~\ref{fig_all_curve_lunarlander}(f) illustrates the learning behavior of CBRL variables given different initialization.
\begin{figure}[h]
% \vspace*{-0.4in}
\centering 
\setcounter{subfigure}{0}
\subfigure[Return vs.\ Episode]
{
\begin{minipage}{0.30\linewidth}
\centering    
\includegraphics[width=0.9\columnwidth]{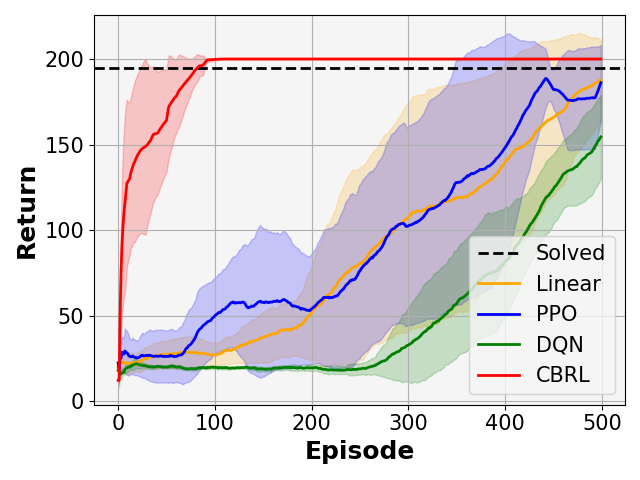}  
\label{fig_cp_episode}
\end{minipage}
}
\subfigure[Return vs.\ Time]
{
\begin{minipage}{0.30\linewidth}
\centering    
\includegraphics[width=0.9\columnwidth]{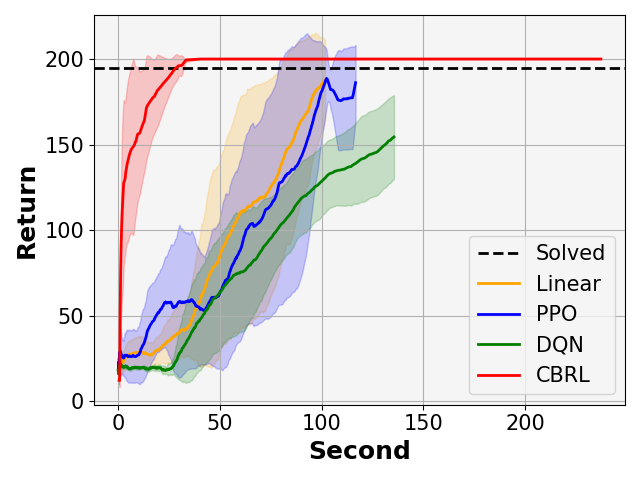}  
\label{fig_cp_second}
\end{minipage}
}
\subfigure[Variable of CBRL]
{
	\begin{minipage}{0.30\linewidth}
	\centering 
	\includegraphics[width=0.9\columnwidth]{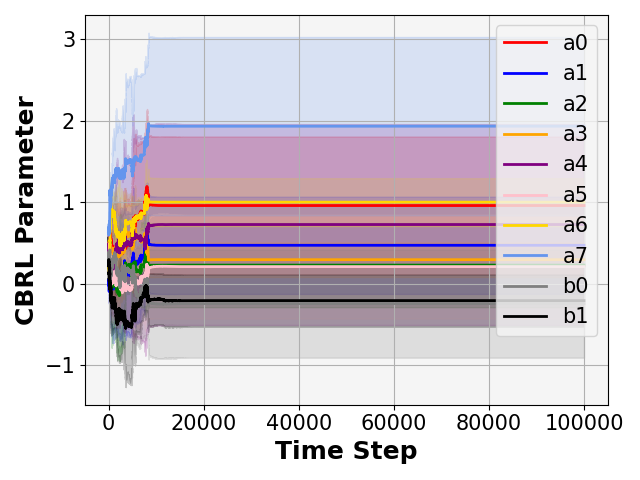}  
    \label{fig_cp_variable}
	\end{minipage}
}
% \label{fig_learn_curve_cartpole}
%
%
\subfigure[Return vs.\ Episode]
{
\begin{minipage}{0.30\linewidth}
\centering    
\includegraphics[width=0.9\columnwidth]{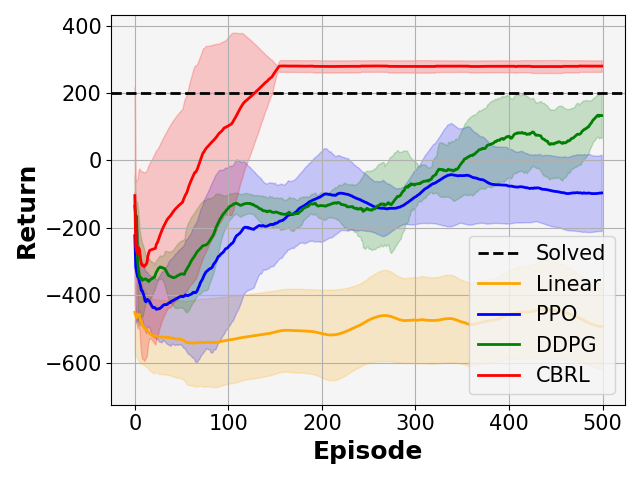}  
\label{fig_ll_episode}
\end{minipage}
}
\subfigure[Return vs.\ Time]
{
\begin{minipage}{0.30\linewidth}
\centering    
\includegraphics[width=0.9\columnwidth]{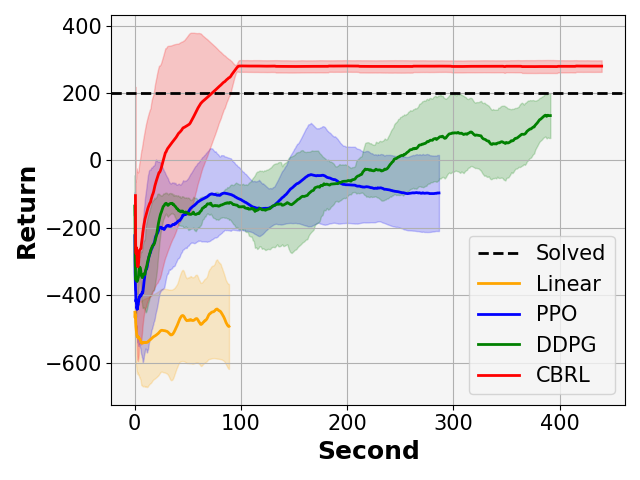}  
\label{fig_ll_second}
\end{minipage}
}
\subfigure[Variable of CBRL]
{
	\begin{minipage}{0.30\linewidth}
	\centering 
	\includegraphics[width=0.9\columnwidth]{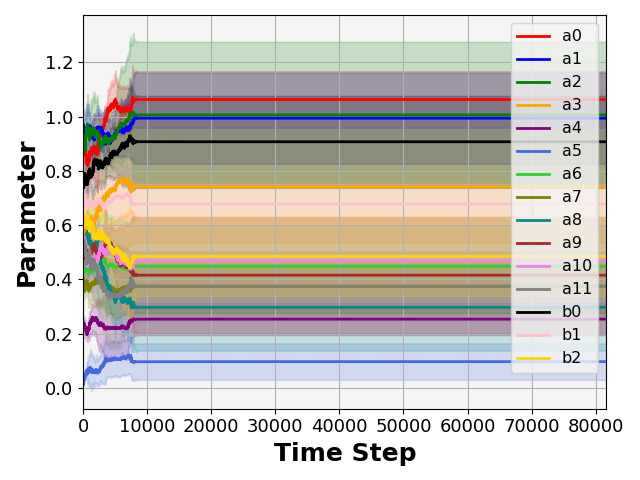} 
    \label{fig_ll_variable}
	\end{minipage}
}
% \label{fig_learn_curve_lunarlander}
%
%
\subfigure[Return vs.\ Episode]
{
\begin{minipage}{0.30\linewidth}
\centering    
\includegraphics[width=0.9\columnwidth]{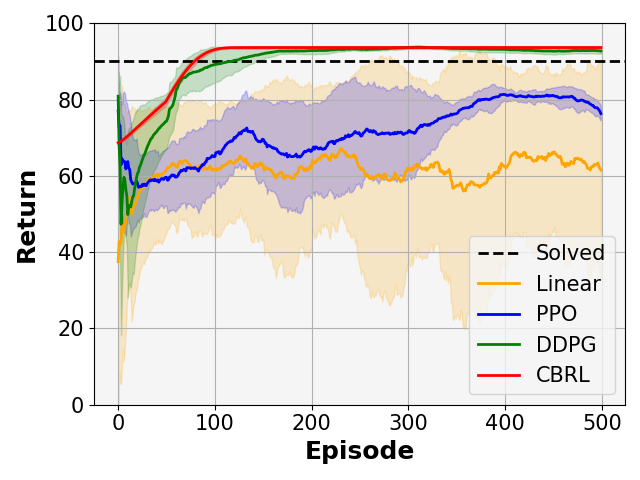}  
\label{fig_mc_episode}
\end{minipage}
}
\subfigure[Return vs.\ Time]
{
\begin{minipage}{0.30\linewidth}
\centering    
\includegraphics[width=0.9\columnwidth]{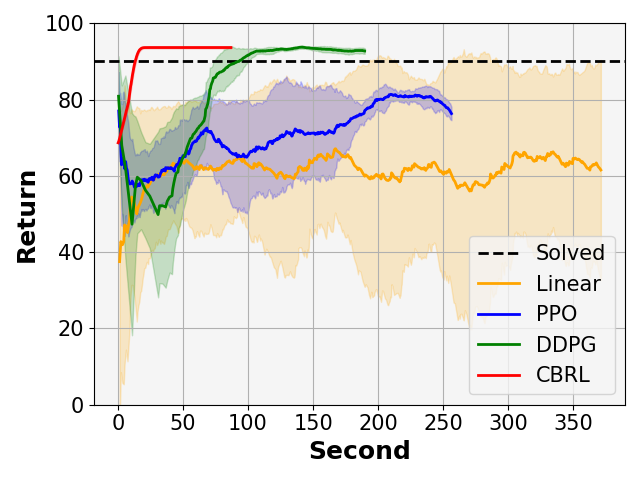}  
\label{fig_mc_second}
\end{minipage}
}
\subfigure[Variable of CBRL]
{
	\begin{minipage}{0.30\linewidth}
	\centering 
	\includegraphics[width=0.9\columnwidth]{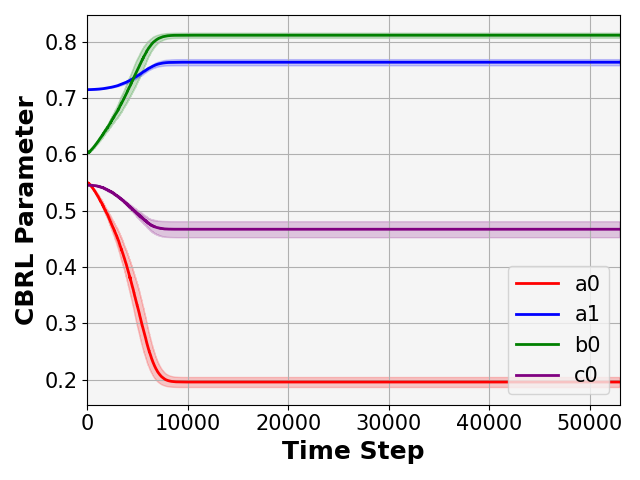}  
    \label{fig_mc_variable}
	\end{minipage}
}
% \label{fig_learn_curve_mountaincar}
%
%
\subfigure[Return vs.\ Episode]
{
\begin{minipage}{0.30\linewidth}
\centering    
\includegraphics[width=0.9\columnwidth]{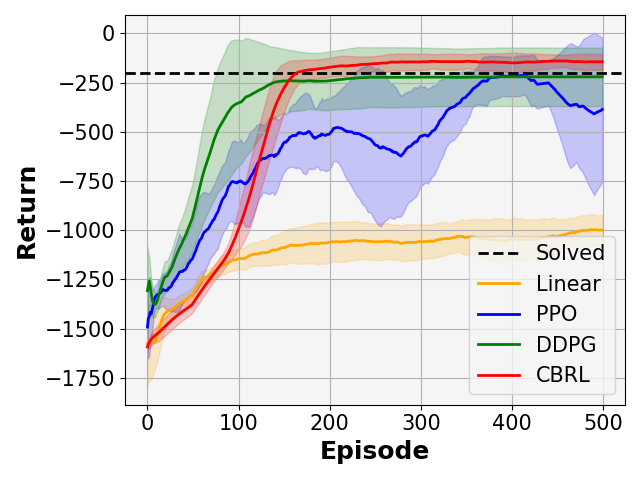}  
\label{fig_pd_episode}
\end{minipage}
}
\subfigure[Return vs.\ Time]
{
\begin{minipage}{0.30\linewidth}
\centering    
\includegraphics[width=0.9\columnwidth]{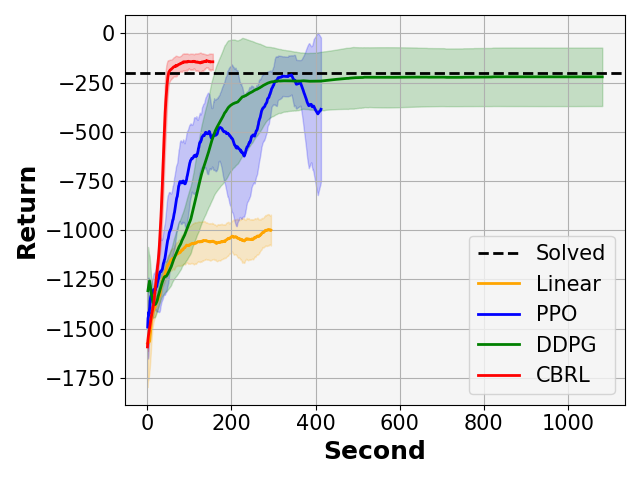} 
\label{fig_pd_second}
\end{minipage}
}
\subfigure[Variable of CBRL]
{
	\begin{minipage}{0.30\linewidth}
	\centering 
	\includegraphics[width=0.9\columnwidth]{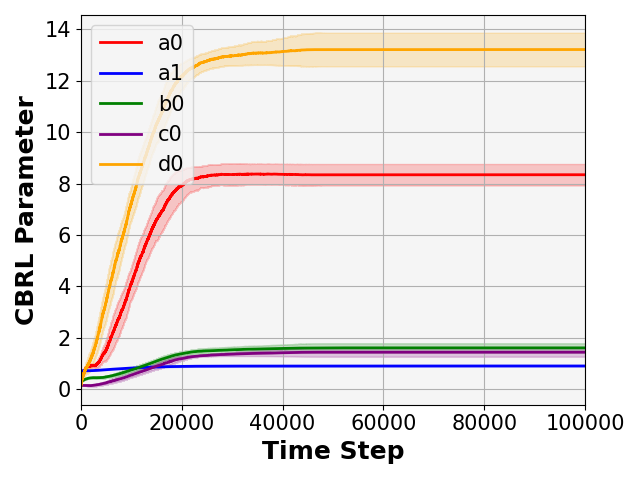}  
    \label{fig_pd_variable}
	\end{minipage}
}
% \label{fig_learn_curve_pendulum}
\caption{
Learning curves over five independent runs comparing our CBRL approach with the Linear policy, PPO, DQN (discrete actions), and DDPG (continuous actions), where the solid line shows the mean and the shaded area depicts the standard deviation for CartPole ($a$)-($c$), LunarLander ($d$)-($f$), MountainCar ($g$)-($i$), and Pendulum ($j$)-($l$). 
}
\end{figure}
%%%%%%%%%%%%%%%%%%%%%%%%%%%%%%%%%%%%%%%%
%%%%%%%%%%%%%%%%%%%%%%%%%%%%%%%%%%%%%%%%
%%%%%%%%%%%%%%%%%%%%%%%%%%%%%%%%%%%%%%%%
%%%%%%%%%%%%%%%%%%%%%%%%%%%%%%%%%%%%%%%%
%%%%%%%%%%%%%%%%%%%%%%%%%%%%%%%%%%%%%%%%

\subsection{\textit{Mountain Car
%(Continuous)
} under CBRL Piecewise-LQR}
\label{ssec:mountaincar}

\textit{Mountain Car (Continuous)}, as depicted in Fig.~\ref{mountaincar_env} of 
% the appendix, 
Appendix~\ref{app:mountaincar},
consists of the placement of a car in a valley with the goal of accelerating the car to reach the target at the top of the hill on the right by applying a force on the car to the left or right. The system state is $x=[p, v]$ in terms of the position of the car $p$ and the velocity of the car $v$. With random initialization at the bottom of the valley, the problem is considered ``solved'' upon achieving an average return of $90$.

% \paragraph{Controller.}
% Recalling that the LQR controller is not sufficient to solve this problem,
% %the \textit{Mountain Car (Continuous)} problem,
% % even provided that all the variables of the system are known (e.g., mass of the car and gravity), as a nonlinear controller is required.
% % Consequently, we consider in this problem a piecewise-LQR controller 
Recall that the LQR controller is not sufficient to solve the \textit{Mountain Car (Continuous)} problem,
even if all the variables of the system are known (e.g., mass of the car and gravity), 
because a nonlinear controller is required.
Consequently, 
we consider a piecewise-LQR controller that partitions the state space into two regions (LQR~1 and LQR~2 in Fig.~\ref{mountaincar_env} of
Appendix~\ref{app:mountaincar}).
% the appendix).
The target state $x^*=[p^*, v^*]$ is selected to be $x^*=[-1.2, 0]$ if $v<0$ (LQR~1), and to be $x^*=[0.6, 0]$ otherwise (LQR~2), where $-1.2$ and $0.6$ represent the position of the left hill and the right hill, respectively.
We address the problem within the context of our CBRL approach by exploiting the 
% canonical (switch to word that we use in Section 3) 
general matrix
form for the piecewise-LQR dynamics given by~\eqref{matrixAB_mountaincar} 
% (Appendix~\ref{app:mountaincar}), 
in 
% the appendix,
Appendix~\ref{app:mountaincar},
solely taking into account general physical relationships and laws, with the $d=4$ unknown variables $a_0, a_1, b_0, c_0$ to be learned; refer to Fig.~\ref{fig_mc_variable}.
%
%
%
%
% \begin{figure}[H]
% % \vspace*{-0.15in}
% \centering 
% \setcounter{subfigure}{0}
% \subfigure[Return vs.\ Episode]
% {
% \begin{minipage}{0.30\linewidth}
% \centering    
% \includegraphics[width=0.9\columnwidth]{figures/mc_return_episode_p1.png}  
% \end{minipage}
% \label{fig_learn_curve_mountaincar:a}
% }
% \subfigure[Return vs.\ Time]
% {
% \begin{minipage}{0.30\linewidth}
% \centering    
% \includegraphics[width=0.9\columnwidth]{figures/mc_return_CPUtime_p1.png}  
% \end{minipage}
% \label{fig_learn_curve_mountaincar:b}
% }
% \subfigure[Variable of CBRL]
% {
% 	\begin{minipage}{0.30\linewidth}
% 	\centering 
% 	\includegraphics[width=0.9\columnwidth]{figures/mc_parameter_c_0_large.png}  
% 	\end{minipage}
% \label{fig_learn_curve_mountaincar:c}
% }
% \caption{Learning curves of \textit{MountainCarContinuous-v0} over five independent runs. Our CBRL approach is in comparison with the Linear policy, PPO, and DDPG.}
% \label{fig_learn_curve_mountaincar}
% \end{figure}
% %

% \paragraph{Numerical Results.}
Fig.~\ref{fig_mc_episode}~--~\ref{fig_mc_variable} and Table~\ref{tab_mean_mc}
% (see appendix) 
(Appendix~\ref{app:mountaincar})
present numerical results for the three state-of-the-art baselines (continuous actions) and our CBRL approach, with each run over five independent random seeds;
%
% All variables in~\eqref{matrixAB_mountaincar} are initialized uniformly within $(0,1)$ and then learned using our control-policy-variable gradient ascent iteration~\eqref{eq:policy-gradient}. 
% We consider four sets of initial variables to validate the robustness of our CBRL approach; refer to Table~\ref{tab_init_para_mc} and Fig.~\ref{fig_all_curve_mountaincar} in the appendix.
%
refer to Table~\ref{tab_init_para_mc} and Fig.~\ref{fig_all_curve_mountaincar} in 
% the appendix 
Appendix~\ref{app:mountaincar}
for the four sets of initial variables in~\eqref{matrixAB_mountaincar}.
Fig.~\ref{fig_mc_episode}~--~\ref{fig_mc_second}, Table~\ref{tab_mean_mc} and Fig.~\ref{fig_all_curve_mountaincar}(a)~--~\ref{fig_all_curve_mountaincar}(b) 
% (appendix)
clearly demonstrate that our CBRL approach 
% outperforms 
provides superior performance (both mean and standard deviation) over
all baselines w.r.t.\ both the number of episodes and running time, in addition to demonstrating a more stable training process.
Fig.~\ref{fig_all_curve_mountaincar}(c)~--~\ref{fig_all_curve_mountaincar}(f) illustrates the learning behavior of CBRL variables given different initialization.

%%%%%%%%%%%%%%%%%%%%%%%%%%%%%%%%%%%%%%%%
%%%%%%%%%%%%%%%%%%%%%%%%%%%%%%%%%%%%%%%%
%%%%%%%%%%%%%%%%%%%%%%%%%%%%%%%%%%%%%%%%
%%%%%%%%%%%%%%%%%%%%%%%%%%%%%%%%%%%%%%%%
%%%%%%%%%%%%%%%%%%%%%%%%%%%%%%%%%%%%%%%%
%%%%%%%%%%%%%%%%%%%%%%%%%%%%%%%%%%%%%%%%
%%%%%%%%%%%%%%%%%%%%%%%%%%%%%%%%%%%%%%%%
%%%%%%%%%%%%%%%%%%%%%%%%%%%%%%%%%%%%%%%%
%%%%%%%%%%%%%%%%%%%%%%%%%%%%%%%%%%%%%%%%

\subsection{\textit{Pendulum} under CBRL Piecewise-LQR}
\label{ssec:pendulum}

\textit{Pendulum}, as depicted in Fig.~\ref{pendulum_env} of 
% the appendix,
Appendix~\ref{app:pendulum},
consists of a link attached at one end to a fixed point and the other end being free, with the goal of swinging up to an upright position by applying a torque on the free end. The state of the system is $x=[\theta, \dot{\theta}]$ in terms of the angle of the link $\theta$ and the angular velocity of the link $\dot{\theta}$. 
With random initial 
% position of the link, 
link position,
each episode comprises $200$ steps and the problem is considered 
``solved'' upon achieving an average return of $-200$.
% \red{CWW,MSS: given how small the figure of the partitioning is, we feel it might be better to move the figure of the partitioning to the supplement where it can appear in full size and be better appreciated.  that would also free up some space in section 4.} \orange{Weiqin: Sounds good.}

% \paragraph{Controller.}
% \begin{wrapfigure}{r}{0.25\textwidth}
%     \centering
% \includegraphics[width=0.2\textwidth]{figures/PW_LQR.png}
% \caption{Partition of piece-wise LQR in \textit{Pendulum}.}
% \label{fig_partition_pendulum}
% \end{wrapfigure}

Recall that the LQR controller is not sufficient to solve the \textit{Pendulum} problem, even if all the variables of the system are known (e.g., mass of the link $m$, length of the link $l$, moment of inertia of the link $J$, and gravity $g$), because a nonlinear controller is required. 
Consequently, we consider a piecewise-LQR controller that partitions the state space into four regions (LQR~1~--~4 in Fig.~\ref{pendulum_env} of 
Appendix~\ref{app:pendulum}). 
% the appendix).
In terms of the target state $x^*=[\theta^*, \dot{\theta}^*]$, the angle $\theta^*$ is selected based on the boundary angle in each partition (counter-clockwise boundary angle if $\dot{\theta} > 0$, and clockwise boundary angle otherwise), while the angular velocity $\dot{\theta}^*$ is selected based on the energy conservation law; see
Appendix~\ref{app:pendulum} 
% the appendix
for more details.
We address the problem within the context of our CBRL approach by exploiting the 
% canonical (switch to word that we use in Section 3) 
general matrix
form for the piecewise-LQR dynamics given by~\eqref{matrixAB_pendulum} 
% (Appendix~\ref{app:pendulum}), 
in 
Appendix~\ref{app:pendulum},
% the appendix,
solely taking into account general physical relationships and laws, with the $d=5$ unknown variables $a_0, a_1, b_0, c_0, d_0$ to be learned; refer to Fig.~\ref{fig_pd_variable}.
%
%

% \paragraph{Numerical Results.}
Fig.~\ref{fig_pd_episode}~--~\ref{fig_pd_variable}  and Table~\ref{tab_mean_pd}
% (see appendix) 
(Appendix~\ref{app:pendulum})
present numerical results for the three state-of-the-art baselines (continuous actions) and our CBRL approach, with each run over five independent random seeds;
%
% All variables in~\eqref{matrixAB_pendulum} are initialized uniformly within $(0,1)$ and then learned using our control-policy-variable gradient ascent iteration~\eqref{eq:policy-gradient}. 
% We consider four sets of initial variables to validate the robustness of our CBRL approach; refer to Table~\ref{tab_init_para_pd} and Fig.~\ref{fig_all_curve_pendulum} in the appendix.
%
see Table~\ref{tab_init_para_pd} and Fig.~\ref{fig_all_curve_pendulum} in 
% the appendix 
Appendix~\ref{app:pendulum}
for the four sets of initial variables in~\eqref{matrixAB_pendulum}.
%
% Despite a somewhat slower convergence than DDPG w.r.t.\ the number of episodes (as depicted in Fig.~\ref{fig_pd_episode}~--~\ref{fig_pd_variable} (a)), Fig.~\ref{fig_pd_episode}~--~\ref{fig_pd_variable} (b) clearly demonstrates that our CBRL approach outperforms all baselines w.r.t.\ the running time, in addition to demonstrating a more stable training process.
Fig.~\ref{fig_pd_episode}~--~\ref{fig_pd_second}, Table~\ref{tab_mean_pd} and Fig.~\ref{fig_all_curve_pendulum}(a)~--~\ref{fig_all_curve_pendulum}(b) clearly demonstrate that our CBRL approach provides far superior performance (both mean and standard deviation) over all baselines w.r.t.\ the number of episodes upon convergence of our CBRL algorithm after a relatively small number of episodes and w.r.t.\ the running time across all episodes.
In particular, even with only four partitions for the difficult nonlinear \textit{Pendulum} RL task, our CBRL approach provides significantly better mean performance and a much lower standard deviation than the closest competitor DDPG after around $150$ episodes.
Moreover, Fig.~\ref{fig_pd_second} clearly demonstrates that our CBRL approach outperforms all other baselines w.r.t.\ running time.
These numerical results also demonstrate
% illustrate a comparable mean performance between our CBRL approach and the best baseline DDPG in terms of the convergence w.r.t.\ the number of episodes, while Figs.~\ref{fig_learn_curve_pendulum}(b),~\ref{fig_all_curve_pendulum}(b) and Table~\ref{tab_mean_pd} clearly demonstrate that our CBRL approach outperforms all baselines w.r.t.\ running time, in addition to demonstrating both a lower standard deviation across independent runs and 
a more stable training process for our CBRL approach.
Finally,
Fig.~\ref{fig_all_curve_pendulum}(c)~--~\ref{fig_all_curve_pendulum}(f) illustrates the learning behavior of CBRL variables given different initialization.

% %
% \begin{figure}[H]
% % \vspace*{-0.15in}
% \centering 
% \setcounter{subfigure}{0}
% \subfigure[Return vs.\ Episode]
% {
% \begin{minipage}{0.30\linewidth}
% \centering    
% \includegraphics[width=0.9\columnwidth]{figures/pd_return_episode_p1.png}  
% \end{minipage}
% }
% \subfigure[Return vs.\ Time]
% {
% \begin{minipage}{0.30\linewidth}
% \centering    
% \includegraphics[width=0.9\columnwidth]{figures/pd_return_CPUtime_p1.png}  
% \end{minipage}
% }
% \subfigure[Variable of CBRL]
% {
% 	\begin{minipage}{0.30\linewidth}
% 	\centering 
% 	\includegraphics[width=0.9\columnwidth]{figures/pd_parameter_c_1_large.png}  
% 	\end{minipage}
% }
% \caption{Learning curves of \textit{Pendulum-v1} over five independent runs. Our CBRL approach is in comparison with the Linear policy, PPO, and DDPG.}
% \label{fig_learn_curve_pendulum}
% \end{figure}

%%%%%%%%%%%%%%%%%%%%%%%%%%%%%%%%%%%%%%%%
%%%%%%%%%%%%%%%%%%%%%%%%%%%%%%%%%%%%%%%%
%%%%%%%%%%%%%%%%%%%%%%%%%%%%%%%%%%%%%%%%
%%%%%%%%%%%%%%%%%%%%%%%%%%%%%%%%%%%%%%%%
%%%%%%%%%%%%%%%%%%%%%%%%%%%%%%%%%%%%%%%%

% \orange{My concern is that AAAI reviewers will barely read appendix in general, so we probably want to show enough work in the main paper. When we discuss the initialization of CBRL approach, should we put some figures for evidence? Due to limited space, we can just put four figures of Return vs Time in the main body and still leave other figures in the appendix, as what we are doing now.}
\subsection{Discussion}
Our numerical results demonstrate an important form of robustness exhibited by the optimal control policy of our CBRL approach w.r.t.\ the learned variable values.
In particular, we performed additional comparative numerical experiments to evaluate the performance of LQR with the known true variables for \textit{Cart Pole} and the performance of piecewise-LQR with the known true variables for \textit{Mountain Car} and \textit{Pendulum}
(\textit{Lunar Lander} is omitted because its underlying true variables are not known to us).
We then compare the relative difference between these numerical results for LQR with the true variables and the corresponding numerical results of our CBRL approach in Fig.~\ref{fig_cp_episode}~--~\ref{fig_cp_second}, Fig.~\ref{fig_mc_episode}~--~\ref{fig_mc_second}, and Fig.~\ref{fig_pd_episode}~--~\ref{fig_pd_second}.
% \red{
It is important to note that the variable values learned by our CBRL approach and presented in Fig.~\ref{fig_cp_variable}, Fig.~\ref{fig_mc_variable}, and Fig.~\ref{fig_pd_variable} differ considerably from the corresponding true variables.
% } \orange{I am not sure if we want to mention the red texts. CBRL does not even have the same number of variables as LQR. Also, would red texts make readers more confused? as we don't discuss and present the variables of LQR here. In my opinion, solely comparing the return of CBRL and LQR is probably better?}
% {\color{blue} the whole reason for this paragraph is to show the robustness of our CBRL approach wrt the learned variables.  if we don't state that the variables are different, then it doesn't make the case that we want to make, i.e., if we learn the correct values, then there is no robustness wrt the learned variables; there is only robustness if the learned variables are different and still obtain a good solution.  also, regarding the different number of variables, for those true variables that do not need to be learned but we go ahead and learn them because we don't know this, it is still the case that there are considerable differences between what we learn and what the true value is (i.e., 0). \orange{I am good now.}
% }
Despite these non-negligible variable value differences, the comparative relative performance differences for \textit{Cart Pole} and \textit{Pendulum} respectively show that LQR and piecewise-LQR with the true variables provide no improvement in return over the corresponding return of our CBRL approach in Fig.~\ref{fig_cp_episode}~--~\ref{fig_cp_second} and Fig.~\ref{fig_pd_episode}~--~\ref{fig_pd_second}.
In addition, the 
% \red{return} 
return
of our CBRL approach for \textit{Mountain Car} in Fig.~\ref{fig_mc_episode}~--~\ref{fig_mc_second} is within
% \red{$2\%$} 
$2\%$
of the corresponding
% \red{return} 
return
of piecewise-LQR with the true variables.

% \orange{$2\%$? You probably refer to $98\%$? We can either say return is within $98\%$ or the relative difference is within $2\%$}
% {\color{blue} no, this seems to be correct.  if the relative difference is within 2\%, it implies that the one value is within 2\% of the other value.}
% \orange{I got it now. But I still think it might be confusing to readers tbh. If it looks good to Santiago and Chai Wah, I am fine to keep it as it is.}

The supremum over the variable space in the definition of our CBRL operator~\eqref{eqn:ContractionOp}, together with its
unique 
fixed point~\eqref{eqn:fixed_point}, 
and the $Q$-learning update rule~\eqref{eqn:q-learn-new}
are primarily used for theoretical purposes in this paper.
However, one might be concerned about potential computational challenges associated with the supremum when the variable space is continuous.
We first note that this issue is not fundamentally different from the corresponding challenges in standard RL when the action space is continuous, and thus similar steps can be taken to address the issue.
Another important factor that mitigates such concerns is the foregoing robustness of the optimal control policy w.r.t.\ the learned variable values.
Lastly, we note that this issue did not arise in our numerical experiments using Algorithm~\ref{alg:control_rl}. 
% First, the computation of the supremum in \eqref{eqn:ContractionOp} to apply the contraction operator may be challenging in cases where the variable space is continuous. This issue is not different than regular reinforcement learning where the supremum over actions is taken for a continuous action-space. In Appendix \ref{app:additional_exp} we provide details to approximate the fixed point \eqref{eqn:fixed_point} that mirror traditional RL methods. 

One important implementation issue is stepsize selection, just like in nonconvex optimization and RL. We address this issue by using similar methods employed in standard RL, i.e., the adaptive stepsize selection in our implementation of Algorithm~\ref{alg:control_rl}.
% More concretely, 
Specifically,
we
% simply 
reduce the stepsize $\eta$ in \eqref{eq:policy-gradient} by a factor of $0.99$ each time achieving the ``solved'' return.
Another important implementation issue concerns variable-vector initialization,
again like in nonconvex optimization and RL. 
As discussed for each of the above RL tasks, the numerical results in 
% Figs.~\ref{fig_all_curve_cartpole}~--~\ref{fig_all_curve_pendulum} 
Figs.~\ref{fig_all_curve_cartpole},~\ref{fig_all_curve_lunarlander},~\ref{fig_all_curve_mountaincar},~\ref{fig_all_curve_pendulum} 
demonstrate the robustness of our CBRL approach w.r.t.\ variable-vector initialization.
An additional important factor here is the aforementioned robustness of the optimal control policy w.r.t.\ the learned variable values.
Moreover, given the efficiency of our CBRL approach relative to state-of-the-art methods, additional computations for stepsize selection and variable-vector initialization can be employed when issues are encountered while retaining overall computational efficiencies.
% Given the speed of the proposed method as compared to other state-of-the-art methods additional computations for the selection of step-sizes and restarts with new variables in case there are issues with the initialization are still worthwhile computationally. 

%===============================================================================

\section{Conclusion}
\label{sec:conclusion}
In this paper we devise a CBRL approach to support direct learning of the optimal policy.
We establish various theoretical properties of our approach, including convergence and optimality of our CBRL operator and $Q$-learning, a
new 
control-policy-variable gradient theorem, and a
% specific 
gradient ascent algorithm based on this theorem
within the context of the LQR control-theoretic framework as a representative example.
We then
% \red{adapt our CBRL approach to the LQR control-theoretic framework, as a representative example, and empirically evaluate its performance}
conduct numerical experiments to empirically evaluate the performance of our general CBRL approach
on several classical RL tasks.
These numerical results demonstrate the significant benefits of our CBRL approach over state-of-the-art methods in terms of improved quality and robustness of the solution and reduced sample complexity and running time. 

\bibliography{l4dc2025}

@book{Sutton1998,
  author = {Sutton, Richard S. and Barto, Andrew G.},
  edition = {Second},
  publisher = {The MIT Press},
  title = {{Reinforcement Learning: An Introduction}},
  year = {2020}
}

@book{BerTsi96,
  title={Neuro-Dynamic Programming},
  author={D.P. Bertsekas and J.N. Tsitsiklis},
  publisher={Athena Scientific},
  year={1996}
}

@phdthesis{Watk89,
  author = "C. Watkins",
  title = "Learning from Delayed Rewards",
  school = "University of Cambridge",
  address = "Cambridge, U.K.",
  year = "1989"
}

@inproceedings{RanAls98,
 author = {J. Randlov and P. Alstrom},
 title = {Learning to drive a bicycle using reinforcement learning and shaping},
 booktitle = {Proc.\ International Conference on Machine Learning},
 year = {1998}
}

@inproceedings{MnKaSi+13,
 author = {V. Mnih and K. Kavukcuoglu and D. Silver and A. Graves and I. Antonoglou and D. Wierstra and R. Riedmiller},
 title = {Playing {Atari} with deep reinforcment learning},
 booktitle = {NIPS Workshop on Deep Learning},
 year = {2013}
}

@inproceedings{DeiRas11,
 author = {M.P. Deisenroth and C.E. Rasmussen},
 title = {{PILCO:} A model-based and data-efficient approach to policy search},
 booktitle = {Proc.\ International Conference on Machine Learning},
 year = {2011}
}

@inproceedings{MeHiXu+15,
 author = {D. Meger and J. Higuera and A. Xu and P. Giguere and G. Dudek},
 title = {Learning legged swimming gaits from experience},
 booktitle = {Proc.\ International Conference on Robotics and Automation},
 year = {2015}
}

@article{Schn97,
 author = {J.G. Schneider},
 title = {Exploiting Model Uncertainty Estimates for Safe Dynamic Control Learning},
 journal = "Advances in Neural Information Processing Systems",
 year = "1997"
}

@article{Scha97,
 author = {S. Schaal},
 title = {Learning From Demonstration},
 journal = "Advances in Neural Information Processing Systems",
 year = "1997"
}

@inproceedings{AtkSan97,
 author = {Christopher G. Atkeson and Juan Carlos Santamaria},
 title = {A Comparison of Direct and Model-Based Reinforcement Learning},
 booktitle = {Proc.\ International Conference on Robotics and Automation},
 year = "1997"
}

@inproceedings{NaKaFe+18,
author = {Anusha Nagabandi and Gregory Kahn and Ronald S. Fearing and Sergey Levine},
title = {Neural Network Dynamics for Model-Based Deep Reinforcement Learning with Model-Free Fine-Tuning},
booktitle = {IEEE International Conference on Robotics and Automation (ICRA)},
month = {May},
pages = {7559–7566},
year = {2018}
}

@article{KaLiMo96,
 author = {Kaelbling, L. and Littman, M. and Moore, A.},
 title = {Reinforcement Learning: A Survey},
 journal = "Journal of Artificial Intelligence Research",
 volume = "4",
 pages = "237–285",
 year = {1996}
}

@incollection{Szep10,
 author = {Csaba Szepesvari},
 title = {Algorithms for Reinforcement Learning},
 booktitle = "Synthesis Lectures on Artificial Intelligence and Machine Learning",
 volume = "4.1",
 publisher = "Morgan \& Claypool",
 pages = "1-103",
 year = "2010"
}

@article{jaakkola:1994,
    author = {Jaakkola, Tommi and Jordan, Michael I. and Singh, Satinder P.},
    title = {On the Convergence of Stochastic Iterative Dynamic Programming Algorithms},
    publisher = {MIT Press},
    journal = {Neural Computation},
    volume = 6,
    issue = 6,
    year = 1994,
    pages = {1185-1201}
}

@Article{Hale1970,
  author = {J. K. Hale and M. A. Cruz},
  title  = {Existence, uniqueness and continuous dependence for hereditary systems},
  journal = {Annali di Matematica Pura ed Applicata},
  year   = {1970},
  issn   = {0373-3114},
  pages  = {63-81},
  volume = {85},
  doi    = {10.1007/bf02413530},
}

@Article{Lin1992,
  author = {J.-N. Lin and R. Unbehauen},
  title  = {Canonical piecewise-linear approximations},
  year   = {1992},
  journal = {IEEE Transactions on Circuits and Systems-I: Fundamental Theory and Applications},
  issn   = {1057-7122},
  pages  = {697-699},
  volume = {39},
  number = 8,
  doi    = {10.1109/81.168933},
}

@Article{Tsitsiklis94,
  author = {John N. Tsitsiklis},
  title  = {Asynchronous Stochastic Approximation and {Q}-Learning},
  journal = {Machine Learning},
  year = {1994},
  pages  = {185–202},
  volume = {16},
  publisher = {Kluwer}
}

@inproceedings{NIPS1999_464d828b,
 author = {Sutton, Richard S and McAllester, David and Singh, Satinder and Mansour, Yishay},
 booktitle = {Advances in Neural Information Processing Systems},
 editor = {S. Solla and T. Leen and K. M\"{u}ller},
 pages = {},
 publisher = {MIT Press},
 title = {Policy Gradient Methods for Reinforcement Learning with Function Approximation},
 url = {https://proceedings.neurips.cc/paper_files/paper/1999/file/464d828b85b0bed98e80ade0a5c43b0f-Paper.pdf},
 volume = {12},
 year = {1999}
}

@article{JMLR:v22:19-736,
  author  = {Alekh Agarwal and Sham M. Kakade and Jason D. Lee and Gaurav Mahajan},
  title   = {On the Theory of Policy Gradient Methods: Optimality, Approximation, and Distribution Shift},
  journal = {Journal of Machine Learning Research},
  year    = {2021},
  volume  = {22},
  number  = {98},
  pages   = {1-76},
  url     = {http://jmlr.org/papers/v22/19-736.html}
}

@article{kao2020automatic,
author = {Ta-Chu Kao and Guillaume Hennequin},
title = {Automatic differentiation of {Sylvester}, {Lyapunov}, and algebraic {Riccati} equations},
journal = {ArXiv e-prints},
archivePrefix = {arXiv},
eprint = {2011.11430v2},
primaryClass={math.OC},
month = {November},
year = {2020}
}

@book{riccati:1995,
    title={Algebraic Riccati Equations},
    author = {Peter Lancaster and Leiba Rodman},
    publisher={Clarendon Press},
    location={Oxford},
    year = 1995}

@article{chen2024probabilistic,
  title={Probabilistic constraint for safety-critical reinforcement learning},
  author={Chen, Weiqin and Subramanian, Dharmashankar and Paternain, Santiago},
  journal={IEEE Transactions on Automatic Control},
  year={2024},
  publisher={IEEE}
}

@inproceedings{zhang2024implicit,
  title={An Implicit Trust Region Approach to Behavior Regularized Offline Reinforcement Learning},
  author={Zhang, Zhe and Tan, Xiaoyang},
  booktitle={Proceedings of the AAAI Conference on Artificial Intelligence},
  volume={38},
  number={15},
  pages={16944--16952},
  year={2024}
}

@misc{towers_gymnasium_2023,
        title = {Gymnasium},
        url = {https://zenodo.org/record/8127025},
        urldate = {2023-07-08},
        publisher = {Zenodo},
        author = {Towers, Mark and Terry, Jordan K. and Kwiatkowski, Ariel and Balis, John U. and Cola, Gianluca de and Deleu, Tristan and Goulão, Manuel and Kallinteris, Andreas and KG, Arjun and Krimmel, Markus and Perez-Vicente, Rodrigo and Pierré, Andrea and Schulhoff, Sander and Tai, Jun Jet and Shen, Andrew Tan Jin and Younis, Omar G.},
        month = mar,
        year = {2023},
        doi = {10.5281/zenodo.8127026},
}

@article{lillicrap2015continuous,
  title={Continuous control with deep reinforcement learning},
  author={Lillicrap, Timothy P and Hunt, Jonathan J and Pritzel, Alexander and Heess, Nicolas and Erez, Tom and Tassa, Yuval and Silver, David and Wierstra, Daan},
  journal={arXiv preprint arXiv:1509.02971},
  year={2015}
}

@article{schulman2017proximal,
  title={Proximal policy optimization algorithms},
  author={Schulman, John and Wolski, Filip and Dhariwal, Prafulla and Radford, Alec and Klimov, Oleg},
  journal={arXiv preprint arXiv:1707.06347},
  year={2017}
}

@inproceedings{peng2023weighted,
  title={Weighted policy constraints for offline reinforcement learning},
  author={Peng, Zhiyong and Han, Changlin and Liu, Yadong and Zhou, Zongtan},
  booktitle={Proceedings of the AAAI Conference on Artificial Intelligence},
  volume={37},
  number={8},
  pages={9435--9443},
  year={2023}
}

@inproceedings{feng2024suf,
  title={SUF: Stabilized Unconstrained Fine-Tuning for Offline-to-Online Reinforcement Learning},
  author={Feng, Jiaheng and Feng, Mingxiao and Song, Haolin and Zhou, Wengang and Li, Houqiang},
  booktitle={Proceedings of the AAAI Conference on Artificial Intelligence},
  volume={38},
  number={11},
  pages={11961--11969},
  year={2024}
}

@inproceedings{shen2024multi,
  title={Multi-world Model in Continual Reinforcement Learning},
  author={Shen, Kevin},
  booktitle={Proceedings of the AAAI Conference on Artificial Intelligence},
  volume={38},
  number={21},
  pages={23757--23759},
  year={2024}
}

@article{kumar2020conservative,
  title={Conservative {Q}-learning for offline reinforcement learning},
  author={Kumar, Aviral and Zhou, Aurick and Tucker, George and Levine, Sergey},
  journal={Advances in Neural Information Processing Systems},
  volume={33},
  pages={1179--1191},
  year={2020}
}

@inproceedings{mcmahan2024optimal,
  title={Optimal Attack and Defense for Reinforcement Learning},
  author={McMahan, Jeremy and Wu, Young and Zhu, Xiaojin and Xie, Qiaomin},
  booktitle={Proceedings of the AAAI Conference on Artificial Intelligence},
  volume={38},
  number={13},
  pages={14332--14340},
  year={2024}
}

@inproceedings{dong2023model,
  title={Model-based offline reinforcement learning with local misspecification},
  author={Dong, Kefan and Flet-Berliac, Yannis and Nie, Allen and Brunskill, Emma},
  booktitle={Proceedings of the AAAI Conference on Artificial Intelligence},
  volume={37},
  number={6},
  pages={7423--7431},
  year={2023}
}

@inproceedings{zheng2023adaptive,
  title={Adaptive policy learning for offline-to-online reinforcement learning},
  author={Zheng, Han and Luo, Xufang and Wei, Pengfei and Song, Xuan and Li, Dongsheng and Jiang, Jing},
  booktitle={Proceedings of the AAAI Conference on Artificial Intelligence},
  volume={37},
  number={9},
  pages={11372--11380},
  year={2023}
}

@inproceedings{schulman2015trust,
  title={Trust region policy optimization},
  author={Schulman, John and Levine, Sergey and Abbeel, Pieter and Jordan, Michael and Moritz, Philipp},
  booktitle={International conference on machine learning},
  pages={1889--1897},
  year={2015},
  organization={PMLR}
}

@book{Puterman2005,
  title={Markov Decision Processes: Discrete Stochastic Dynamic Programming},
  author={Martin L. Puterman},
  year={2005},
  publisher={John Wiley \& Sons}
}

\appendix

\newpage
\section{Control-Based Reinforcement Learning 
% Framework}
Approach}
\label{app:proofs}
In this appendix, we present the proofs of our main theoretical results 
% on control-based RL
together with additional results and technical details related to our CBRL approach
% adapted to 
combined with
the LQR control-theoretic framework 
as presented
in 
% the main body of the paper.
Section~\ref{sec:CBRLapproach}.
We also provide the algorithmic details of our control-policy-variable gradient ascent method used for
% the results in Section~\ref{sec:experiments}.
our numerical results of 
% the main body of the paper 
Section~\ref{sec:experiments}
and
% in the appendix below.
Appendix~\ref{app:experiments}.

% \red{MSS: i would recommend NOT including the figure below because it does not exactly match the description in the main part of the paper.  and i'm not sure what it really adds.  it made sense in the proposal because we were trying to differentiate from generic RL.  but this audience is more sophisticated about RL and thus i don't think it has the same benefit.

% that said, if the majority view is to include the figure, then i think we need a new subsection that describes it in sufficient detail and we need to modify the overview statement above to cover this new subsection (as opposed to just pointing to it parenthetically.  we would also need to find the figure and modify the notation to be consistent with the main body of the paper.}
% \orange{I was planning to modify the figure so that is is consistent with the notations used in the paper. And yes it will only exists in the Appendix.}
% %
% \begin{figure}[h]
%     \centering
%     \includegraphics[width=0.7\columnwidth]{figures/crbl-larger.pdf}
% \caption{The schematic of control-based reinforcement learning.}
% \label{cbrl_schematic}
% \end{figure}
% %
% %

\subsection{Proof of Theorem~\ref{contraction}.}
%

% \begin{assumption}%[Richness of policy family]
\noindent\textbf{Assumption~\ref{asm:richness}.}
\textit{
There exist a policy function $\cG$ in the family $\bbF$ and a unique 
variable 
vector $v^\ast$ in the 
independent-variable set $\setV$ such that, for any state $x\in\setX$, 
$\pi^\ast(x) = \cF_{v^\ast}(x)=\cG(v^\ast)(x)$.
} \\
% \end{assumption}

%\begin{theorem}\label{contraction}
\noindent\textbf{Theorem~\ref{contraction}.}
\textit{
For any $\gamma \in(0,1)$, the operator $\ContractionOp$ in~\eqref{eqn:ContractionOp} is a contraction in the supremum norm.
% {\color{red}
Moreover, the sequence defined by the iterative process $\tilde{q}_{t+1}= \ContractionOp (\tilde{q}_{t})$ converges with rate $\gamma$ to $\tilde{q}^\ast$ as $t\rightarrow\infty$ such that
\begin{equation*}
\| \tilde{q}_{t} - \tilde{q}^\ast \|_\infty \leq \gamma^t \| \tilde{q}_{0} - \tilde{q}^\ast \|_\infty .
\end{equation*}
% }
Supposing Assumption~\ref{asm:richness} holds for the family of policy functions $\bbF$ and its variable set $\setV$,
the
contraction operator $\ContractionOp$ achieves the same asymptotically optimal outcome as that of
the Bellman operator
$\BellmanOp$.
}\\
%\blue{SP: If we decided to change the notation before we should adapt the text here. If we keep it, this can be ignored.}
%\red{MSS/CWW: attempted to address as noted above}
%\end{theorem}

\begin{proof}
Recall that $\tilde{\Pi}\subseteq \Pi$ denotes the set of all stationary policies that are directly determined by the control-policy functions $\cG(v)$ over all $v\in\setV$.
Then, for
% For 
the operator $\ContractionOp$ defined on $\tilde{\setQ}(\setX\times \tilde{\Pi})$ in~\eqref{eqn:ContractionOp} and for any two functions
% $$
% \tilde{q}_1(x, \cF_{v_1})\in\setQ(\setX\times \bbF) \quad \mbox{ and } \quad \tilde{q}_2(x, \cF_{v_2})\in \setQ(\setX\times \bbF)
% $$
$\tilde{q}_1(x, \cF_{v})\in\tilde{\setQ}(\setX\times \tilde{\Pi})$ and $\tilde{q}_2(x, \cF_{v})\in \tilde{\setQ}(\setX\times \tilde{\Pi})$
with $v\in\setV$, we obtain
\begin{align*}
 \|\ContractionOp \tilde{q}_1-\ContractionOp \tilde{q}_2 \|_{\infty} & \stackrel{(a)}{=} \sup_{x\in\setX, v\in\setV} \Big| \sum_{y\in \setX} 
 % \pr_{\cF_v(x)}(x,y) 
 \pr(y \, | \, x, \cF_v)
 \Big[r(x,\cF_v(x)) + \gamma \sup_{v_1\in\setV} \tilde{q}_1(y,\cF_{v_1}) \\
 & \qquad\qquad\qquad\qquad\qquad\qquad\quad - r(x,\cF_v(x)) - \gamma \sup_{v_2\in\setV} \tilde{q}_2(y,\cF_{v_2}) \Big] \Big|  
 \\
& \stackrel{(b)}{=} \sup_{x\in\setX,v\in\setV} \gamma \Big| \sum_{y\in \setX}
% \pr_{\cF_v(x)}(x,y) 
\pr(y \, | \, x, \cF_v)
\Big[ \sup_{v_1\in\setV} \tilde{q}_1(y,\cF_{v_1}) - \sup_{v_2\in\setV} \tilde{q}_2(y,\cF_{v_2}) \Big] \Big| 
\\
& \stackrel{(c)}{\leq} \sup_{x\in\setX,v\in\setV} \gamma \sum_{y\in \setX}
% \pr_{\cF_v(x)}(x,y) 
\pr(y \, | \, x, \cF_v)
\Big| 
\sup_{v_1\in\setV} \tilde{q}_1(y,\cF_{v_1}) - \sup_{v_2\in\setV} \tilde{q}_2(y,\cF_{v_2}) \Big| 
\\
& \stackrel{(d)}{\leq} \sup_{x\in\setX,v\in\setV} \gamma \sum_{y\in \setX}
% \pr_{\cF_v(x)}(x,y) 
\pr(y \, | \, x, \cF_v)
\sup_{z\in\setX,v^\prime\in\setV} \left| \tilde{q}_1(z,\cF_{v^\prime}) - \tilde{q}_2(z,\cF_{v^\prime}) \right| \\
& \stackrel{(e)}{\leq} \sup_{x\in\setX,v\in\setV} \gamma \sum_{y\in \setX}
% \pr_{\cF_v(x)}(x,y) 
\pr(y \, | \, x, \cF_v)
\| \tilde{q}_1 - \tilde{q}_2 \|_{\infty} \\
& \stackrel{(f)}{=} \; \gamma \| \tilde{q}_1 - \tilde{q}_2 \|_{\infty},
\end{align*}
where
(a) is by definition in~\eqref{eqn:ContractionOp}
noting that taking the supremum over $\setV$ is equivalent to taking the supremum over $\tilde{\Pi}$,
(b) follows from straightforward algebra, 
(c) and (d) are due to the triangle inequality,
and (e) and (f) directly follow by definition.
For any $\gamma \in(0,1)$, this establishes that the operator $\ContractionOp$ in~\eqref{eqn:ContractionOp} is a contraction in the supremum norm, thus rendering the desired result for the first part of the theorem.

% {\color{red}
Next, for the rate of convergence,
since the operator $\ContractionOp$ is a contraction mapping,
we have from above
\begin{equation*}
    \| \tilde{q}_{t+1} - \tilde{q}^\ast \|_\infty = \| \ContractionOp \tilde{q}_{t} - \ContractionOp \tilde{q}^\ast \|_\infty \leq \gamma \| \tilde{q}_{t} - \tilde{q}^\ast \|_\infty , \qquad t \in \Ints_+ .    
\end{equation*}
Upon recursively applying the above inequality, we obtain
\begin{equation*}
    \| \tilde{q}_{t} - \tilde{q}^\ast \|_\infty \leq \gamma^t \| \tilde{q}_{0} - \tilde{q}^\ast \|_\infty , \qquad t \in \Ints_+  
\end{equation*}
which asymptotically tends to zero as $t\rightarrow\infty$ implying that $\tilde{q}_{t}$ converges exponentially to $\tilde{q}^\ast$ with rate $\gamma$.
% }

% \blue{
% There may be a need to add more detail in this proof. Given that $\tilde{q}_1$ is defined for a variable $v_1$ and $\tilde{q}_2$ is defined for a variable $v_2$ it seems to me that the difference of $\ContractionOp \tilde{q}_1-\ContractionOp\tilde{q}_2$ should be

% \begin{align}
% (\ContractionOp \tilde{q}_1)(x,\cF_{v_1}) -(\ContractionOp \tilde{q}_2)(x,\cF_{v_2})&=  \sum_{y\in \setX} 
% % \pr_{\cF_v(x)}(x,y) 
% \; \pr(y \, | \, x, \cF_{v_1}) 
% %\;\; \times \nonumber \\
% %& 
% \Big[r(x,\cF_{v_1}(x)) + \gamma \sup_{v^\prime \in\setV} \tilde{q}(y,\cF_{v^\prime}) \Big] \\ 
% &-\sum_{y\in \setX} 
% % \pr_{\cF_v(x)}(x,y) 
% \; \pr(y \, | \, x, \cF_{v_2}) 
% %\;\; \times \nonumber \\
% %& 
% \Big[r(x,\cF_{v_2}(x)) + \gamma \sup_{v^\prime \in\setV} \tilde{q}(y,\cF_{v^\prime}) \Big].
% \end{align}

% From here I don't see how taking the supremum gets you to the first equality in the proof above.
% }
For the second part of the theorem, 
under the stated supposition, we know that the optimal policy $\pi^{\ast}(x)$ realized by the Bellman operator holds for a unique vector $v^\ast$ in the variable set $\setV$ and any $x\in \setX$.
We also know that the Bellman operator $\BellmanOp$ in~\eqref{eqn:BellmanOperator}
and our CBRL operator $\ContractionOp$ in~\eqref{eqn:ContractionOp} are both contractions in supremum norm with unique fixed points, where the fixed point of the Bellman operator is the optimal solution of the Bellman equation.
In particular,
repeatedly applying the Bellman operator $\BellmanOp$ in~\eqref{eqn:BellmanOperator} to an action-value function $q(x,u)\in\bbQ(\setX \times \setA)$ is well-known to asymptotically yield the unique fixed point equation $\BellmanOp q_B^\ast = q_B^\ast$ 
where
\begin{align}
q_B^{\ast}(x,\pi^{\ast}(x)) &= \sum_{y\in \setX} 
\pr(y \, | \, x, \pi^{\ast})
\Big[ r(x,\pi^{\ast}(x)) + \gamma \sup_{u^\prime\in\setA} q_B^{\ast}(y,u^\prime)  \Big]
\nonumber \\
& 
= \sum_{y\in \setX} 
\pr(y \, | \, x, \pi^{\ast})
\Big[ r(x,\pi^{\ast}(x)) + \gamma q_B^{\ast}(y,\pi^{\ast}(y))  \Big] ,
\label{eqn:Bellmanfixed_point}
\end{align}
with the second equality following from the definition of $\pi^\ast(x)$,
and $q_B^\ast$ is the optimal action-value function~\citep{Szep10}.
Likewise, repeatedly applying our CBRL operator $\ContractionOp$ in~\eqref{eqn:ContractionOp} to an action-value function $\tilde{q}(x,\cF_v) \in\tilde{\bbQ}(\setX \times \tilde{\Pi})$ asymptotically renders the unique fixed point equation $\ContractionOp \tilde{q}_C^{\ast} = \tilde{q}_C{^\ast}$ where
\begin{align}
\tilde{q}_C^{\ast}(x,\cF_{v^\ast}) & = \sum_{y\in \setX} 
\pr(y \, | \, x, \cF_{v^\ast})
\Big[ r(x,\cF_{v^\ast}(x)) + \gamma \sup_{v^\prime \in \setV} \tilde{q}_C^{\ast}(y,\cF_{v^\prime})  \Big] 
\nonumber \\
& 
= \sum_{y\in \setX} 
\pr(y \, | \, x, \cF_{v^\ast})
\Big[ r(x,\cF_{v^\ast}(x)) + \gamma \tilde{q}_C^{\ast}(y,\cF_{v^\ast})  \Big] ,
\label{eqn:CBRLfixed_point}
\end{align}
with the second equality following from the relationship between $\tilde{q}(x,\cF_{v^\ast})$ and $\tilde{q}(x,\pi^\ast)$, given below in~\eqref{eqn:SuppositionEquivalence}, under the supposition of the theorem together with the optimality of $\pi^\ast$.
By showing that these two fixed points $q_B^{\ast}(x,\pi^{\ast}(x))$ and $\tilde{q}_C^{\ast}(x,\cF_{v^\ast})$ in~\eqref{eqn:Bellmanfixed_point} and~\eqref{eqn:CBRLfixed_point} coincide, the desired result follows.

To this end, from the definition of $\tilde{q}(x,\cF_{v^\ast})$ and under the stated supposition of the theorem, we have for any $\tilde{q}(x,\cF_{v^\ast})  \in\tilde{\bbQ}(\setX \times \tilde{\Pi})$ and any $x\in \setX$
\begin{equation}
    \tilde{q}(x,\cF_{v^\ast}) := q(x,\cF_{v^\ast}(x)) = q(x,\pi^\ast(x)) =: \tilde{q}(x,\pi^\ast) ,
    \label{eqn:SuppositionEquivalence}
\end{equation}
noting that $v^\ast$ is the unique vector in the variable set $\setV$ associated with $\pi^\ast(x)$.
Let ${q}_C^\ast(x,\cF_{v^\ast}(x)) = \tilde{q}_C^\ast(x,\cF_{v^\ast})$ denote the action-value function that corresponds by definition to $\tilde{q}_C^\ast(x,\cF_{v^\ast})$ of~\eqref{eqn:SuppositionEquivalence} in accordance with~\eqref{eqn:CBRLfixed_point}.
% \red{
% Upon substituting $\tilde{q}^\ast(x,\pi^\ast) = \tilde{q}^\ast(x,\cF_{v^\ast})$ from~\eqref{eqn:SuppositionEquivalence} into~\eqref{eqn:CBRLfixed_point}, we obtain
% \begin{equation}
% \tilde{q}^{\ast}(x,\pi^{\ast}) = \sum_{y\in \setX} 
% \pr(y \, | \, x, \pi^{\ast})
% \Big[ r(x,\pi^{\ast}(x)) + \gamma \tilde{q}^{\ast}(y,\pi^{\ast})  \Big] ,
% \label{eqn:CBRLfixed_point-pi}
% \end{equation}
% which together with the substitution of $\hat{q}^\ast(x,\pi^\ast(x)) = \tilde{q}^\ast(x,\pi^\ast)$ from~\eqref{eqn:SuppositionEquivalence} yields
% \begin{equation*}
% \hat{q}^{\ast}(x,\pi^{\ast}(x)) = \sum_{y\in \setX} 
% \pr(y \, | \, x, \pi^{\ast})
% \Big[ r(x,\pi^{\ast}(x)) + \gamma \hat{q}^{\ast}(y,\pi^{\ast}(y))  \Big] .
% \end{equation*}
% }
Then, in one direction to show $\tilde{q}_C^{\ast}(x,\cF_{v^\ast}) \leq q_B^{\ast}(x,\pi^{\ast}(x))$, we derive
\begin{align}
q_B^{\ast}(x,\pi^{\ast}(x)) & \stackrel{(a)}{=} 
\sum_{y\in \setX} 
\pr(y \, | \, x, \pi^{\ast})
\Big[ r(x,\pi^{\ast}(x)) + \gamma q_B^{\ast}(y,\pi^{\ast}(y))  \Big]
\nonumber \\
& 
\stackrel{(b)}{\geq} \sum_{y\in \setX} 
\pr(y \, | \, x, \pi^{\ast})
\Big[ r(x,\pi^{\ast}(x)) + \gamma {q}_C^{\ast}(y,\pi^{\ast}(y))  \Big]
\nonumber \\
& 
\stackrel{(c)}{=} \sum_{y\in \setX} 
\pr(y \, | \, x, \cF_{v^{\ast}})
\Big[ r(x,\cF_{v^{\ast}}(x)) + \gamma {q}_C^{\ast}(y,\cF_{v^{\ast}}(y))  \Big]
\nonumber \\
& 
\stackrel{(d)}{=} \sum_{y\in \setX} 
\pr(y \, | \, x, \cF_{v^{\ast}})
\Big[ r(x,\cF_{v^{\ast}}(x)) + \gamma \tilde{q}_C^{\ast}(y,\cF_{v^{\ast}}) \Big]
% \nonumber \\
% & 
\stackrel{(e)}{=} \tilde{q}_C^{\ast}(x,\cF_{v^{\ast}}) ,
\end{align}
where:
(a) is directly from the definition of the fixed point in~\eqref{eqn:Bellmanfixed_point};
(b) follows from the definition of the Bellman operator $\BellmanOp$ for any action-value function $q(x,u) \in \bbQ(\setX \times \setA)$, the unique fixed point~\eqref{eqn:Bellmanfixed_point}, and the optimality of $q_B^\ast$ and $\pi^\ast$;
(c) follows upon the substitution of policy $\cF_{v^\ast}$ for $\pi^\ast$ and ${q}_C^\ast(x,\pi^\ast(x)) = {q}_C^\ast(x,\cF_{v^\ast}(x))$ from~\eqref{eqn:SuppositionEquivalence};
(d) follows upon the substitution of ${q}_C^\ast(x,\cF_{v^\ast}(x)) = \tilde{q}_C^\ast(x,\cF_{v^\ast})$ from~\eqref{eqn:SuppositionEquivalence};
and (e) follows from \eqref{eqn:CBRLfixed_point}.
% and thus we have $\tilde{q}^\ast(x,\cF_{v^\ast}) \leq q^{\ast}(x,\pi^{\ast}(x))$ for all $x\in\setX$.
%
In the other direction to show $q_B^{\ast}(x,\pi^{\ast}(x)) \leq \tilde{q}_C^{\ast}(x,\cF_{v^\ast})$, we derive
\begin{align}
\tilde{q}_C^{\ast}(x,\cF_{v^\ast}) & \stackrel{(a)}{=} \sum_{y\in \setX} 
\pr(y \, | \, x, \cF_{v^\ast})
\Big[ r(x,\cF_{v^\ast}(x)) + \gamma \tilde{q}_C^{\ast}(y,\cF_{v^\ast})  \Big]
\nonumber \\
& 
\stackrel{(b)}{=} \sum_{y\in \setX} 
\pr(y \, | \, x, \cF_{v^\ast})
\Big[ r(x,\cF_{v^\ast}(x)) + \gamma {q}_C^{\ast}(y,\cF_{v^\ast}(y)) \Big] \nonumber \\
& \stackrel{(c)}{\geq} \sum_{y\in \setX} 
\pr(y \, | \, x, \cF_{v^\ast})
\Big[ r(x,\cF_{v^\ast}(x)) + \gamma {q}_B^{\ast}(y,\cF_{v^\ast}(y)) \Big]
\nonumber \\
&
\stackrel{(d)}{=} \sum_{y\in \setX} 
\pr(y \, | \, x, \pi^\ast)
\Big[ r(x,\pi^\ast(x)) + \gamma {q}_B^{\ast}(y,\pi^{\ast}(y)) \Big] \nonumber \\
& \stackrel{(e)}{=} q_B^{\ast}(x,\pi^{\ast}(x)) ,
\end{align}
where:
(a) is directly from the definition of the fixed point in~\eqref{eqn:CBRLfixed_point};
(b) follows upon the substitution of ${q}_C^\ast(x,\cF_{v^\ast}(x)) = \tilde{q}_C^\ast(x,\cF_{v^\ast})$ from~\eqref{eqn:SuppositionEquivalence};
(c) follows from the definition of the CBRL operator $\ContractionOp$ for any action-value function $q(x,u) \in \bbQ(\setX \times \setA)$, the unique fixed point~\eqref{eqn:CBRLfixed_point}, and the application of the Bellman optimality conditions on $\tilde{q}_C^\ast$ in~\eqref{eqn:CBRLfixed_point} for $\cF_{v^\ast}$,
all under the supposition of the theorem;
(d) follows upon the substitution of policy $\pi^\ast$ for $\cF_{v^\ast}$ and $q_B^\ast(x,\cF_{v^\ast}(x)) = q_B^\ast(x,\pi^\ast(x))$ from~\eqref{eqn:SuppositionEquivalence};
and (e) follows from \eqref{eqn:Bellmanfixed_point}.
We therefore have coincidence of the fixed points~\eqref{eqn:Bellmanfixed_point} and~\eqref{eqn:CBRLfixed_point}, 
thus completing the proof.

\end{proof}

\subsection{Proof of Theorem~\ref{thm:QLearning}.} 
\noindent\textbf{Theorem~\ref{thm:QLearning}.}
\textit{
Suppose 
Assumption~\ref{asm:richness} holds for the family of policy functions $\bbF$ and its independent-variable set $\setV$ with a contraction operator $\ContractionOp$ as defined in~\eqref{eqn:ContractionOp}.
If $\sum_t \alpha_t = \infty$, $\sum_t \alpha_t^2 < \infty$, and $r_t$ are bounded, then $\tilde{q}_{t}$ under the $Q$-learning update rule~\eqref{eqn:q-learn-new} converges to the optimal fixed point $\tilde{q}^*$ as $t\rightarrow\infty$
and
the optimal policy function is obtained from a unique variable vector $v^* \in \setV$.
}\\

\begin{proof}
We first rewrite \eqref{eqn:q-learn-new} as a convex combination of
$$
\tilde{q}_{t}(x_t,\cF_{v,t}) \qquad \mbox{and} \qquad r_t + \gamma \sup_{v^\prime\in\setV} \tilde{q}_t(x_{t+1},\cF_{v^\prime}) ,
$$
where
\begin{equation*}\tilde{q}_{t+1}(x_t,\cF_{v,t}) = \left(1-\alpha_t(x_t,\cF_{v,t})\right)\tilde{q}_{t}(x_t,\cF_{v,t}) + \alpha_t(x_t,\cF_{v,t})\left[ r_t + \gamma \sup_{v^\prime\in\setV} \tilde{q}_t(x_{t+1},\cF_{v^\prime}) \right].
\end{equation*}
Define $\Delta_t := \tilde{q}_t - \tilde{q}^*$ to be the difference between $\tilde{q}_t$ and $\tilde{q}^*$, which satisfies
\begin{equation*}\Delta_{t+1}(x_t,\cF_{v,t}) = \left(1-\alpha_t(x_t,\cF_{v,t})\right){\Delta}_{t}(x_t,\cF_{v,t}) + \alpha_t(x_t,\cF_{v,t})\left[ r_t + \gamma \sup_{v^\prime\in\setV} \tilde{q}_t(x_{t+1},\cF_{v^\prime}) - \tilde{q}^*\right].
\end{equation*}
Further define
\begin{equation*}
H_t(x,\cF_v) := r(x,\cF_v(x)) + \gamma \sup_{v^\prime\in\setV} \tilde{q}_t(X(x,\cF_v),\cF_{v^\prime}) - \tilde{q}^*(x,\cF_v), 
\end{equation*}
where $X(x,\cF_v)$ is a randomly sampled state from the MDP under an initial state $x$ and policy $\cF_v$.
We then derive
\begin{align*}
\ex[H_t(x,\cF_v)|P_t] & = \sum_{y\in \setX} \pr(y | x, \cF_v)\Big[r(x,\cF_v(x)) +\gamma \sup_{v^\prime\in\setV}\tilde{q}_t(y,\cF_{v^\prime}) - \tilde{q}^*(x,\cF_v)\Big] \\
& = (\ContractionOp \tilde{q}_t)(x,\cF_v) - \tilde{q}^*(x,\cF_v) = (\ContractionOp \tilde{q}_t)(x,\cF_v) - (\ContractionOp \tilde{q}^*)(x,\cF_v) ,
\end{align*}
where $P_t = \{\Delta_t, \Delta_{t-1}, \ldots, H_{t-1},\ldots, \alpha_{t-1},\ldots \}$ represents the 
% past state and parameters~\citep{Tsitsiklis94}. 
past history of the process~\citep{jaakkola:1994,Tsitsiklis94},
and
the last equality follows from $\tilde{q}^*(x,\cF_v) = (\ContractionOp \tilde{q}^*)(x,\cF_v)$.
From this equation together with Theorem~\ref{contraction}, we conclude
\begin{equation*}
\left\|\ex[H_t(x,\cF_v)|P_t]\right\|_\infty = \left\|\ContractionOp \tilde{q}_t - \ContractionOp \tilde{q}^*\right\|_\infty
\leq \gamma \left\|\tilde{q}_t - \tilde{q}^*\right\|_\infty = \gamma \|\Delta_t\|_{\infty} .
\end{equation*}
Similarly, we derive
\begin{align*}
    \var[H_t(x,\cF_v)|P_t] & = \ex\left[\Big( r(x,\cF_v(x)) +\gamma\sup_{v^\prime\in\setV}\tilde{q}_t(X(x,\cF_v),\cF_{v^\prime}) - \tilde{q}^*(x,\cF_v) \right. \\
    & \qquad\qquad\qquad\qquad\qquad\qquad\qquad \left. - (\ContractionOp \tilde{q}_t)(x,\cF_v) + \tilde{q}^*(x,\cF_v) \Big)^2\right] \\
    & = \ex\left[\Big(r(x,\cF_v(x)) +\gamma\sup_{v^\prime\in\setV}\tilde{q}_t(X(x,\cF_v),\cF_{v^\prime}) - (\ContractionOp \tilde{q}_t)(x,\cF_v)\Big)^2\right] \\
    & = \var \left[r(x,\cF_v(x))\left. +\gamma\sup_{v^\prime\in\setV}\tilde{q}_t(X(x,\cF_v),\cF_{v^\prime})\right|P_t\right] \\ 
    & \leq \kappa(1+\|\Delta_t\|_{\infty}^2) , \quad \mbox{for some } \kappa > 0 , 
\end{align*}
where the inequality follows from $r(\cdot,\cdot)$ being bounded.

The desired result 
for the convergence of the $Q$-learning algorithm to the optimal fixed point then follows from~\cite[Theorem 3]{Tsitsiklis94}, \cite[Theorem 1]{jaakkola:1994}.
Finally, as a consequence of Assumption~\ref{asm:richness}, we conclude that the optimal policy function is obtained from a unique variable vector $v^* \in \setV$,
thus completing the proof. 
\end{proof}

\subsection{Proof of Theorem~\ref{thm:PWL}.}
Suppose $\bbF$
satisfies
Assumption~\ref{asm:richness},
and consider 
a sequence of less rich families $\bbF_1,\ldots, \bbF_k$ 
of policy functions $\cG^{(i)}\in\bbF_i$ obtained from independent-variable vectors of the corresponding 
independent variable 
sets $\setV_i$, 
further defining the operators $\ContractionOp_i:\tilde{\setQ}(\setX\times \tilde{\Pi}_i) \rightarrow \tilde{\setQ}(\setX\times \tilde{\Pi}_i)$ as in~\eqref{eqn:ContractionOp} for any function
$\tilde{q}_i(x,\cF_v^{(i)})\in \tilde{\setQ}(\setX\times \tilde{\Pi}_i)$, $i\in [k] := \{ 1, \ldots, k \}$.
From Theorem~\ref{contraction}, for $i\in [k]$, we have that the contraction operators 
$\ContractionOp_i:\setQ(\setX\times \tilde{\Pi}_i) \rightarrow \setQ(\setX\times \tilde{\Pi}_i)$
under the variable sets $\setV_i$ converge to the unique fixed points $\tilde{q}_i^\ast(x,\cF_v^{(i)})$ which satisfy,
for all $x\in\setX$ and $v\in \setV_i$,
\begin{equation}
\tilde{q}_i^\ast(x,\cF^{(i)}_v) =  \sum_{y\in \setX}
% \pr_{\cF^{(i)}_v(x)}(x,y) 
\pr(y \, | \, x, \cF^{(i)}_v)
\bigg[r(x,\cF^{(i)}_v(x)) + \gamma \sup_{v^\prime\in\setV_i} \tilde{q}_i^\ast(y,\cF^{(i)}_{v^\prime}) \bigg] 
\end{equation}
corresponding to~\eqref{eqn:fixed_point}.
Let us first consider two such families $\bbF_1$ and $\bbF_2$,
for which we introduce the following lemma used in the proof of Theorem~\ref{thm:PWL}.
\begin{lemma}\label{lem:PWL}
Assume the state and action spaces are compact and $\cF_v$ is uniformly continuous for each $v$.
For two variable sets $\setV_1$ and $\setV_2$ and any two variable vectors $v_1\in\setV_1$ and $v_2\in\setV_2$, let
$d(\cdot,\cdot)$ be a $\sup$-norm distance function defined over the policy space $\tilde{\Pi}$, i.e., $d(\cF^{(1)}_{v_1},\cF^{(2)}_{v_2}) = \sup_{x\in\setX} \| \cF^{(1)}_{v_1}(x) - \cF^{(2)}_{v_2}(x) \|$.
Then, for all $\epsilon > 0$ there exists $\delta >0$ such that,
if for all $v_1\in \setV_1$ there exists $v_2\in \setV_2$ with  $d(\cF^{(1)}_{v_1},\cF^{(2)}_{v_2})< \delta$
and
if for all $v_2\in \setV_2$ there exists $v_1\in \setV_1$ with  $d(\cF^{(1)}_{v_1},\cF^{(2)}_{v_2})< \delta$, 
we have $\sup_{x\in\setX}\|\tilde{q}_1^\ast-\tilde{q}_2^\ast\|<\epsilon$.
\end{lemma}
\begin{proof}
    Since the state and action spaces are compact and $\cF_v$ is uniformly continuous for each $v$, it then follows from Lemma~2.2 in~\cite{Hale1970} that the unique fixed point depends continuously on the contraction mapping $\ContractionOp$ and thus we find that $\tilde{q}^\ast$ w.r.t.\ its second argument depends continuously on the sets $\setV_i$, which implies the desired result.
\end{proof}

Intuitively, Lemma~\ref{lem:PWL} shows that, for any policy families $\bbF_1$ and $\bbF_2$ sufficiently close to each other, the fixed points $\tilde{q}_1$ and $\tilde{q}_2$ of the corresponding operators $\ContractionOp_1$ and $\ContractionOp_2$ are also close to each other.
When the policy family $\bbF_k$ is sufficiently rich and approaches $\bbF$, then the fixed point of the corresponding operator $\ContractionOp_k$ approaches the unique fixed point of $\bbF$ satisfying~\eqref{eqn:fixed_point}, and therefore they approach the optimal $q$-value as promised by Bellman from Theorem~\ref{contraction}.
We formally characterize this asymptotic convergence of approximate optimality to global optimality in Theorem~\ref{thm:PWL}.\\

%\begin{theorem}\tag{\ref{thm:PWL}}
\noindent\textbf{Theorem~\ref{thm:PWL}.}
\textit{
Assume the state and action spaces are compact and $\cF_v$ is uniformly continuous for each $v$. 
Consider $\bbF$ and a sequence of families of policy functions $\bbF_1,  \bbF_2, \ldots, \bbF_{k-1}, \bbF_k$, with $\setV$ and $\setV_i$ respectively denoting the independent-variable sets corresponding to $\bbF$ and $\bbF_i$, $i\in [k]$.
Let $d(\cdot,\cdot)$ be a $\sup$-norm distance function defined over the policy space $\tilde{\Pi}$, i.e., $d(\cF^{(i)}_{v_i},\cF^{(j)}_{v_j}) = \sup_{x\in\setX} \| \cF^{(i)}_{v_i}(x) - \cF^{(j)}_{v_j}(x) \|$, $v_i\in\setV_i, v_j\in\setV_j$.
Further let $d^\prime(\cdot,\cdot)$ be an $\inf$-norm distance function defined over the policy function space $\bbF$, i.e., $d^\prime(\cG^{(i)},\cG^{(j)}) = \inf_{v_i\in\setV_i,v_j\in\setV_j} d(\cG^{(i)}(v_i) , \cG^{(j)}(v_j))$, $\cG^{(i)}\in\bbF_i, \cG^{(j)}\in\bbF_j$.
Suppose, for all $\cG \in \bbF$, there exists a $\cG^{(k)}\in\bbF_k$ such that $d^\prime(\cG^{(k)},\cG) \rightarrow 0$ as $k\rightarrow \infty$.
Then,
% $\sup_{x\in\setX} \|\tilde{q}_k^\ast-\tilde{q}^\ast\| \rightarrow 0$ 
$\|\tilde{q}_k^\ast-\tilde{q}^\ast\|_\infty \rightarrow 0$ 
as $k\rightarrow \infty$.
}\\
%\end{theorem}

\begin{proof}
First, in the definition of the contraction mapping $\ContractionOp$, the result of the supremum depends continuously on the set $\setV$ and thus the corresponding argument of $\ContractionOp$ depends continuously on $\setV$. 
% Next, since the state and action spaces are compact and $\cF_v$ is uniformly continuous for each $v$, then~\cite[Lemma 2.2]{Hale1970} shows that the unique fixed point depends continuously on the contraction mapping $\ContractionOp$ and thus we find that $\tilde{q}^\ast$ w.r.t. its second argument depends continuously on the sets $\setV_i$. This implies Lemma~\ref{lem:PWL},
% from which Theorem~\ref{thm:PWL} then follows directly.
The desired result then follows from Lemma~\ref{lem:PWL}.
\end{proof}

\subsection{Proof of Theorem~\ref{thm:PGT}.}
%
% \begin{theorem}
\noindent\textbf{Theorem~\ref{thm:PGT}.}
\textit{
Consider a family of control policy functions $\bbF$, its independent-variable set $\setV$ with contraction operator $\ContractionOp$ in the form of~\eqref{eqn:ContractionOp}, and
the value function $V_{{\cF_{v}}}$
under the control policy $\cF_v$.
Assuming $\cF_{v,u}(x)$ is differentiable 
% with respect to 
w.r.t.\
$\cF_v$ and $\cF_v(x)$ is differentiable 
% with respect to 
w.r.t.\
$v$
(i.e., both
$\partial \cF_{v,u}(x) / \partial \cF_v$ and $\partial \cF_v(x) / \partial v$
exist),
we then have
\begin{equation}
    \nabla_v V_{{\cF_{v}}} (x_0) = \sum_{x\in \setX} \: \sum_{k=0}^\infty \gamma^k \,
    % \pr_{\cF_v} (x_0,x,k)
    \pr(x, k \, | \, x_0, \cF_v)
    \; \left[ \sum_{u \in \setA} \frac{\partial \cF_{v,u}(x)}{\partial \cF_v} \frac{\partial\cF_v(x)}{\partial v} \tilde{q}_u(x,{\cF}_{v}) \right] .
    \tag{\ref{eqn:nablavp}}
\end{equation}
% \label{thm:PGT}
}
% \end{theorem}

\begin{proof}
We derive
\begin{align*}
\nabla_v V_{{\cF_{v}}} (x_0) 
& \stackrel{(a)}{=} \nabla_v \left[ \sum_u \cF_{v,u}(x_0)\tilde{q}_{u}(x_0,\cF_v)\right] \\
& \stackrel{(b)}{=} \sum_u \left[ \nabla_v \cF_{v,u}(x_0)\tilde{q}_{u}(x_0,\cF_v) + \cF_{v,u}(x_0)\nabla_v \tilde{q}_{u}(x_0,\cF_v) \right] \\
& \stackrel{(c)}{=} \sum_u \Bigg[\nabla_v \cF_{v,u}(x_0)\tilde{q}_{u}(x_0,\cF_v) +  \cF_{v,u}(x_0)\nabla_v\sum_{x,r}
% \pr_{\cF_v}(x,r|x,u)
\pr(x,r \, | \, x_0,\cF_v,u)
(r+\gamma V_{\cF_v}(x))\Bigg] \\
& \stackrel{(d)}{=} \sum_u \left[\nabla_v \cF_{v,u}(x_0)\tilde{q}_{u}(x_0,\cF_v) + \gamma \cF_{v,u}(x_0)\sum_{x}
% \pr_{\cF_v}(x|x,u)
\pr(x \, | \, x_0,\cF_v,u)
\nabla_{v}V_{\cF_v}(x)\right] \\
% & \stackrel{(e)}{=} \sum_u \Bigg[\nabla_v \cF_{v,u}(x_0)\tilde{q}_{u}(x_0,\cF_v) \\
% &\qquad
% + \gamma \cF_{v,u}(x_0)\sum_{x'} \pr_{\cF_v}(x'|x,u)
% \sum_{u'}\bigg[\nabla_{v}\cF_{v,u'}(x') \tilde{q}_{u}(x',\cF_v)
% + \gamma \cF_{v,u'}(x')\sum_{x''} \pr_{\cF_v}(x''|x',u')\nabla_v V_{\cF_v}(x'')\bigg]\Bigg] \\
& \stackrel{(e)}{=} \sum_u \Bigg[\nabla_v \cF_{v,u}(x_0)\tilde{q}_{u}(x_0,\cF_v) + \gamma \cF_{v,u}(x_0)\sum_{x}
% \pr_{\cF_v}(x'|x,u)
\pr(x \, | \, x_0,\cF_v,u)
\sum_{u'}\bigg[\nabla_{v}\cF_{v,u'}(x) \tilde{q}_{u'}(x,\cF_v) \\
&\qquad \qquad \qquad \qquad \qquad \qquad \qquad \qquad\qquad
+ \gamma \cF_{v,u'}(x)\sum_{x'}
% \pr_{\cF_v}(x'|x,u')
\pr(x' \, | \, x,\cF_v,u')
\nabla_v V_{\cF_v}(x')\bigg]\Bigg] \\
& \stackrel{(f)}{=} \sum_{x\in \setX} \sum_{k=0}^\infty \gamma^k \,
% \left(
% \pr_{\cF_v} (x_0,x,k) 
\pr(x,k \, | \, x_0,\cF_v)
% \right) 
\sum_u \nabla_v \cF_{v,u}(x)\tilde{q}_{u}(x,\cF_v) ,
\end{align*}
% \blue{SP: The step c is a bit confusing to me based on the notation we have before. In the main text we conclude that $\tilde{q}_u(x,\cF_v) = q(x,u)$ but at any point we claim this q is the q corresponding to applying $\cF_v$. And therefore the step c is not too clear. We need to explicitly define a q for the given policy.}
% \red{MSS/CWW: we looked at this and even compared with the standard policy gradient theorem, but we must be missing something because we do not see your concern.  let's discuss on Wednesday.}
where
% the first equality 
(a) is by the definition of the value function for state $x_0$,
% the second equality 
(b) follows from the product rule,
% the third equality 
(c) follows by the definition of $\tilde{q}_u(x,\cF_v)$,
% the fourth equality 
(d) follows by applying the gradient w.r.t.\ $v$ summed over all~$r$,
% the fifth equality 
(e) follows by repeating each of the preceding steps for $\nabla_{v}V_{\cF_v}(x')$,
and 
% the last equality 
(f) follows from repeated unrolling along the lines of 
% the fifth equality
(e) 
and upon recalling
% $\pr_{\cF_v} (x_0,x,k)$ 
$\pr(x,k \, | \, x_0,\cF_v)$
to be the probability of going from state $x_0$ to state $x$ in $k$ steps under the control policy $\cF_v$.

The result then follows since 
$\nabla_v \cF_{v,u}(x) = \frac{\partial \cF_{v,u}(x)}{\partial \cF_v}
    \frac{\partial\cF_v(x)}{\partial v}$, by the chain rule.
\end{proof}

\subsection{Control-Policy-Variable Gradient Ascent Algorithm}
\label{app:algorithm}
Based on our new CBRL gradient theorem (Theorem~\ref{thm:PGT}), we devise control-policy-variable gradient ascent methods within the context of our general CBRL approach.
One such gradient ascent method for directly learning the unknown variable vector $v$ of the optimal control policy 
% with respect to 
w.r.t.\
the value function $V$ comprises the iterative process according to~\eqref{eq:policy-gradient} with stepsize $\eta$.
Here $\nabla_v V_{{\cF_{v}}}$ is as given by~\eqref{eqn:nablavp}, where the first gradient term $\partial \cF_{v,u}(x) / \partial \cF_v$ is essentially the standard policy gradient, and the second gradient term $\partial \cF_v(x) / \partial v$ is specific to the control-theoretic framework employed in our general CBRL approach;
refer to 
% % Section~\ref{sec:adaptation}.
% our adaptation of the CBRL approach 
% to the LQR control-theoretic framework in 
Section~\ref{sec:CBRLapproach}
regarding the LQR control-theoretic framework combined as part of our CBRL approach.
We note that standard policy gradient ascent methods are a special case of~\eqref{eq:policy-gradient} where the independent-variable vector $v$ is directly replaced by the policy $\pi$.
In particular, the special case of $\cal G$ being an identity map and $\frac{\partial V}{\partial {\cF}_{v_t}} \frac{\partial {\cF}_{v_t}}{\partial {v_t}}$ replaced by $\frac{\partial V}{\partial \pi_t}$ corresponds to the direct policy gradient parameterization case in~\cite{JMLR:v22:19-736}.

Our algorithmic implementation of the above control-policy-variable gradient ascent method is summarized in Algorithm~\ref{alg:control_rl}.
This algorithm,
together with
% the adaptation of 
our general CBRL approach 
% % to 
% adapted to
combined with
the LQR control-theoretic framework
% according to 
% % Section~\ref{sec:adaptation}, 
% the corresponding section of the main body of the paper,
as presented
in Section~\ref{sec:CBRLapproach},
is used to obtain the numerical results for our CBRL approach presented in 
% % Section~\ref{sec:experiments}.
% the experimental results of 
% the main body of the paper 
Section~\ref{sec:experiments}
and
% in the appendix below.
Appendix~\ref{app:experiments}.

% \begin{algorithm}
% \caption{Control-Policy-Variable Gradient Ascent Algorithm}\label{alg:control_rl}
% \begin{algorithmic}
% \Require $\mathbf{X}_0$, $\mathbf{X_f}$,  $\mathbb{U}$ \Comment{Initial state, Desired final and State Feasible Control Input Set }
% \Ensure  Initialize $\hat\xi^0\;,\;\mathbf{u}^0(\cdot)\;|\;\mathbf{u}^0(t)\in \mathbb{U}\;\forall t \geq 0\;,\;\;D^0 = \Phi$ \Comment{Initial Model, Initial input signal, Empty Data Set  }
% \State $\hat{\xi} \gets \hat{\xi}^0$
% \State $\mathbf{u}(\cdot) \gets \mathbf{u}^0(\cdot)$
% \State $D \gets D^0$\\
% \While{Not Done}\\
% \;\;\;\;\textbf{1.} Collect observation data from real environment\; \\
% \;\;\;\;\;\; $D \gets D \cup \textbf{collect}(\xi,\mathbf{u}(\cdot))$\\
% \;\;\;\;\textbf{2.} Update model using the collected data \\
% \begin{equation*}
%     v_{t+1} = v_{t}+\eta  \nabla_v V_{{\cF_{v}}} (x_t),
% \end{equation*}
% \;\;\;\;\;\; $\hat{\xi} \gets \textbf{train}(\hat{\xi},D)$ \\
% \;\;\;\;\textbf{3.} Update control policy using learnt model \\
% \;\;\;\;\;\; $\mathbf{u}(\cdot) \gets \textbf{update}(\hat{\xi})$ \\
% \EndWhile
% \end{algorithmic}
% \end{algorithm}
%
%
%

\begin{algorithm}
\caption{Control-Policy-Variable Gradient Ascent Algorithm}\label{alg:control_rl}
\begin{algorithmic}
\Require $\mathcal{E}, \mathbb{X}$, $\mathbb{U}, \mathbb{F}, \cF_v, \mathbf{u}(\cdot, \cdot), D_0 = \Phi, \eta, N$ \coolgrey{\Comment{Environment, State Set, Control Action Set, Control-Policy Family, Control-Policy, Control Action Generator (CBRL-LQR), Empty Rollout Dataset, Stepsize, Number of Iterations}}
% \Ensure $v_0, D_0 = \Phi, N$ 
%\Comment{Initial Control-Policy-Variable, Empty Data Set and Total Iterations}
% \State $v \gets v_0$
% \State $\mathbf{u}(\cdot) \gets \mathbf{u}(\mathbf{X}_0, v_0)$
% \State $D \gets D_0$
\State \textbf{Initialize} $v_0 \mid \cF_{v_0} \in \mathbb{F}$, $\mathbf{X}_0 \in \mathbb{X}$ \\
\textbf{for} $t=0, 1, \cdots, N$ \textbf{do}\\
% \;\;\;\;\textbf{1.} Get the state $\mathbf{X}_t \in \mathbb{X}$ from the environment $\mathcal{E}$ \\
% %
% %
\;\;\;\;\textbf{1.} Generate the control action given the current state and control-policy-variable
\begin{equation*}
    \mathbf{u_t} = \mathbf{u}(\mathbf{X}_t, v_t)
\end{equation*}
\;\;\;\;\textbf{2.} Collect the rollout transition data by applying the control action to the environment $\mathcal{E}$
\begin{align*}
    &\mathcal{E} \stackrel{\mathbf{u_t}}{\longrightarrow} \textbf{rollout} := (\mathbf{X}_{t}, \mathbf{u_t}, \rewardR, \mathbf{X}_{t+1})  \\
    &D_{t+1} = D_t \cup \textbf{rollout}
\end{align*}
\;\;\;\;\textbf{3.} Compute the control-policy-variable gradient~\eqref{eqn:nablavp} by using
the standard policy gradient \State \;\;\;\;\;\;\; solution for the
first partial derivative and 
\eqref{eqn:PDcare} for the second partial derivative, given \State \;\;\;\;\;\;\; the current rollout dataset $D_{t+1}$ \\

\;\;\;\;\textbf{4.} Update the control-policy-variable using the gradient ascent iteration \eqref{eq:policy-gradient}
\begin{equation*}
    v_{t+1} = v_t + \eta  \nabla_v V_{{\cF_{v}}} (\mathbf{X}_t)
\end{equation*}
%\EndFor
\end{algorithmic}
\end{algorithm}

Our CBRL approach
% adapted to 
combined with
the LQR control-theoretic framework concerns the linear dynamics $\dot{x} = A_v x + B_v u$ where $A_v$ and $B_v$ contain elements of the unknown variable vector $v$.
By leveraging known basic information about the LQR control problem at hand, only a relatively small number $d$ of unknown variables need to be learned.
% In our CRBL approach we leverage background knowledge about the task at hand in order to reduce the number of unknown variables. 
For
% robotics and physics 
a wide range of
applications where state variables are derivatives of each other w.r.t.\ time (e.g., position, velocity, acceleration, jerk), the corresponding rows in the $A_v$ matrix consist of a single $1$ and the corresponding rows in $B_v$ comprise zeros.
We exploit this basic information to consider general matrix forms for $A_v$ and $B_v$ that reduce the number of unknown variables to be learned. 
As a representative illustration, the system dynamics when there are two groups of such variables have the form given by
\begin{align}\label{matrixAB}
    \underbrace{ \begin{bmatrix} \frac{dx_1}{dt} \\ \vdots \\ \frac{d^{k_1}x_1}{dt^{k_1}}  \\ \frac{dx_2}{dt} \\ \vdots \\ \frac{d^{k_2}x_2}{dt^{k_2}} \end{bmatrix} }_{\dot{x}} \!\!=\!\! \underbrace{\begin{bmatrix} 0 & 1 & 0 & \cdots & & &0 \\ 
    0 & 0 & 1 & \cdots &&\\
    a_{10} & \cdots && && & a_{1n} \\ 0 & \cdots &&0& 1 &\dots & 0 \\ 
    0 & \cdots &&& 0 & 1 &\cdots \\ 
    a_{20} & \cdots && & & & a_{2n} \end{bmatrix} }_{A_v} x + \underbrace{\begin{bmatrix} 0 \\ \vdots \\ b_0 \\ 0 \\ \vdots \\ b_1 \end{bmatrix}}_{B_v} u.
\end{align}

% \vspace*{0.5in}
\clearpage

\section{Experimental Results}
\label{app:experiments}
In this appendix we provide additional details and results w.r.t.\ our numerical experiments to evaluate the performance of our general CBRL approach 
% % % of Section~\ref{sec:CBRLapproach} 
% % adapted to
% adapted to
combined with
the LQR control-theoretic framework 
as presented
in Section~\ref{sec:CBRLapproach}
% according to Section~\ref{sec:adaptation} 
and using  Algorithm~\ref{alg:control_rl}.
Recall that we consider the following classical
% control 
RL
tasks from Gymnasium~\citep{towers_gymnasium_2023}:
\textit{Cart Pole}, 
\textit{Lunar Lander (Continuous)},
\textit{Mountain Car (Continuous)}, and
\textit{Pendulum}.
Our CBRL approach is compared against the three state-of-the-art RL algorithms
DQN~\citep{MnKaSi+13} for discrete actions, DDPG~\citep{lillicrap2015continuous} for continuous actions and PPO~\citep{schulman2017proximal}, together with a variant of PPO that solely replaces the nonlinear policy of PPO with a linear policy (since we know the optimal policy for problems such as \textit{Cart Pole} is linear).
These baselines are selected as the state-of-the-art RL algorithms for solving the 
% control 
RL
tasks under consideration.

Our CBRL approach depends in part upon the control-theoretic framework 
% to which it is adapted, 
with which it is combined, 
where we have chosen LQR as a representative example with the understanding that not all of the above
% control 
RL
tasks can be adequately solved using LQR even if the variable vector $v$ is known.
Recall, however, that our CBRL approach allows the domain of the control policies to span a subset of 
states in $\setX$, thus enabling the partitioning of the state space so that properly increased richness w.r.t.\ finer and finer granularity can provide improved approximations and asymptotic optimality according to Theorem~\ref{thm:PWL} 
of Section~\ref{sec:CBRLapproach}, 
% in the main body of the paper,
analogous to the class of canonical piecewise-linear function approximations~\citep{Lin1992}. 
While \textit{Cart Pole} and \textit{Lunar Lander (Continuous)} can be directly addressed within the context of LQR, this is not the case for \textit{Mountain Car (Continuous)} and \textit{Pendulum} which require a nonlinear controller.
We therefore partition the state space in the case of these 
% control 
RL
tasks and consider a corresponding piecewise-LQR controller where the learned variable vectors may differ or be shared across the partitions.
Such details are provided below for \textit{Mountain Car (Continuous)} and \textit{Pendulum}.

\subsection{\textit{Cart Pole} under CBRL LQR}
\label{app:cartpole}

\paragraph{Environment.}
As depicted in Figure~\ref{cartpole_env} and described in 
% the main body of the paper, 
Section~\ref{ssec:cartpole},
\textit{Cart Pole} consists of a pole connected to a horizontally moving cart with the goal of balancing the pole by applying a force on the cart to the left or right.
The state of the system is given by $x=[p, \dot{p}, \theta, \dot{\theta}]$ in terms of the position of the cart $p$, the velocity of the cart $\dot{p}$, the angle of the pole $\theta$, and the angular velocity of the pole $\dot{\theta}$.
% With upright initialization of the pole on the cart, each episode comprises $200$ steps and the problem is considered ``solved'' upon achieving an average return of $195$.
%
\begin{figure}[H]
    \centering
    \includegraphics[width=0.25\columnwidth]{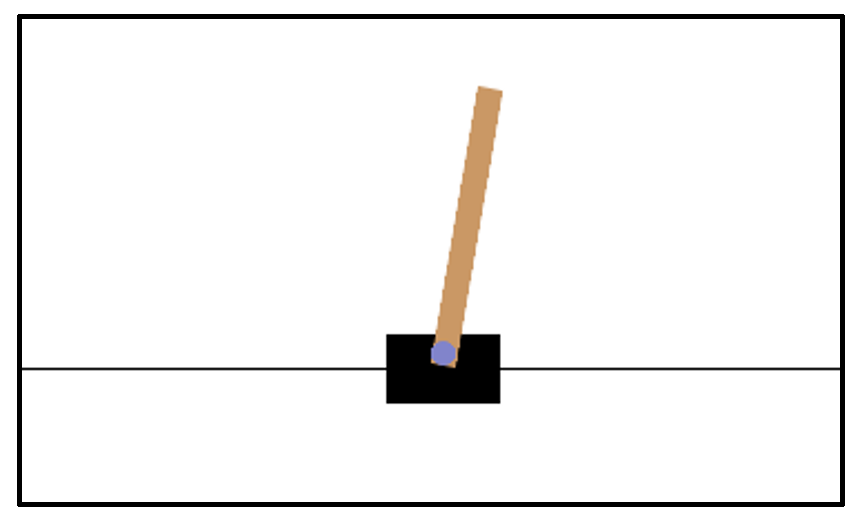}
\caption{The \textit{CartPole-v0} environment.}
\label{cartpole_env}
\end{figure}

\paragraph{Controller.} 
% The LQR controller can be used to solve the \textit{Cart Pole} problem provided that all the variables of the system are known (e.g., mass of cart, mass and length of pole, gravity), where the angle is kept small by our CBRL approach adapted to the LQR framework.
We address the \textit{Cart Pole} problem within the context of our CBRL approach by exploiting
% the general matrix form for the LQR dynamics~\eqref{matrixAB_cartpole}, solely taking into account general physical relationships (e.g., the derivative of the angle of the pole is equivalent to its angular velocity) and laws (e.g., the force can only affect the acceleration), with the unknown variables $a_0, \ldots, a_7, b_0, b_1$ to be learned:
% \begin{align}\label{matrixAB_cartpole}
%     \dot{x} = \begin{bmatrix} \dot{s} \\ \ddot{s} \\ \dot{\theta} \\ \ddot{\theta} \end{bmatrix} = \underbrace{\begin{bmatrix} 0 & 1 & 0 & 0 \\ a_0 & a_1 & a_2 & a_3 \\ 0 & 0 & 0 & 1 \\ a_4 & a_5 & a_6 & a_7 \end{bmatrix} }_{A} x + \underbrace{\begin{bmatrix} 0 \\ b_0 \\ 0 \\  b_1 \end{bmatrix}}_B u.
% \end{align}
the following general matrix form for the LQR dynamics
\begin{align}\label{matrixAB_cartpole}
    \dot{x} = \begin{bmatrix} \dot{p} \\ \ddot{p} \\ \dot{\theta} \\ \ddot{\theta} \end{bmatrix} = \underbrace{\begin{bmatrix} 0 & 1 & 0 & 0 \\ a_0 & a_1 & a_2 & a_3 \\ 0 & 0 & 0 & 1 \\ a_4 & a_5 & a_6 & a_7 \end{bmatrix} }_{A_{v}} x + \underbrace{\begin{bmatrix} 0 \\ b_0 \\ 0 \\  b_1 \end{bmatrix}}_{B_{u}} u,
\end{align}
which solely takes into account general physical relationships (e.g., the derivative of the angle of the pole is equivalent to its angular velocity) and laws (e.g., the force can only affect the acceleration), with the $d=10$ unknown variables $a_0, \ldots, a_7, b_0, b_1$ to be learned.

\paragraph{Numerical Results.} 
Table~\ref{tab_mean_cp}  and Figure~\ref{fig_all_curve_cartpole} present numerical results for the three state-of-the-art baselines (discrete actions) and our CBRL approach, each run over five independent random seeds.
All variables in~\eqref{matrixAB_cartpole} are initialized uniformly within $(0,1)$ and then learned using our control-policy-variable gradient ascent iteration~\eqref{eq:policy-gradient}. 
We consider four sets of initial variables (see Table~\ref{tab_init_para_cp}) to validate the robustness of our CBRL approach.
Figure~\ref{fig_all_curve_cartpole}(c)~--~\ref{fig_all_curve_cartpole}(f) illustrates the learning behavior of CBRL variables given different initialization.

Table~\ref{tab_mean_cp} and Figure~\ref{fig_all_curve_cartpole}(a) and~\ref{fig_all_curve_cartpole}(b) clearly demonstrate that our CBRL approach 
% outperforms 
provides far superior performance, in terms of both mean and standard deviation, over all baselines w.r.t.\ both the number of episodes and running time, in addition to demonstrating a more stable training process. 
More specifically, our CBRL approach outperforms all baselines with a significant improvement in mean return over independent random environment runs from Gymnasium~\citep{towers_gymnasium_2023}, together with a significant reduction in the standard deviation of the return over these independent runs, thus rendering a more robust solution approach where the performance of each run is much closer to the mean than under the corresponding baseline results.
In fact, all independent random environment runs under our CBRL approach yield returns equal to the mean (i.e., standard deviation of $0$) for episodes $150$ and beyond.
Moreover, the relative percentage difference\footnote{For the relative comparison of improved mean return performance of our CBRL method over the next-best RL method,
since all the performance values are positive, 
we use the standard relative percentage difference formula $\frac{M_{\scriptscriptstyle CBRL} - M_{\scriptscriptstyle RL}}{M_{\scriptscriptstyle RL}} \times 100$. Similarly, for the relative comparison of reduced standard deviation of return performance from the next-best RL method to our CBRL method, 
since all the performance values are positive,
we typically use the standard relative percentage difference formula $\frac{S_{\scriptscriptstyle RL} - S_{\scriptscriptstyle CBRL}}{S_{\scriptscriptstyle CBRL}} \times 100$, with the sole exception of replacing the denominator with $(S_{\scriptscriptstyle RL} + S_{\scriptscriptstyle CBRL}) / 2$ based on the arithmetic mean when $S_{\scriptscriptstyle CBRL}$ is zero.
Refer to \url{https://en.wikipedia.org/wiki/Relative_change}}
in improved mean return performance of our CBRL approach over the next-best RL method Linear (based on mean performance at $500$ episodes) is $294\%$ and $6.6\%$ at $200$ and $500$ episodes, respectively;
and the relative percentage difference in reduced standard deviation of return performance from the next-best RL method Linear to our CBRL approach is $200\%$ and $200\%$ at $200$ and $500$ episodes, respectively.
In addition,
the relative percentage difference
in improved mean return performance of our CBRL approach over the DQN RL method is $927\%$ and $29.5\%$ at $200$ and $500$ episodes, respectively;
and the relative percentage difference in reduced standard deviation of return performance from the DQN RL method to our CBRL approach is $200\%$ and $200\%$ at $200$ and $500$ episodes, respectively.

We note an important difference between Fig.~\ref{fig_cp_episode}~--~\ref{fig_cp_second} and Fig.~\ref{fig_all_curve_cartpole}(a)~--~\ref{fig_all_curve_cartpole}(b), namely the shaded areas in Fig.~\ref{fig_cp_episode}~--~\ref{fig_cp_second} represent one form of variability across independent random seeds with the same variable initialization, whereas the shaded areas in Fig.~\ref{fig_all_curve_cartpole}(a)~--~\ref{fig_all_curve_cartpole}(b) represent the combination of two forms of variability---one across independent random seeds with the same variable initialization and the other across different random variable initializations.

\renewcommand{\arraystretch}{1.5}
\begin{table}[htbp]
	\centering
	% \fontsize{9}{9}\selectfont
    \caption{Mean and Standard Deviation of Return of RL Methods for \emph{CartPole-v0}}
		\label{tab_mean_cp}
  \begin{tabular}{||c||c|c|c|c|c|c|c|c||}
  \hline\hline
  Episode & \multicolumn{2}{c|}{CBRL} & \multicolumn{2}{c|}{Linear} & \multicolumn{2}{c|}{PPO} & \multicolumn{2}{c||}{DQN} \\
   Number & Mean & Std.\ Dev.\ & Mean & Std.\ Dev.\ & Mean & Std.\ Dev.\ & Mean & Std.\ Dev.\ \\
  \hline
  50 & $163.49$ & $30.57$ & $27.15$ & $7.79$ & $26.19$ & $15.18$ & $19.99$ & $0.67$ \\
  \hline
  100 & $199.69$ & $ 0.63$ & $27.05$ & $7.21$ & $49.30$ & $22.28$ & $19.76$ & $1.08$ \\
  \hline
  150 & $200.00$ & $0.00$ & $37.20$ & $14.80$ & $57.84$ & $43.68$ & $18.85$ & $1.23$ \\
  \hline
  200 & $200.00$ & $0.00$ & $50.77$ & $26.93$ & $53.58$ & $32.06$ & $19.47$ & $1.79$ \\
  \hline
  250 & $200.00$ & $0.00$ & $80.05$ & $56.27$ & $75.65$ & $50.67$ & $18.88$ & $3.35$ \\
  \hline
  300 & $200.00$ & $0.00$ & $107.18$ & $67.03$ & $102.31$ & $57.89$ & $32.59$ & $21.07$ \\
  \hline
  350 & $200.00$ & $0.00$ & $118.56$ & $66.73$ & $127.67$ & $71.26$ & $56.44$ & $31.43$ \\
  \hline
  400 & $200.00$ & $0.00$ & $138.79$ & $57.88$ & $147.52$ & $66.20$ & $80.85$ & $31.94$ \\
  \hline
  450 & $200.00$ & $0.00$ & $165.75$ & $45.35$ & $183.12$ & $10.49$ & $123.49$ & $20.59$ \\
  \hline
  500 & $200.00$ & $0.00$ & $187.70$ & $23.72$ & $186.16$ & $22.05$ & $154.46$ & $24.42$ \\
  \hline\hline 
  \end{tabular}
\end{table}

% \begin{table}[htbp]
% 	\centering
% 	% \fontsize{9}{9}\selectfont
%     \caption{Mean and Standard Deviation of Return of RL Methods for \emph{CartPole-v0}}
% 		\label{tab_mean_cp}
%   \begin{tabular}{||c||c|c|c|c||}
%   \hline\hline
%   Episode & \multicolumn{1}{c|}{CBRL} & \multicolumn{1}{c|}{Linear} & \multicolumn{1}{c|}{PPO} & \multicolumn{1}{c||}{DQN} \\
%   \hline
%   100 & $199.7$($0.6$) & $27.1$($7.2$) & $49.3$($22.3$) & $19.8$($1.1$) \\
%   \hline
%   200 & $200.0$($0.0$) & $50.8$($26.9$) & $53.6$($32.1$) & $19.5$($1.8$) \\
%   \hline
%   300 & $200.0$($0.0$) & $107.2$($67.0$) & $102.3$($57.9$) & $32.6$($21.1$) \\
%   \hline
%   400 & $200.0$($0.0$) & $138.8$($57.9$) & $147.5$($66.2$) & $80.9$($31.9$) \\
%   \hline
%   500 & $200.0$($0.0$) & $187.7$($23.7$) & $186.2$($22.1$) & $154.5$($24.4$) \\
%   \hline\hline 
%   \end{tabular}
% \end{table}

\renewcommand{\arraystretch}{1.5}
\begin{table}[htbp]
	\centering
	% \fontsize{9}{9}\selectfont
	% \begin{threeparttable}
		\caption{Initial Variables of \emph{CartPole-v0}}
		\label{tab_init_para_cp}
		\begin{tabular}{|c|cccccccccc|}
        \hline
        \hline
			 Name & $a_0$
			& $a_1$ & $a_2$
			& $a_3$ & $a_4$
			& $a_5$ & $a_6$
			& $a_7$ & $b_0$ & $b_1$  \cr
        \hline
   
          Initial Variables 1 (P1) & 0.436 & 0.026 & 0.55 &  0.435 & 0.42 & 0.33 & 0.205 & 0.619 & 0.3 &  0.267 \cr

         Initial Variables 2 (P2) & 0.076 & 0.78 & 0.438 & 0.723 & 0.978 & 0.538 & 0.501 & 0.072 & 0.268 & 0.5 \cr

          Initial Variables 3 (P3) & 0.154 & 0.74 & 0.263 & 0.534 & 0.015 & 0.919 & 0.901 & 0.033 & 0.957 & 0.137 \cr

         Initial Variables 4 (P4) & 0.295 & 0.531 & 0.192 & 0.068 & 0.787 & 0.656 & 0.638 & 0.576 & 0.039 & 0.358 \cr

        \hline
        \hline
		\end{tabular}
	% \end{threeparttable}
\end{table}

\begin{figure}[b]
\centering 
\setcounter{subfigure}{0}
\subfigure[Return vs.\ Episode]
{
\begin{minipage}{0.45\linewidth}
\centering    
\includegraphics[width=1\columnwidth]{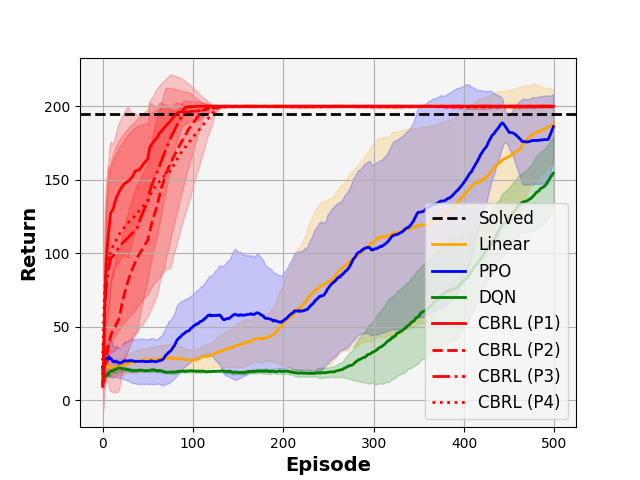}  
\end{minipage}
}
\subfigure[Return vs.\ Time]
{
\begin{minipage}{0.45\linewidth}
\centering    
\includegraphics[width=1\columnwidth]{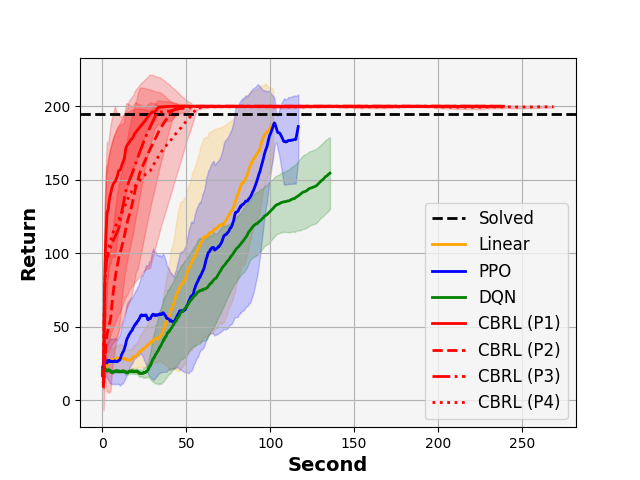}  
\end{minipage}
}\\
\subfigure[P1]
{
\begin{minipage}{0.45\linewidth}
\centering    
\includegraphics[width=1\columnwidth]{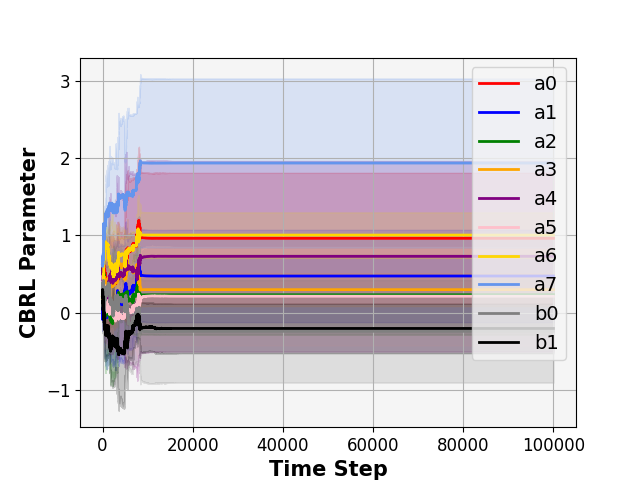}  
\end{minipage}
}
\subfigure[P2]
{
\begin{minipage}{0.45\linewidth}
\centering    
\includegraphics[width=1\columnwidth]{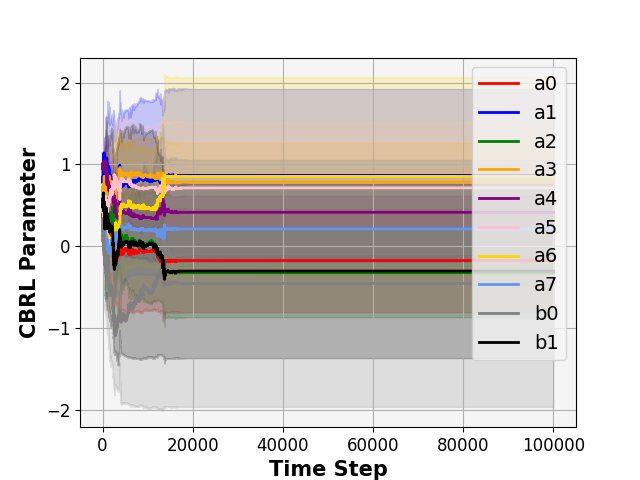}  
\end{minipage}
}\\
\subfigure[P3]
{
	\begin{minipage}{0.45\linewidth}
	\centering 
	\includegraphics[width=1\columnwidth]{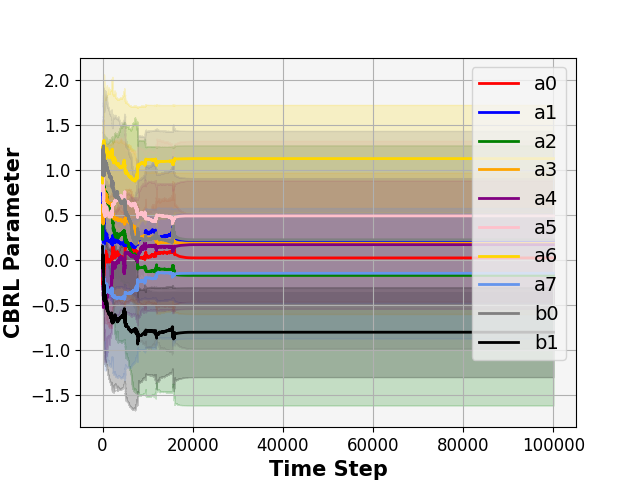}  
	\end{minipage}
}
\subfigure[P4]
{
	\begin{minipage}{0.45\linewidth}
	\centering 
	\includegraphics[width=1\columnwidth]{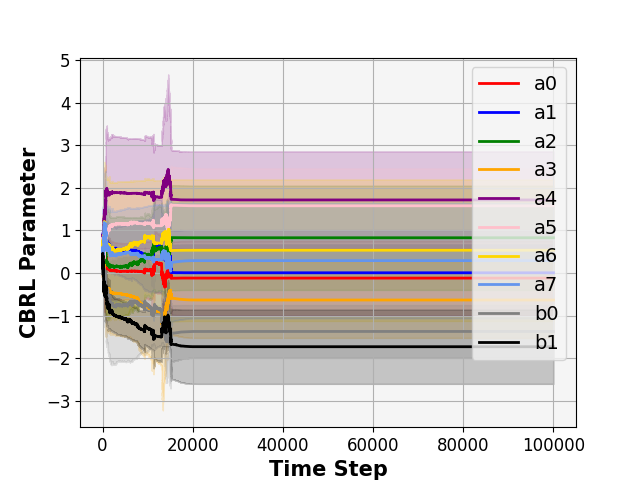}  
	\end{minipage}
}
\caption{Learning curves of \textit{CartPole-v0} over five independent runs. 
The solid line shows the mean and the shaded area depicts the standard deviation. 
(a) and (b): Return vs.\ number of episodes and running time, respectively, for our CBRL approach 
(over the five independent runs and the four initializations in Table~\ref{tab_init_para_cp})
in comparison with the Linear policy, PPO, and DQN
(over the five independent runs). 
(c) -- (f): Learning behavior of CBRL variables, initialized by Table~\ref{tab_init_para_cp}.}
\label{fig_all_curve_cartpole}
\end{figure}

\clearpage

\subsection{\textit{Lunar Lander (Continuous)} under CBRL LQR}
\label{app:lunarlander}

\paragraph{Environment.} 
As depicted in Figure~\ref{lunarlander_env} and described in 
% the main body of the paper, 
Section~\ref{ssec:lunarlander},
\textit{Lunar Lander (Continuous)} is a classical spacecraft trajectory optimization problem with the goal to land softly and fuel-efficiently on a landing pad by applying thrusters to the left, to the right, and upward.
The state of the system is given by $x=[p_x, v_x, p_y, v_y, \theta, \dot{\theta}]$ in terms of the $(x,y)$ positions $p_x$ and $p_y$, two linear velocities $v_x$ and $v_y$, angle $\theta$, and angular velocity~$\dot{\theta}$.
% Starting at the top center of the viewport with a random initial force applied to its center of mass, the lander (spacecraft) is subject to gravity, friction and turbulence while surfaces on the ``moon'' are randomly generated in each episode with the target landing pad centered at $(0,0)$. The problem is considered ``solved'' upon achieving an average return of $200$.
%
\begin{figure}[H]
\centering
\includegraphics[width=0.4\columnwidth]{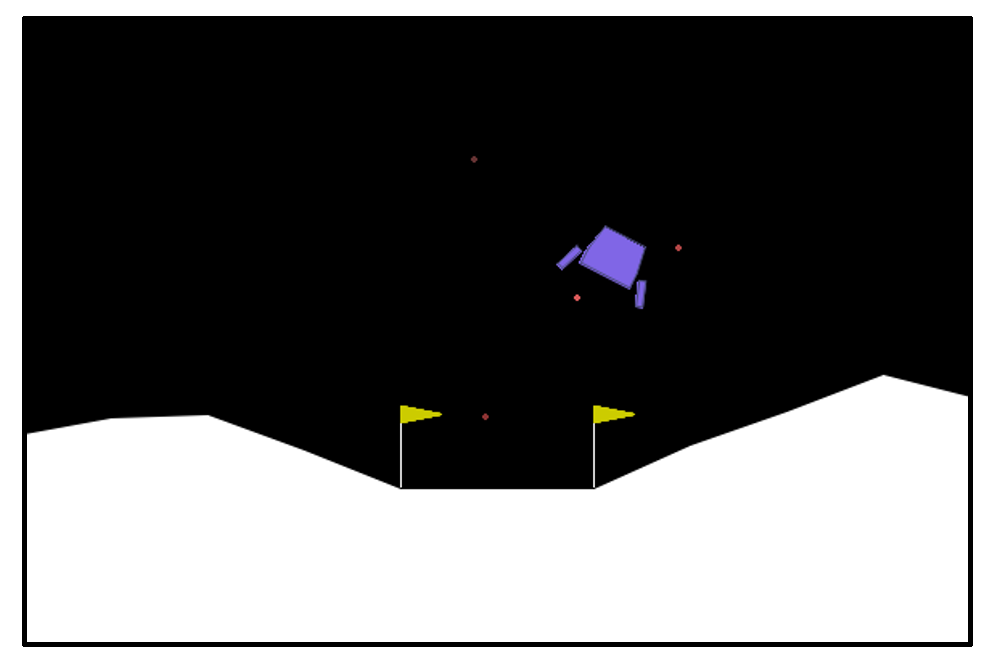}
\caption{The \textit{LunarLanderContinuous-v2} environment.}
\label{lunarlander_env}
\end{figure}

\paragraph{Controller.} 
% The LQR controller can be used to solve the \textit{Lunar Lander (Continuous)} problem provided that all variables of the system are known (e.g., mass of lander, gravity, friction, etc.), where the angle is kept small by our CBRL approach adapted to the LQR framework.
We address the \textit{Lunar Lander (Continuous)} problem within the context of our CBRL approach by exploiting 
% the general matrix form for the LQR dynamics~\eqref{matrixAB_lunarlander}, solely taking into account general physical relationships (akin to \textit{Cart Pole}) and mild physical information from the system state (e.g., the acceleration is independent of the position), with the unknown variables $a_0, \ldots, a_{11}, b_0, b_1, b_2$ to be learned:
% % \orange{Note that the piecewise-LQR that partitions the state space (akin to Section~\ref{subsec:mountaincar} and \ref{subsec:pendulum}) can apply when the angle is not small.}
% %
% \begin{align}\label{matrixAB_lunarlander}
%     \dot{x} =  \begin{bmatrix} v_x \\ \ddot{s}_x \\  v_y  \\  \ddot{s}_y \\ \dot{\theta} \\ \ddot{\theta} \end{bmatrix} = \underbrace{ \begin{bmatrix} 0 & 1 & 0 & 0 & 0 & 0 \\ 0 & a_0 & 0 & a_1 & a_2 & a_3 \\ 0 & 0 & 0 & 1 & 0 & 0 \\ 0 & a_4 & 0 & a_5 & a_6 & a_7 \\ 0 & 0 & 0 & 0 & 0 & 1 \\ 0 & a_8 & 0 & a_9 & a_{10} & a_{11} \\ \end{bmatrix} }_A x + \underbrace{\begin{bmatrix} 0 & 0 \\ 0 & b_0 \\ 0 & 0 \\ b_1 & 0 \\ 0 & 0 \\ 0 & b_2\\ \end{bmatrix} }_B u.
% \end{align}
the following general matrix form for the LQR dynamics
\begin{align}\label{matrixAB_lunarlander}
    \dot{x} =  \begin{bmatrix} v_x \\ \ddot{p}_x \\  v_y  \\  \ddot{p}_y \\ \dot{\theta} \\ \ddot{\theta} \end{bmatrix} = \underbrace{ \begin{bmatrix} 0 & 1 & 0 & 0 & 0 & 0 \\ 0 & a_0 & 0 & a_1 & a_2 & a_3 \\ 0 & 0 & 0 & 1 & 0 & 0 \\ 0 & a_4 & 0 & a_5 & a_6 & a_7 \\ 0 & 0 & 0 & 0 & 0 & 1 \\ 0 & a_8 & 0 & a_9 & a_{10} & a_{11} \\ \end{bmatrix} }_{A_{v}} x + \underbrace{\begin{bmatrix} 0 & 0 \\ 0 & b_0 \\ 0 & 0 \\ b_1 & 0 \\ 0 & 0 \\ 0 & b_2\\ \end{bmatrix} }_{B_{v}} u,
\end{align}
which solely takes into account general physical relationships (akin to \textit{Cart Pole}) and mild physical information from the system state (e.g., the acceleration is independent of the position), with the $d=15$ unknown variables $a_0, \ldots, a_{11}, b_0, b_1, b_2$ to be learned.
% \orange{Note that the piecewise-LQR that partitions the state space (akin to Section~\ref{subsec:mountaincar} and \ref{subsec:pendulum}) can apply when the angle is not small.}

\paragraph{Numerical Results.} 
Table~\ref{tab_mean_ll} and Figure~\ref{fig_all_curve_lunarlander} present numerical results for the three state-of-the-art baselines (continuous actions) and our CBRL approach, each run over five independent random seeds. 
All variables in~\eqref{matrixAB_lunarlander} are initialized uniformly within $(0,1)$ and 
then learned using our control-policy-variable gradient ascent iteration~\eqref{eq:policy-gradient}.
We consider four sets of initial variables (see Table~\ref{tab_init_para_ll}) to validate the robustness of our CBRL approach.
Figure~\ref{fig_all_curve_lunarlander}(c)~--~\ref{fig_all_curve_lunarlander}(f) illustrates the learning behavior of CBRL variables given different initialization.

Table~\ref{tab_mean_ll} and Figure~\ref{fig_all_curve_lunarlander}(a) and~\ref{fig_all_curve_lunarlander}(b) clearly demonstrate that our CBRL approach provides far superior performance, in terms of both mean and standard deviation, over all baselines w.r.t.\ both the number of episodes and running time, in addition to demonstrating a more stable training process.
More specifically, our CBRL approach outperforms all baselines with a significant improvement in mean return over independent random environment runs from Gymnasium~\citep{towers_gymnasium_2023}, together with a significant reduction in the standard deviation of the return over these independent runs, thus rendering a more robust solution approach where the performance of each run is much closer to the mean than under the corresponding baseline results.
In fact, the relative percentage difference\footnote{For the relative comparison of improved mean return performance of our CBRL method over the next-best RL method, since some of the performance values are negative and $|M_{\scriptscriptstyle RL}| < M_{\scriptscriptstyle CBRL}$, we use the standard relative percentage difference formula $\frac{M_{\scriptscriptstyle CBRL} - M_{\scriptscriptstyle RL}}{|M_{\scriptscriptstyle RL}|} \times 100$. Similarly, for the relative comparison of reduced standard deviation of return performance from the next-best RL method to our CBRL method, 
since all the performance values are positive, 
we use the standard relative percentage difference formula $\frac{S_{\scriptscriptstyle RL} - S_{\scriptscriptstyle CBRL}}{S_{\scriptscriptstyle CBRL}} \times 100$.
Refer to \url{https://en.wikipedia.org/wiki/Relative_change}}
in improved mean return performance of our CBRL approach over the next-best RL method DDPG (based on mean performance at $500$ episodes) is $309\%$ and $110\%$ at $200$ and $500$ episodes, respectively;
and the relative percentage difference in reduced standard deviation of return performance from the next-best RL method DDPG to our CBRL approach is $115\%$ and $288\%$ at $200$ and $500$ episodes, respectively.
Moreover, the relative percentage difference in improved mean return performance of our CBRL approach over the PPO RL method is $361\%$ and $390\%$ at $200$ and $500$ episodes, respectively;
and the relative percentage difference in reduced standard deviation of return performance from the PPO RL method to our CBRL approach is $639\%$ and $559\%$ at $200$ and $500$ episodes, respectively.
We note that the baseline algorithms often crash and terminate sooner than the more successful landings of our CBRL approach, resulting in the significantly worse performance exhibited in Fig.~\ref{fig_ll_episode}~--~\ref{fig_ll_variable}, Table~\ref{tab_mean_ll} and Figure~\ref{fig_all_curve_lunarlander}.

We also note an important difference between Fig.~\ref{fig_ll_episode}~--~\ref{fig_ll_second} and Fig.~\ref{fig_all_curve_lunarlander}(a)~--~\ref{fig_all_curve_lunarlander}(b), namely the shaded areas in Fig.~\ref{fig_ll_episode}~--~\ref{fig_ll_second} represent one form of variability across independent random seeds with the same variable initialization, whereas the shaded areas in Fig.~\ref{fig_all_curve_lunarlander}(a)~--~\ref{fig_all_curve_lunarlander}(b) represent the combination of two forms of variability---one across independent random seeds with the same variable initialization and the other across different random variable initializations.

% random seed = (3, 8, 13, 18)

% \renewcommand{\arraystretch}{1.5}
% \begin{table}[H]
% 	\centering
% 	\fontsize{10}{10}\selectfont
% 	\begin{threeparttable}
% 		\caption{Initial Variables of \emph{LunarLanderContinuous-v2}}
% 		\label{tab_init_para_ll}
% 		\begin{tabular}{c|c|cccccccc}
%         \hline
%         \hline
% 			 Name & Seed & $a_0$
% 			& $a_1$ & $a_2$
% 			& $a_3$ & $a_4$
% 			& $a_5$ & $a_6$
% 			& $a_7$  \cr
%         \hline
   
%           Variables 1 (P1) & 3 & 0.551 & 0.708 & 0.291 & 0.511 & 0.893 & 0.896 & 0.126 & 0.207  \cr

%          Variables 2 (P2) & 8 & 0.873 & 0.969 & 0.869 & 0.531 & 0.233 & 0.011 & 0.43 & 0.402  \cr

%           Variables 3 (P3) & 13 & 0.778 & 0.238 & 0.824 & 0.966 & 0.973 & 0.453 & 0.609 & 0.776  \cr

%          Variables 4 (P4) & 18 & 0.65 & 0.505 & 0.879 & 0.182 & 0.852 & 0.75 & 0.666 & 0.988  \cr

%         \hline
% 			 Name & Seed & $a_8$
% 			& $a_9$ & $a_{10}$
% 			& $a_{11}$ & $b_0$ & $b_1$
% 			& $b_2$ \cr

%            \hline
   
%           Variables 1 (P1) & 3 & 0.051 & 0.441 & 0.03 & 0.457 &
%  0.649 & 0.278 & 0.676 \cr

%          Variables 2 (P2) & 8 & 0.523 & 0.478 & 0.555 & 0.543 &
%  0.761 & 0.712 & 0.62 \cr

%           Variables 3 (P3) & 13 & 0.642 & 0.722 & 0.035 & 0.298 &
%  0.059 & 0.857 & 0.373  \cr

%          Variables 4 (P4) & 18 & 0.257 & 0.028 & 0.636 & 0.847 &
%  0.736 & 0.021 & 0.112 \cr
         
%         \hline
%         \hline
% 		\end{tabular}
% 	\end{threeparttable}
% \end{table}

\renewcommand{\arraystretch}{1.5}
\begin{table}[htbp]
	\centering
	% \fontsize{9}{9}\selectfont
    \caption{Mean and Standard Deviation of Return of RL Methods for \emph{LunarLanderContinuous-v2}}
		\label{tab_mean_ll}
  \begin{tabular}{||c||c|c|c|c|c|c|c|c||}
  \hline\hline
  Episode & \multicolumn{2}{c|}{CBRL} & \multicolumn{2}{c|}{Linear} & \multicolumn{2}{c|}{PPO} & \multicolumn{2}{c||}{DDPG} \\
   Number & Mean & Std.\ Dev.\ & Mean & Std.\ Dev.\ & Mean & Std.\ Dev.\ & Mean & Std.\ Dev.\ \\
  \hline
  50 & $-128.81$ & $276.58$ & $-531.05$ & $115.63$ & $-403.56$ & $143.06$ & $-337.22$ & $98.64$ \\
  \hline
  100 & $101.33$ & $262.09$ & $-533.6$ & $130.04$ & $-260.08$ & $238.84$ & $-145.05$ & $42.29$ \\
  \hline
  150 & $258.84$ & $31.52$ & $-510.13$ & $123.56$ & $-188.50$ & $121.78$ & $-154.48$ & $43.2$ \\
  \hline
  200 & $279.24$ & $18.12$ & $-513.44$ & $129.48$ & $-106.87$ & $133.97$ & $-133.51$ & $38.96$ \\
  \hline
  250 & $280.36$ & $17.19$ & $-475.43$ & $125.81$ & $-127.29$ & $81.82$ & $-143.46$ & $112.20$ \\
  \hline
  300 & $279.19$ & $18.42$ & $-473.77$ & $127.78$ & $-109.72$ & $79.65$ & $-73.1$ & $70.24$ \\
  \hline
  350 & $280.08$ & $17.14$ & $-479.01$ & $124.67$ & $-45.12$ & $140.02$ & $-9.04$ & $98.49$ \\
  \hline
  400 & $279.51$ & $17.67$ & $-458.10$ & $131.61$ & $-75.52$ & $111.54$ & $70.04$ & $120.93$ \\
  \hline
  450 & $278.96$ & $19.22$ & $-445.52$ & $143.65$ & $-91.77$ & $105.02$ & $51.15$ & $99.21$ \\
  \hline
  500 & $280.09$ & $17.16$ & $-492.67$ & $126.01$ & $-96.50$ & $113.10$ & $133.20$ & $66.54$ \\
  \hline\hline 
  \end{tabular}
\end{table}

\renewcommand{\arraystretch}{1.5}
\begin{table}[htbp]
	\centering
	% \fontsize{9}{9}\selectfont
	% \begin{threeparttable}
		\caption{Initial Variables of \emph{LunarLanderContinuous-v2}}
		\label{tab_init_para_ll}
		\begin{tabular}{|c|cccccccc|}
        \hline
        \hline
			 Name & $a_0$
			& $a_1$ & $a_2$
			& $a_3$ & $a_4$
			& $a_5$ & $a_6$
			& $a_7$  \cr
        \hline
   
         Initial Variables 1 (P1) & 0.551 & 0.708 & 0.291 & 0.511 & 0.893 & 0.896 & 0.126 & 0.207  \cr

         Initial Variables 2 (P2) & 0.873 & 0.969 & 0.869 & 0.531 & 0.233 & 0.011 & 0.43 & 0.402  \cr

         Initial Variables 3 (P3) & 0.778 & 0.238 & 0.824 & 0.966 & 0.973 & 0.453 & 0.609 & 0.776  \cr

         Initial Variables 4 (P4) & 0.65 & 0.505 & 0.879 & 0.182 & 0.852 & 0.75 & 0.666 & 0.988  \cr

        \hline
			 Name & $a_8$
			& $a_9$ & $a_{10}$
			& $a_{11}$ & $b_0$ & $b_1$
			& $b_2$ & \\

           \hline
   
          Initial Variables 1 (P1) & 0.051 & 0.441 & 0.03 & 0.457 &
 0.649 & 0.278 & 0.676 & \\

        Initial Variables 2 (P2) & 0.523 & 0.478 & 0.555 & 0.543 &
 0.761 & 0.712 & 0.62 & \\

         Initial Variables 3 (P3) & 0.642 & 0.722 & 0.035 & 0.298 &
 0.059 & 0.857 & 0.373 & \\

        Initial Variables 4 (P4) & 0.257 & 0.028 & 0.636 & 0.847 &
 0.736 & 0.021 & 0.112 & \\
         
        \hline
        \hline
		\end{tabular}
	% \end{threeparttable}
\end{table}

\begin{figure}[b]
\centering 
\setcounter{subfigure}{0}
\subfigure[Return vs.\ Episode]
{
\begin{minipage}{0.45\linewidth}
\centering    
\includegraphics[width=1\columnwidth]{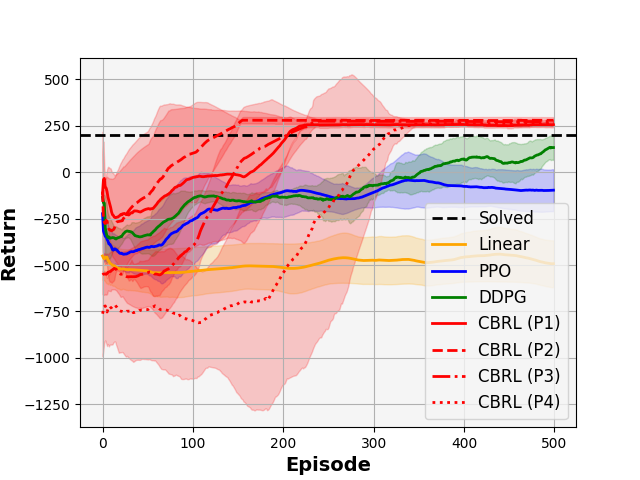}  
\end{minipage}
}
\subfigure[Return vs.\ Time]
{
\begin{minipage}{0.45\linewidth}
\centering    
\includegraphics[width=1\columnwidth]{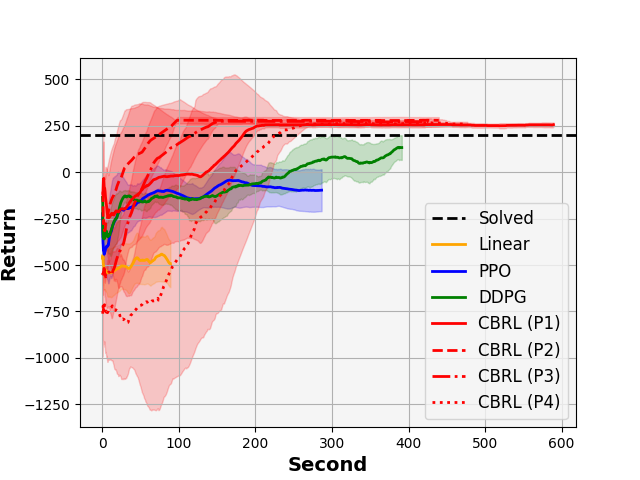}  
\end{minipage}
}\\
\subfigure[P1]
{
\begin{minipage}{0.45\linewidth}
\centering    
\includegraphics[width=1\columnwidth]{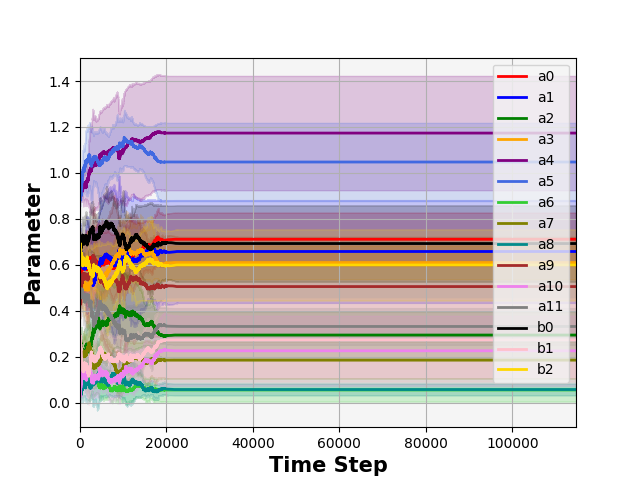}  
\end{minipage}
}
\subfigure[P2]
{
\begin{minipage}{0.45\linewidth}
\centering    
\includegraphics[width=1\columnwidth]{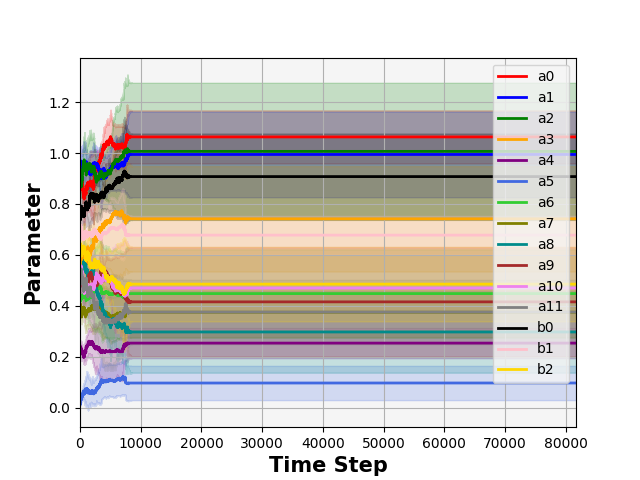}  
\end{minipage}
}\\
\subfigure[P3]
{
	\begin{minipage}{0.45\linewidth}
	\centering 
	\includegraphics[width=1\columnwidth]{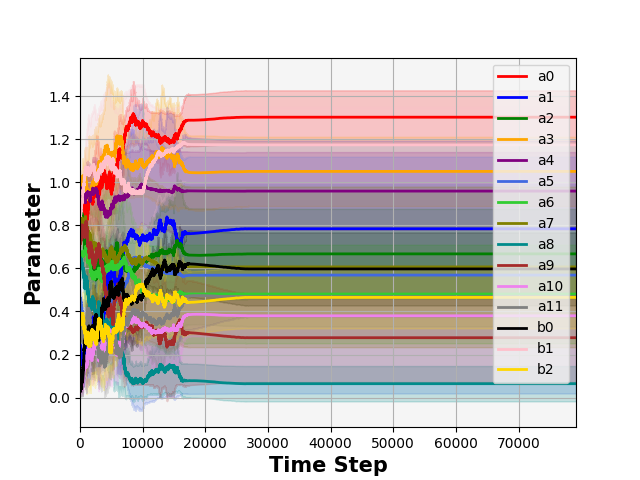}  
	\end{minipage}
}
\subfigure[P4]
{
	\begin{minipage}{0.45\linewidth}
	\centering 
	\includegraphics[width=1\columnwidth]{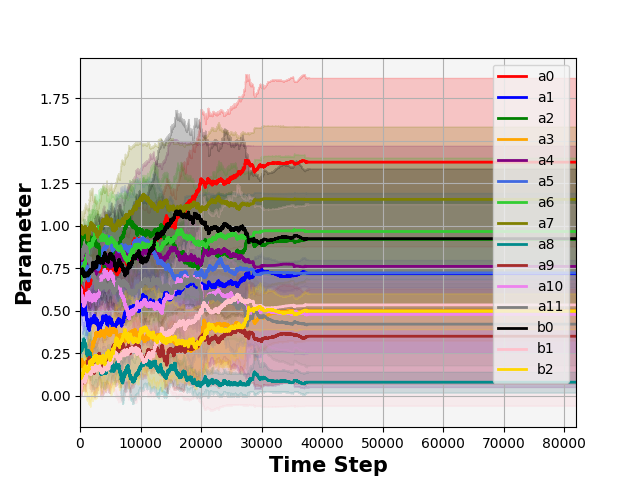}  
	\end{minipage}
}
\caption{Learning curves of \textit{LunarLanderContinuous-v2} over five independent runs. 
The solid line shows the mean and the shaded area depicts the standard deviation. 
(a) and (b): Return vs.\ number of episodes and running time, respectively, for our CBRL approach 
(over the five independent runs and the four initializations in Table~\ref{tab_init_para_ll})
in comparison with the Linear policy, PPO, and DDPG
(over the five independent runs). 
(c) -- (f): Learning behavior of CBRL variables, initialized by Table~\ref{tab_init_para_ll}.}
\label{fig_all_curve_lunarlander}
\end{figure}

\clearpage

\subsection{\textit{Mountain Car (Continuous)} under CBRL Piecewise-LQR}
\label{app:mountaincar}

\paragraph{Environment.}
As depicted in Figure~\ref{mountaincar_env} and described in 
% the main body of the paper, 
Section~\ref{ssec:mountaincar},
\textit{Mountain Car (Continuous)} consists of a car placed in a valley with the goal of accelerating the car to reach the target at the top of the hill on the right by applying a force on the car to the left or right.
The state of the system is given by $x=[p, v]$ in terms of the position of the car $p$ and the velocity of the car $v$. 
% With random initialization at the bottom of the valley, the problem is considered ``solved'' upon achieving an average return of $90$.
%
%
\begin{figure}[H]
    \centering
    \includegraphics[width=0.4\columnwidth]{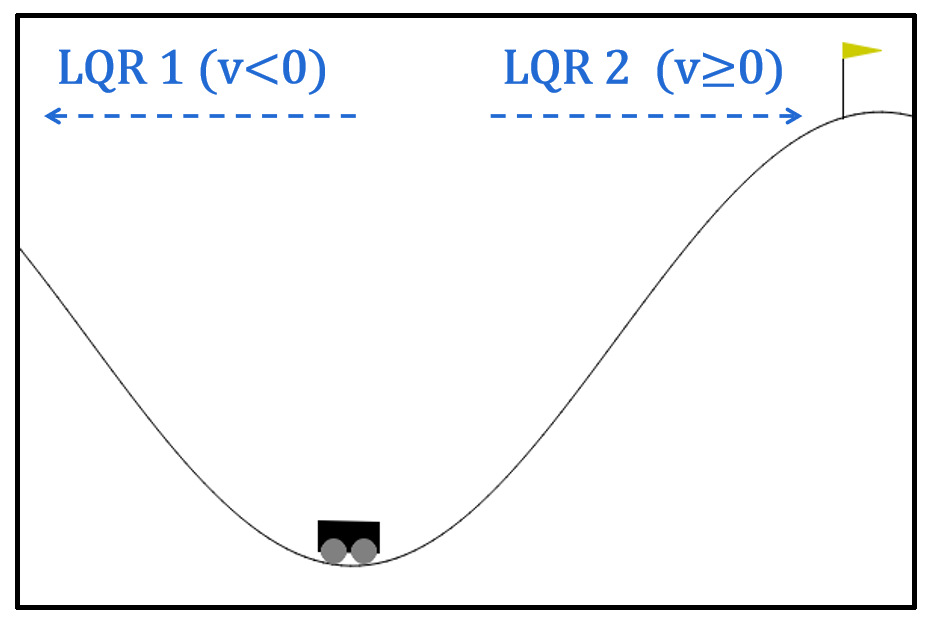}
\caption{The \textit{MountainCarContinuous-v0} environment and partitions for piecewise-LQR.}
\label{mountaincar_env}
\end{figure}

\paragraph{Controller.}
Recall that the LQR controller is not sufficient to solve the \textit{Mountain Car (Continuous)} problem,
even if all the variables of the system are known (e.g., mass of the car and gravity), because a nonlinear controller is required.
Consequently, we consider a piecewise-LQR controller
that partitions the state space into two regions: LQR 1 and LQR 2 (see Figure~\ref{mountaincar_env}).
The target state $x^*=[p^*, v^*]$ for LQR 1 and LQR 2 are respectively selected as $x^*=[-1.2, 0]$ and $x^*=[0.6, 0]$, where $-1.2$ and $0.6$ correspondingly represent the position of the left hill and the right hill.
We address the problem within the context of our CBRL approach by exploiting 
% the general matrix form for the piecewise-LQR dynamics~\eqref{matrixAB_mountaincar}, solely taking into account general physical relationships and laws, with the unknown variables $a_0, a_1, b_0, c_0$ to be learned:
% % Let the objective of LQR be $(s_0, v_0)$, and $e = [s - s_0, v - v_0]$ therefore we obtain
% \begin{align}\label{matrixAB_mountaincar}
%     \dot{e} = \begin{bmatrix} v - v^* \\ \dot{v} - \dot{v}^* \end{bmatrix} = \begin{bmatrix} 0 & 1 \\ a_0 \sin(3 s^*) & a_1 \end{bmatrix} e + \begin{bmatrix} 0 \\ b_0 \end{bmatrix} (u - c_0 \cos(3 s^*)),
% \end{align}
% where $e=x-x^*$ and we select $v^*=0, \dot{v}^*=0$, and $s^*=-1.2$ or $0.6$.
% % Then, 
% % \begin{align}
% %     u = -K e + c_0 \cos(3 s_0).
% % \end{align}
the following general matrix form for the piecewise-LQR dynamics
\begin{align}\label{matrixAB_mountaincar}
    \dot{e} = \begin{bmatrix} v - v^* \\ \dot{v} - \dot{v}^* \end{bmatrix} = \begin{bmatrix} 0 & 1 \\ a_0 \sin(3 p^*) & a_1 \end{bmatrix} e + \begin{bmatrix} 0 \\ b_0 \end{bmatrix} (u - c_0 \cos(3 p^*)),
\end{align}
which solely takes into account general physical relationships and laws, with the $d=4$ unknown variables $a_0, a_1, b_0, c_0$ to be learned
% Let the objective of LQR be $(s_0, v_0)$, and $e = [s - s_0, v - v_0]$ therefore we obtain
where $e=x-x^*$ and we select $v^*=0, \dot{v}^*=0$, and $p^*=-1.2$ or $0.6$.
% Then, 
% \begin{align}
%     u = -K e + c_0 \cos(3 s_0).
% \end{align}

\paragraph{Numerical Results. } 
Table~\ref{tab_mean_mc} and Figure~\ref{fig_all_curve_mountaincar} present numerical results for the three state-of-the-art baselines (continuous actions) and our CBRL approach, each run over five independent random seeds. 
All variables in~\eqref{matrixAB_mountaincar} are initialized uniformly within $(0,1)$ and then learned using our control-policy-variable gradient ascent iteration~\eqref{eq:policy-gradient}. 
We consider four sets of initial variables (see Table~\ref{tab_init_para_mc}) to validate the robustness of our CBRL approach.
Figure~\ref{fig_all_curve_mountaincar}(c)~--~\ref{fig_all_curve_mountaincar}(f) illustrates the learning behavior of CBRL variables given different initialization.

Table~\ref{tab_mean_mc} and Figure~\ref{fig_all_curve_mountaincar}(a) and~\ref{fig_all_curve_mountaincar}(b) clearly demonstrate that our CBRL approach provides superior performance, in terms of both mean and standard deviation, over all baselines w.r.t.\ both the number of episodes and running time, in addition to demonstrating a more stable training process.
More specifically, our CBRL approach outperforms all baselines with a significant improvement in mean return over independent random environment runs from Gymnasium~\citep{towers_gymnasium_2023}, together with a significant reduction in the standard deviation of the return over these independent runs, thus rendering a more robust solution approach where the performance of each run is much closer to the mean than under the corresponding baseline results.
In fact, the relative percentage difference\footnote{For the relative comparison of improved mean return performance of our CBRL method over the next-best RL method, 
since all the performance values are positive, 
we use the standard relative percentage difference formula $\frac{M_{\scriptscriptstyle CBRL} - M_{\scriptscriptstyle RL}}{M_{\scriptscriptstyle RL}} \times 100$. Similarly, for the relative comparison of reduced standard deviation of return performance from the next-best RL method to our CBRL method, 
since all the performance values are positive, 
we use the standard relative percentage difference formula $\frac{S_{\scriptscriptstyle RL} - S_{\scriptscriptstyle CBRL}}{S_{\scriptscriptstyle CBRL}} \times 100$.
Refer to \url{https://en.wikipedia.org/wiki/Relative_change}}
in improved mean return performance of our CBRL approach over the next-best RL method DDPG (based on mean performance at $500$ episodes) is $0.9\%$ and $1\%$ at $200$ and $500$ episodes, respectively;
and the relative percentage difference in reduced standard deviation of return performance from the next-best RL method DDPG to our CBRL approach is $324\%$ and $257\%$ at $200$ and $500$ episodes, respectively.
Moreover, the relative percentage difference
in improved mean return performance of our CBRL approach over the PPO RL method is $41\%$ and $23\%$ at $200$ and $500$ episodes, respectively;
and the relative percentage difference in reduced standard deviation of return performance from the PPO RL method to our CBRL approach is $5743\%$ and $729\%$ at $200$ and $500$ episodes, respectively.

We note an important difference between Fig.~\ref{fig_mc_episode}~--~\ref{fig_mc_second} and Fig.~\ref{fig_all_curve_mountaincar}(a)~--~\ref{fig_all_curve_mountaincar}(b), namely the shaded areas in Fig.~\ref{fig_mc_episode}~--~\ref{fig_mc_second} represent one form of variability across independent random seeds with the same variable initialization, whereas the shaded areas in Fig.~\ref{fig_all_curve_mountaincar}(a)~--~\ref{fig_all_curve_mountaincar}(b) represent the combination of two forms of variability---one across independent random seeds with the same variable initialization and the other across different random variable initializations.

% random seed = (0, 5, 10, 15)

% \renewcommand{\arraystretch}{1.5}
% \begin{table}[H]
% 	\centering
% 	\fontsize{10}{10}\selectfont
% 	\begin{threeparttable}
% 		\caption{Initial Variables of \emph{MountainCarContinuous-v0}}
% 		\label{tab_init_para_mc}
% 		\begin{tabular}{c|c|cccc}
%         \hline
%         \hline
% 			 Name & Seed & $a_0$
% 			& $a_1$ & $b_0$ & $c_0$  \cr
%         \hline
   
%           Variables 1 (P1) & 0 & 0.549 & 0.715 & 0.603 & 0.545 \cr

%          Variables 2 (P2) & 5 & 0.222 & 0.871 & 0.207 & 0.919 \cr

%           Variables 3 (P3) &  10 & 0.771 & 0.021 & 0.634 & 0.749 \cr

%          Variables 4 (P4) & 15 & 0.849 & 0.179 & 0.054 & 0.362 \cr

%           % Variables 5 (P5) &  20 & 0.588 & 0.898 & 0.892 & 0.816 \cr

%         \hline
%         \hline
% 		\end{tabular}
% 	\end{threeparttable}
% \end{table}

\renewcommand{\arraystretch}{1.5}
\begin{table}[htbp]
	\centering
	% \fontsize{9}{9}\selectfont
    \caption{Mean and Standard Deviation of Return of RL Methods for \emph{MountainCarContinuous-v0}}
		\label{tab_mean_mc}
  \begin{tabular}{||c||c|c|c|c|c|c|c|c||}
  \hline\hline
  Episode & \multicolumn{2}{c|}{CBRL} & \multicolumn{2}{c|}{Linear} & \multicolumn{2}{c|}{PPO} & \multicolumn{2}{c||}{DDPG} \\
   Number & Mean & Std.\ Dev.\ & Mean & Std.\ Dev.\ & Mean & Std.\ Dev.\ & Mean & Std.\ Dev.\ \\
  \hline
  50 & $79.34$ & $1.02$ & $61.09$ & $16.84$ & $59.62$ & $7.32$ & $74.20$ & $7.20$ \\
  \hline
  100 & $93.03$ & $0.37$ & $61.64$ & $17.19$ & $65.01$ & $9.69$ & $89.10$ & $4.78$ \\
  \hline
  150 & $93.63$ & $0.21$ & $63.12$ & $19.53$ & $69.58$ & $10.57$ & $91.98$ & $1.47$ \\
  \hline
  200 & $93.63$ & $0.21$ & $61.94$ & $21.83$ & $66.36$ & $12.27$ & $92.82$ & $0.89$ \\
  \hline
  250 & $93.63$ & $0.21$ & $62.30$ & $25.59$& $71.35$ & $13.15$ & $93.18$ & $0.41$ \\
  \hline
  300 & $93.63$ & $0.21$ & $62.00$ & $25.12$ & $71.25$ & $12.27$ & $93.62$ & $0.35$ \\
  \hline
  350 & $93.63$ & $0.21$ & $57.13$ & $34.55$ & $76.23$ & $3.03$ & $93.39$ & $0.51$ \\
  \hline
  400 & $93.63$ & $0.21$ & $62.42$ & $25.91$ & $81.32$ & $1.63$ & $93.15$ & $0.81$ \\
  \hline
  450 & $93.63$ & $0.21$ & $66.19$ & $20.54$ & $81.10$ & $2.02$ & $92.67$ & $0.85$ \\
  \hline
  500 & $93.63$ & $0.21$ & $61.52$ & $28.84$ & $76.31$ & $1.74$ & $92.72$ & $0.75$ \\
  \hline\hline 
  \end{tabular}
\end{table}

\renewcommand{\arraystretch}{1.5}
\begin{table}[htbp]
	\centering
	% \fontsize{9}{9}\selectfont
	% \begin{threeparttable}
		\caption{Initial Variables of \emph{MountainCarContinuous-v0}}
		\label{tab_init_para_mc}
		\begin{tabular}{|c|cccc|}
        \hline
        \hline
			 Name & $a_0$
			& $a_1$ & $b_0$ & $c_0$ \\
        \hline
   
          Initial Variables 1 (P1) & 0.549 & 0.715 & 0.603 & 0.545 \\

         Initial Variables 2 (P2) & 0.222 & 0.871 & 0.207 & 0.919 \\

         Initial Variables 3 (P3) & 0.771 & 0.021 & 0.634 & 0.749 \\

        % Initial Variables 4 (P4) & 0.849 & 0.179 & 0.054 & 0.362 \cr % original P4

         Initial Variables 4 (P4) & 0.588 & 0.898 & 0.892 & 0.816 \\

        \hline
        \hline
		\end{tabular}
	% \end{threeparttable}
\end{table}

\begin{figure}[b]
\centering 
\setcounter{subfigure}{0}
\subfigure[Return vs.\ Episode]
{
\begin{minipage}{0.45\linewidth}
\centering    
\includegraphics[width=1\columnwidth]{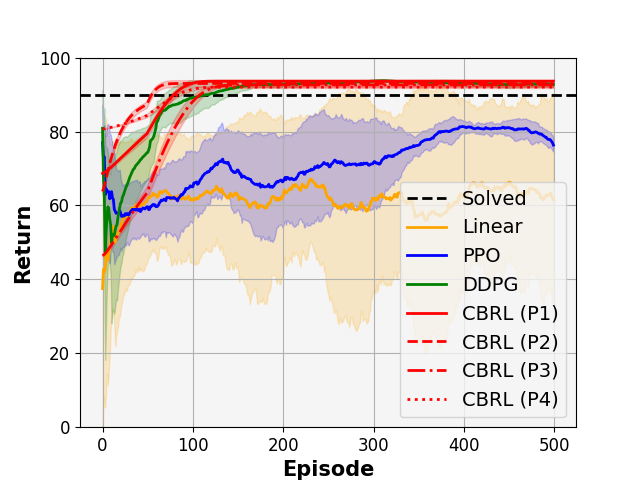}  
\end{minipage}
}
\subfigure[Return vs.\ Time]
{
\begin{minipage}{0.45\linewidth}
\centering    
\includegraphics[width=1\columnwidth]{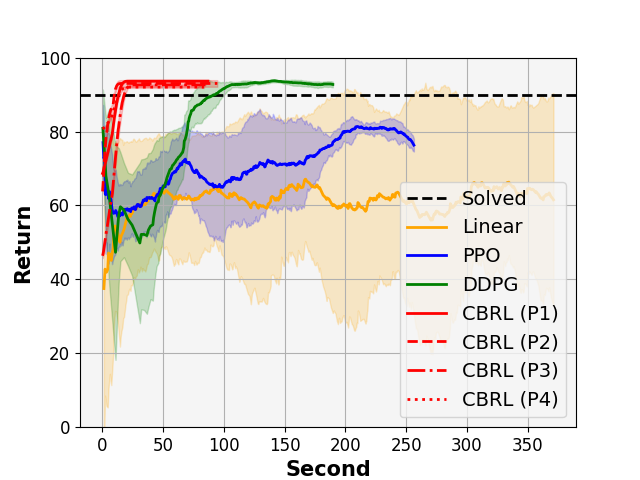}  
\end{minipage}
}\\
\subfigure[P1]
{
\begin{minipage}{0.45\linewidth}
\centering    
\includegraphics[width=1\columnwidth]{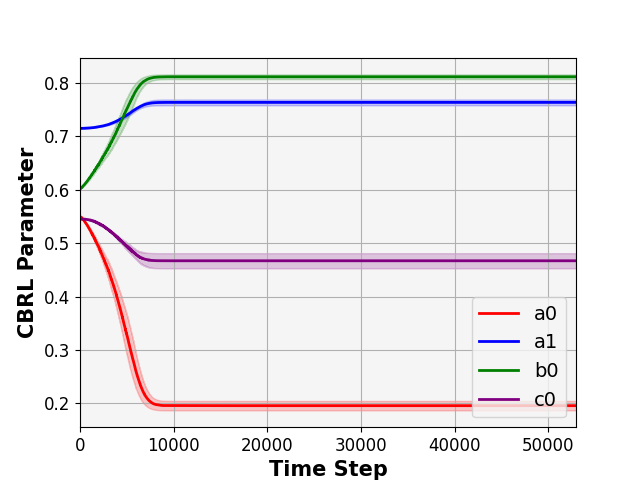}  
\end{minipage}
}
\subfigure[P2]
{
\begin{minipage}{0.45\linewidth}
\centering    
\includegraphics[width=1\columnwidth]{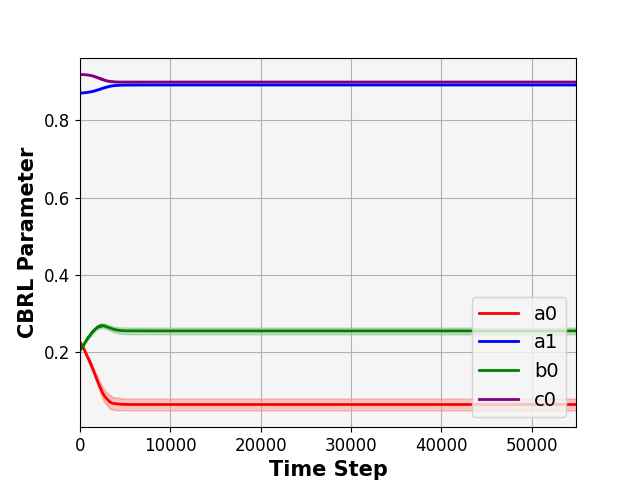}  
\end{minipage}
}\\
\subfigure[P3]
{
	\begin{minipage}{0.45\linewidth}
	\centering 
	\includegraphics[width=1\columnwidth]{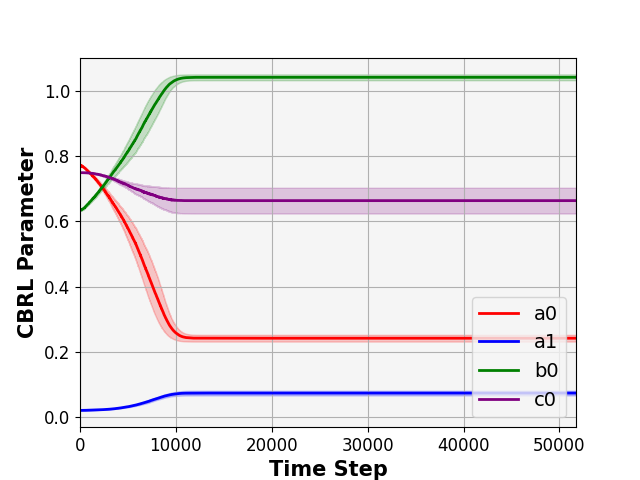}  
	\end{minipage}
}
\subfigure[P4]
{
	\begin{minipage}{0.45\linewidth}
	\centering 
	\includegraphics[width=1\columnwidth]{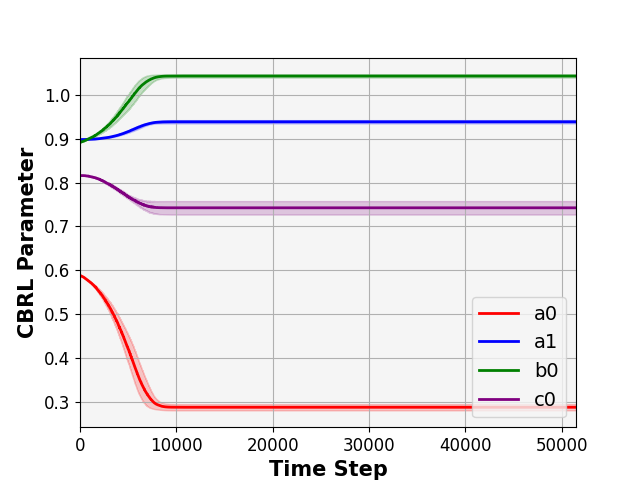}  
	\end{minipage}
}
\caption{Learning curves of \textit{MountainCarContinuous-v0} over five independent runs. 
The solid line shows the mean and the shaded area depicts the standard deviation. 
(a) and (b): Return vs.\ number of episodes and running time, respectively, for our CBRL approach 
(over the five independent runs and the four initializations in Table~\ref{tab_init_para_mc})
in comparison with the Linear policy, PPO, and DDPG
(over the five independent runs).
(c) -- (f): Learning behavior of CBRL variables, initialized by Table~\ref{tab_init_para_mc}.}
\label{fig_all_curve_mountaincar}
\end{figure}

\clearpage

\subsection{\textit{Pendulum} under CBRL Piecewise-LQR}
\label{app:pendulum}

\paragraph{Environment.}
As depicted in Figure~\ref{pendulum_env} and described in 
% the main body of the paper, 
Section~\ref{ssec:pendulum},
\textit{Pendulum} consists of a link attached at one end to a fixed point and the other end being free, with the goal of swinging up to an upright position by applying a torque on the free end.
The state of the system is given by $x=[\theta, \dot{\theta}]$ in terms of the angle of the link $\theta$ and the angular velocity of the link $\dot{\theta}$. 
% With random initial position of the link, each episode comprises $200$ steps and the problem is considered ``solved'' upon achieving an average return of $-200$.
%
%
\begin{figure}[H]
    \centering
    \includegraphics[width=0.3\columnwidth]{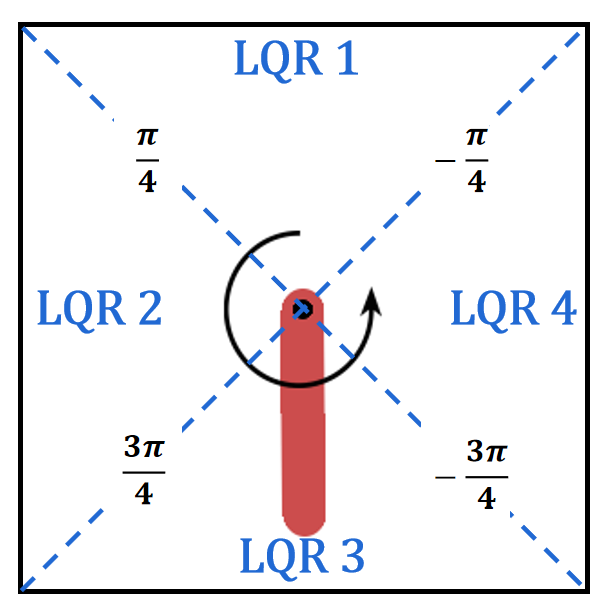}
\caption{The \textit{Pendulum-v1} environment and partitions for piecewise-LQR.}
\label{pendulum_env}
\end{figure}

\paragraph{Controller.}
Recall that the LQR controller is not sufficient to solve the \textit{Pendulum} problem, even if all the variables of the system are known (e.g., mass of the link $m$, length of the link $l$, moment of inertia of the link $J$, and gravity $g$), because a nonlinear controller is required. Consequently, we consider a piecewise-LQR controller that partitions the state space into four regions: LQR 1, LQR 2, LQR 3, LQR 4 (see Figure~\ref{pendulum_env}). 
In terms of the target state $x^*=[\theta^*, \dot{\theta}^*]$, $\theta^*$ is selected based on the boundary angle in each partition (counter-clockwise boundary angle if $\dot{\theta} > 0$, and clockwise boundary angle otherwise), while $\dot{\theta}^*$ is selected based on the energy conservation law. More specifically, $\theta^*$ in LQR 1 is selected to be $\theta^* = \pi/4$ if $\dot{\theta} > 0$, and to be $\theta^* = -\pi/4$ otherwise.
Given the assumption that the link reaches a zero velocity at the upright position, $\dot{\theta}^*$ is selected at any position to be $\dot{\theta}^* = \sqrt{m g l (1- \cos(\theta)) / J}$. We address the problem within the context of our CBRL approach by exploiting 
% the general matrix form for the piecewise-LQR dynamics~\eqref{matrixAB_pendulum}, solely taking into account general physical relationships and laws, with the unknown variables $a_0, a_1, b_0, c_0, d_0$ to be learned:
% %
% \begin{align}\label{matrixAB_pendulum}
%     \dot{e} = \begin{bmatrix} \dot{\theta} - \dot{\theta}^* \\ \ddot{\theta} - \ddot{\theta}^* \end{bmatrix} = \begin{bmatrix} 0 & 1 \\ a_0 \cos(\theta^*) & a_1 \end{bmatrix} e + \begin{bmatrix} 0 \\ b_0 \end{bmatrix} (u + c_0 \sin(\theta^*)), \quad \dot{\theta}^* = \sqrt{d_0 ( 1-\cos(\theta) )},
% \end{align}
% where $e=x-x^*$ and we select $\ddot{\theta}^*=0$.
% % Then, 
% % \begin{align*}
% %     u = -K e - c_0 \sin(\theta_0).
% % \end{align*}
the following general matrix form for the piecewise-LQR dynamics
\begin{align}\label{matrixAB_pendulum}
    \dot{e} = \begin{bmatrix} \dot{\theta} - \dot{\theta}^* \\ \ddot{\theta} - \ddot{\theta}^* \end{bmatrix} = \begin{bmatrix} 0 & 1 \\ a_0 \cos(\theta^*) & a_1 \end{bmatrix} e + \begin{bmatrix} 0 \\ b_0 \end{bmatrix} (u + c_0 \sin(\theta^*)), \quad \dot{\theta}^* = \sqrt{d_0 ( 1-\cos(\theta) )},
\end{align}
which solely takes into account general physical relationships and laws, with the $d=5$ unknown variables $a_0, a_1, b_0, c_0, d_0$ to be learned
where $e=x-x^*$ and we select $\ddot{\theta}^*=0$.
% Then, 
% \begin{align*}
%     u = -K e - c_0 \sin(\theta_0).
% \end{align*}

% \begin{figure}[h]
%     \centering
% \includegraphics[width=0.3\textwidth]{figures/PW_LQR.png}
% \caption{Partition of piece-wise LQR in \textit{Pendulum}.}
% \label{fig_partition_pendulum}
% \end{figure}

\paragraph{Numerical Results.}
Table~\ref{tab_mean_pd} and Figure~\ref{fig_all_curve_pendulum} present numerical results for the three state-of-the-art baselines (continuous actions) and our CBRL approach, each run over five independent random seeds. 
All variables in~\eqref{matrixAB_pendulum} are initialized uniformly within $(0,1)$ and then learned using our control-policy-variable gradient ascent iteration~\eqref{eq:policy-gradient}. 
We consider four sets of initial variables (see Table~\ref{tab_init_para_pd}) to validate the robustness of our CBRL approach. 
Figure~\ref{fig_all_curve_pendulum}(c)~--~\ref{fig_all_curve_pendulum}(f) illustrates the learning behavior of CBRL variables given different initialization.

Table~\ref{tab_mean_pd} and Figure~\ref{fig_all_curve_pendulum}(a) and~\ref{fig_all_curve_pendulum}(b) clearly demonstrate that our CBRL approach provides far superior performance, in terms of both mean and standard deviation, over all baselines w.r.t.\ both the number of episodes and running time upon convergence of our CBRL algorithm after a relatively small number of episodes,
in addition to demonstrating a more stable training process.
More specifically, even with only four partitions for the difficult nonlinear \textit{Pendulum} RL task, our CBRL approach outperforms all baselines after $150$ episodes with a significant improvement in mean return over independent random environment runs from Gymnasium~\citep{towers_gymnasium_2023}, together with a significant reduction in the standard deviation of the return over these independent runs, thus rendering a more robust solution approach where the performance of each run is much closer to the mean than under the corresponding baseline results.
In fact, the relative percentage difference\footnote{For the relative comparison of improved mean return performance of our CBRL method over the next-best RL method, 
since all the performance values are negative,
% and $|M_{\scriptscriptstyle CBRL}| < |M_{\scriptscriptstyle RL}|$, 
we use the reversed standard relative percentage difference formula $\frac{|M_{\scriptscriptstyle RL} - M_{\scriptscriptstyle CBRL}|}{|M_{\scriptscriptstyle CBRL}|} \times 100$. Similarly, for the relative comparison of reduced standard deviation of return performance from the next-best RL method to our CBRL method, 
since all the performance values are positive, 
we use the standard relative percentage difference formula $\frac{S_{\scriptscriptstyle RL} - S_{\scriptscriptstyle CBRL}}{S_{\scriptscriptstyle CBRL}} \times 100$.
Refer to \url{https://en.wikipedia.org/wiki/Relative_change}}
in improved mean return performance of our CBRL approach over the next-best RL method DDPG (based on mean performance at $500$ episodes) is $40\%$ and $53\%$ at $200$ and $500$ episodes, respectively;
and the relative percentage difference in reduced standard deviation of return performance from the next-best RL method DDPG to our CBRL approach is $210\%$ and $263\%$ at $200$ and $500$ episodes, respectively.
Moreover, the relative percentage difference
in improved mean return performance of our CBRL approach over the PPO RL method is $197\%$ and $168\%$ at $200$ and $500$ episodes, respectively;
and the relative percentage difference in reduced standard deviation of return performance from the PPO RL method to our CBRL approach is $275\%$ and $793\%$ at $200$ and $500$ episodes, respectively.

We note an important difference between Fig.~\ref{fig_pd_episode}~--~\ref{fig_pd_second} and Fig.~\ref{fig_all_curve_pendulum}(a)~--~\ref{fig_all_curve_pendulum}(b), namely the shaded areas in Fig.~\ref{fig_pd_episode}~--~\ref{fig_pd_second} represent one form of variability across independent random seeds with the same variable initialization, whereas the shaded areas in Fig.~\ref{fig_all_curve_pendulum}(a)~--~\ref{fig_all_curve_pendulum}(b) represent the combination of two forms of variability---one across independent random seeds with the same variable initialization and the other across different random variable initializations.

% random seed = (1, 6, 11, 16)

% \renewcommand{\arraystretch}{1.5}
% \begin{table}[H]
% 	\centering
% 	\fontsize{10}{10}\selectfont
% 	\begin{threeparttable}
% 		\caption{Initial Variables of \emph{Pendulum-v1}}
% 		\label{tab_init_para_pd}
% 		\begin{tabular}{c|c|ccccc}
%         \hline
%         \hline
% 			 Name & Seed & $a_0$
% 			& $a_1$ & $b_0$ & $c_0$ & $d_0$  \cr
%         \hline
   
%           Variables 1 (P1) & 1 & 0.417 & 0.72 & 0.302 & 0.147 & 0.092  \cr

%          Variables 2 (P2) & 6 & 0.893 & 0.332 & 0.821 & 0.042 & 0.108  \cr

%           Variables 3 (P3) &  11 & 0.18 & 0.019 & 0.463 & 0.725 & 0.42  \cr

%          Variables 4 (P4) & 16 & 0.223 & 0.523 & 0.551 & 0.046 & 0.361  \cr

%         \hline
%         \hline
% 		\end{tabular}
% 	\end{threeparttable}
% \end{table}

\renewcommand{\arraystretch}{1.5}
\begin{table}[htbp]
	\centering
	% \fontsize{9}{9}\selectfont
    \caption{Mean and Standard Deviation of Return of RL Methods for \emph{Pendulum-v1}}
		\label{tab_mean_pd}
  \begin{tabular}{||c||c|c|c|c|c|c|c|c||}
  \hline\hline
  Episode & \multicolumn{2}{c|}{CBRL} & \multicolumn{2}{c|}{Linear} & \multicolumn{2}{c|}{PPO} & \multicolumn{2}{c||}{DDPG} \\
   Number & Mean & Std.\ Dev.\ & Mean & Std.\ Dev.\ & Mean & Std.\ Dev.\ & Mean & Std.\ Dev.\ \\
  \hline
  50 & $-1378.78$ & $43.89$ & $-1331.73$ & $36.25$ & $-1147.35$ & $181.20$ & $-942.18$ & $170.83$ \\
  \hline
  100 & $-1004.75$ & $101.66$ & $-1145.30$ & $53.57$ & $-757.11$ & $205.66$ & $-355.09$ & $324.16$ \\
  \hline
  150 & $-284.26$ & $130.17$ & $-1088.67$ & $83.73$ & $-559.65$ & $151.39$ & $-241.16$ & $162.85$ \\
  \hline
  200 & $-172.48$ & $48.32$ & $-1062.06$ & $102.63$ & $-511.85$ & $181.28$ & $-240.94$ & $149.57$ \\
  \hline
  250 & $-152.80$ & $37.43$ & $-1058.88$ & $92.55$ & $-561.78$ & $384.15$ & $-222.18$ & $151.71$ \\
  \hline
  300 & $-145.56$ & $44.16$ & $-1057.02$ & $90.29$ & $-524.10$ & $259.46$ & $-221.78$ & $150.40$ \\
  \hline
  350 & $-141.80$ & $36.63$ & $-1036.65$ & $88.37$ & $-325.66$ & $133.77$ & $-222.83$ & $146.68$ \\
  \hline
  400 & $-149.34$ & $51.70$ & $-1044.34$ & $100.90$ & $-218.37$ & $100.83$ & $-220.28$ & $148.80$ \\
  \hline
  450 & $-139.65$ & $32.36$ & $-1032.82$ & $87.11$ & $-296.94$ & $184.32$ & $-220.50$ & $148.68$ \\
  \hline
  500 & $-143.95$ & $40.91$ & $-999.41$ & $78.75$ & $-385.50$ & $365.39$ & $-220.38$ & $148.54$ \\
  \hline\hline 
  \end{tabular}
\end{table}

\renewcommand{\arraystretch}{1.5}
\begin{table}[htbp]
	\centering
	% \fontsize{9}{9}\selectfont
	% \begin{threeparttable}
		\caption{Initial Variables of \emph{Pendulum-v1}}
		\label{tab_init_para_pd}
		\begin{tabular}{|c|ccccc|}
        \hline
        \hline
			 Name & $a_0$
			& $a_1$ & $b_0$ & $c_0$ & $d_0$ \\
        \hline
   
         Initial Variables 1 (P1) & 0.417 & 0.72 & 0.302 & 0.147 & 0.092 \\

        Initial Variables 2 (P2) & 0.893 & 0.332 & 0.821 & 0.042 & 0.108 \\

         Initial Variables 3 (P3) & 0.18 & 0.019 & 0.463 & 0.725 & 0.42 \\

        Initial Variables 4 (P4) & 0.223 & 0.523 & 0.551 & 0.046 & 0.361 \\

        \hline
        \hline
		\end{tabular}
	% \end{threeparttable}
\end{table}

\begin{figure}[b]
\centering 
\setcounter{subfigure}{0}
\subfigure[Return vs.\ Episode]
{
\begin{minipage}{0.45\linewidth}
\centering    
\includegraphics[width=1\columnwidth]{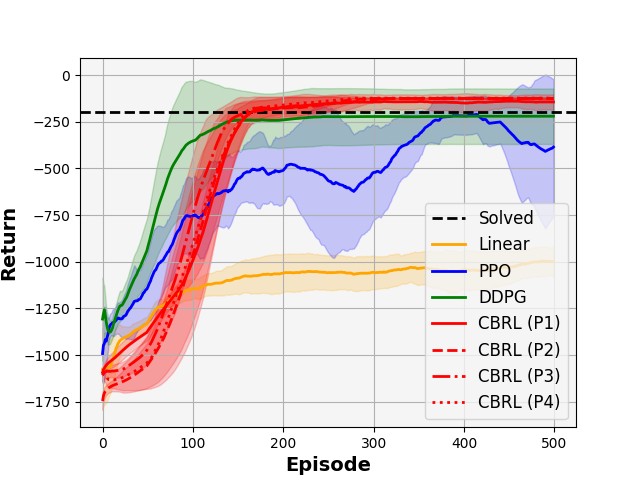}  
\end{minipage}
}
\subfigure[Return vs.\ Time]
{
\begin{minipage}{0.45\linewidth}
\centering    
\includegraphics[width=1\columnwidth]{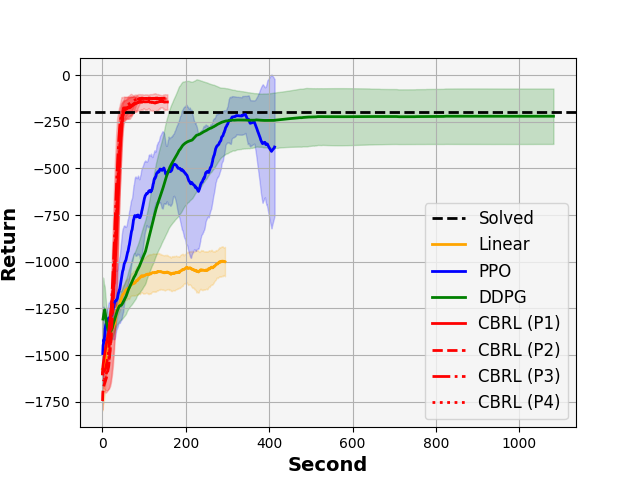}  
\end{minipage}
}
\subfigure[P1]
{
\begin{minipage}{0.45\linewidth}
\centering    
\includegraphics[width=1\columnwidth]{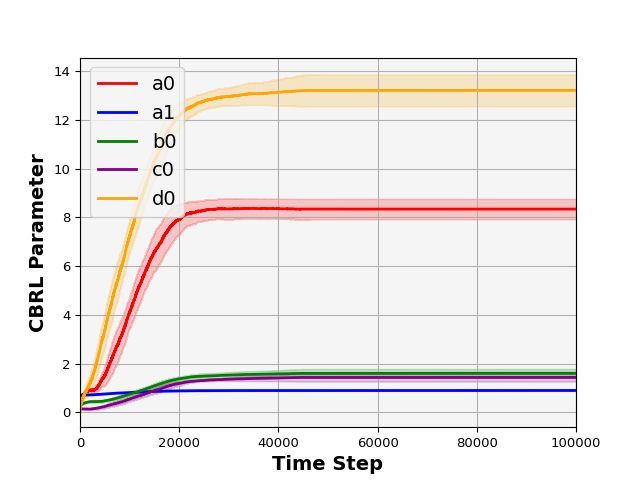}  
\end{minipage}
}
\subfigure[P2]
{
\begin{minipage}{0.45\linewidth}
\centering    
\includegraphics[width=1\columnwidth]{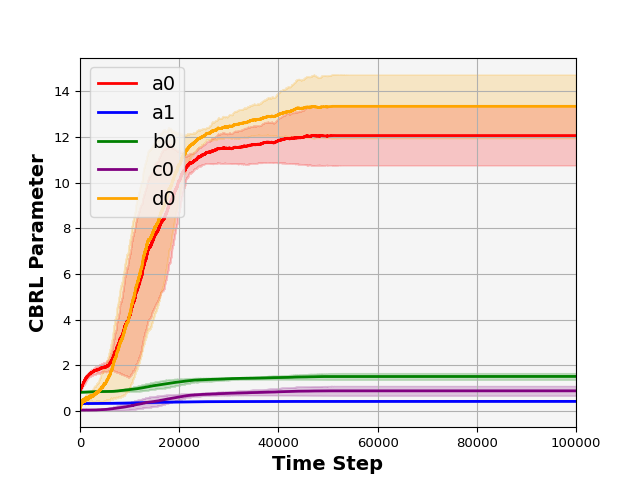}  
\end{minipage}
}\\
\subfigure[P3]
{
	\begin{minipage}{0.45\linewidth}
	\centering 
	\includegraphics[width=1\columnwidth]{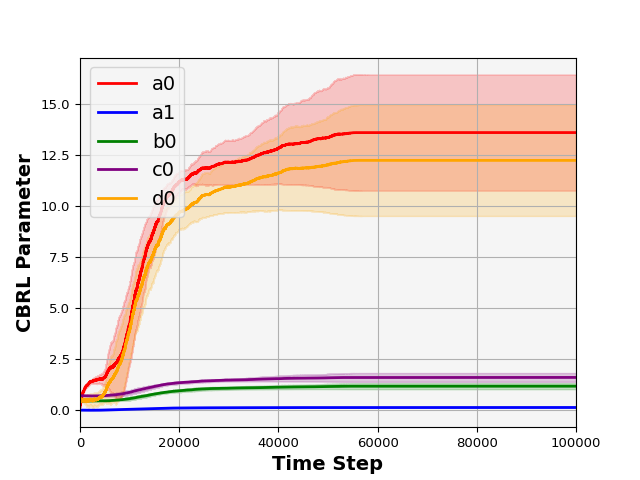}  
	\end{minipage}
}
\subfigure[P4]
{
	\begin{minipage}{0.45\linewidth}
	\centering 
	\includegraphics[width=1\columnwidth]{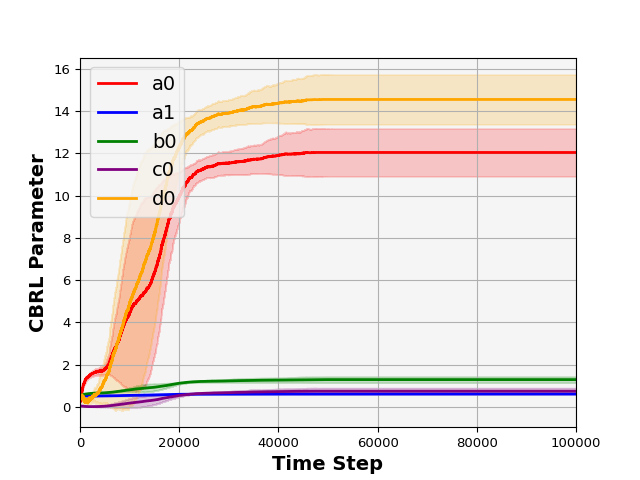}  
	\end{minipage}
}
\caption{Learning curves of \textit{Pendulum-v1} over five independent runs. 
The solid line shows the mean and the shaded area depicts the standard deviation. 
(a) and (b): Return vs.\ number of episodes and running time, respectively, for our CBRL approach 
(over the five independent runs and the four initializations in Table~\ref{tab_init_para_pd})
in comparison with the Linear policy, PPO, and DDPG
(over the five independent runs). 
(c) -- (f): Learning behavior of CBRL variables, initialized by Table~\ref{tab_init_para_pd}.}
\label{fig_all_curve_pendulum}
\end{figure}

\clearpage

\subsection{Implementation}

\paragraph{Baselines.}
As previously noted, the evaluation of our CBRL approach includes comparisons against the three baselines of DQN for discrete actions, DDPG for continuous actions and PPO, which are selected as the state-of-the-art RL algorithms for solving the 
% control 
RL
tasks under consideration.
The additional baseline of replacing the nonlinear policy of PPO with a linear policy was again included because the optimal policy for some problems, e.g., \textit{Cart Pole}, is known to be linear.
This Linear algorithm solely replaces the policy network of PPO with a linear policy while maintaining all other components unchanged. More specifically, we decrease the number of hidden layers in the policy network by $1$, reduce the hidden layer size of the policy network to $16$, and remove all (nonlinear) activation functions in the policy network.

\paragraph{Hyperparameters.}
We summarize in Table~\ref{tab_hyper_nn} the hyperparameters that have been used in our numerical experiments. To ensure a fair comparison, we maintain consistent hyperparameters (e.g., hidden layer size of the value network and the activation function) across all algorithms wherever applicable.

% \orange{Our implementation is run on a MacBook Air M1 chip with Python 3.8.18, Numpy 1.22.0, Scipy 1.10.1, Gymnasium 0.29.1, Torch 2.2.2, etc.}
% {\color{blue} I'm not arguing that we need to include such information. I just looked back at other papers published/submitted to AI conferences and the \textbf{Implementation} paragraph often includes a sentence along the lines of "We implement the above procedure in Java, using version $22.1$ of the commercial solver IBM ILOG CPLEX to solve the linear programs and mixed-integer linear programs."
% if what we use is standard and doesn't involve some version of specialized code, then maybe we don't need this.}
% {\color{red} if we include this paragraph in the main body of the paper, then it should be commented out here.
% }
% \orange{I see. The libraries and environments that we have used are standard. So probably we don't need these.}

\renewcommand{\arraystretch}{1.5}
\begin{table}[h]
	\centering
	% \fontsize{9}{9}\selectfont
	% \begin{threeparttable}
		\caption{Summary of Experimental Hyperparameters }
		\label{tab_hyper_nn}
		\begin{tabular}{|c|ccccc|}
        \hline
        \hline
			 Algorithm & Linear & PPO
			& DDPG & DQN & CBRL \\
        \hline
            Action Generator & NN & NN & NN & NN & CBRL-LQR \\
   
            \# of Hidden Layers (Policy) & 1 & 2 & 2 & 2 & N/A \\

            \# of Hidden Layers (Value) & 2 & 2 & 2 & 2 & 2 \\

            Hidden Layer Size (Policy) & 16 & 128 & 128 & 128 & N/A \\

           Hidden Layer Size (Value) & 128 & 128 & 128 & 128 & 128 \\

            Activation (Policy) & N/A & ReLU \& Tanh & ReLU \& Tanh & ReLU \& Tanh & N/A \\

            Activation (Value) & ReLU & ReLU & ReLU & ReLU & ReLU \\

            Discount Factor & 0.99 & 0.99 & 0.99 & 0.99 & 0.99 \\

            Clip Ratio & 0.2 & 0.2 & N/A & N/A & N/A \\

           Soft Update Ratio & N/A & N/A & 0.005 & N/A & N/A \\

        \hline
        \hline
		\end{tabular}
	% \end{threeparttable}
\end{table}

% \paragraph{Computing Infrastructure.}
% %
% All experiments in this work were implemented and executed on an 8G-memory Macbook Air with a M1 Chip.

% {\color{red} should we add more details on the implementation?  for example, "We implement the above procedure in Java, using version $22.1$ of the commercial solver IBM ILOG CPLEX to solve the 
% linear programs 
% and mixed-integer
% linear programs, blah, blah, blah."}
% \orange{I personally don't think the version of Python/Gymnasium/Pytorch etc are necessary because they are standard libraries, but I added a sentence in the paragraph above in case you still want it.}

\end{document}